\documentclass{article}

     \usepackage[preprint]{neurips_2024}

\usepackage{amsmath}
\usepackage{amssymb}
\usepackage{mathtools}
\usepackage{amsthm}
\usepackage{algorithm}
\usepackage{algorithmicx}
\usepackage{algpseudocode}
\usepackage{wrapfig}
\usepackage[T1]{fontenc}    
\usepackage[utf8]{inputenc} 
\usepackage[T1]{fontenc}    
\usepackage{hyperref}       
\usepackage{xurl}
\usepackage{booktabs}       
\usepackage{amsfonts}       
\usepackage{nicefrac}       
\usepackage{microtype}      
\usepackage{xcolor}         
\usepackage[page, header]{appendix}
\usepackage{graphicx}
\usepackage{caption}
\usepackage{subcaption}
\usepackage{titletoc}
\usepackage[capitalize,noabbrev]{cleveref}
\usepackage{soul}
\newtheorem{theorem}{Theorem}[section]

\newtheorem{lemma}[theorem]{Lemma}

\newtheorem{definition}[theorem]{Definition}
\newtheorem{assumption}[theorem]{Assumption}
\newtheorem{remark}[theorem]{Remark}
\newtheorem{problem}[theorem]{Problem}

\usepackage[utf8]{inputenc} 
\usepackage[T1]{fontenc}    
\usepackage{hyperref}       
\usepackage{url}            
\usepackage{booktabs}       
\usepackage{amsfonts}       
\usepackage{nicefrac}       
\usepackage{microtype}      
\usepackage{xcolor}         

\title{Simplex-enabled Safe Continual Learning Machine}

\author{
Hongpeng Cao\\
  Technical University of Munich, Germany \\
  \texttt{cao.hongpeng@tum.de}
   \AND
Yanbing Mao \\
   Wayne State Univerity, USA \\
  \texttt{hm9062@wayne.edu}\\
  \And
  Yihao Cai\\
  Wayne State Univerity, USA\\
  \texttt{yihao.cai@wayne.edu} \\
   \And
   Lui Sha \\
   University of Illinois Urbana-Champaign, USA \\
  \texttt{lrs@illinois.edu} \\
   \And
   Marco Caccamo \\
   Technical University of Munich, Germany \\
   \texttt{mcaccamo@tum.de} \\
}

\begin{document}

\maketitle

\begin{abstract}
This paper proposes the \textbf{SeC-Learning Machine:} Simplex-enabled safe continual learning for safety-critical autonomous systems. The SeC-learning machine is built on Simplex logic (that is, ``using simplicity to control complexity'') and physics-regulated deep reinforcement learning (Phy-DRL). The SeC-learning machine thus constitutes HP (high performance)-Student, HA (high assurance)-Teacher, and Coordinator. Specifically, the HP-Student is a pre-trained high-performance but not fully verified Phy-DRL, continuing to learn in a real plant to tune the action policy to be safe. In contrast, the HA-Teacher is a mission-reduced, physics-model-based, and verified design. As a complementary, HA-Teacher has two missions: backing up safety and correcting unsafe learning. The Coordinator triggers the interaction and the switch between HP-Student and  HA-Teacher. Powered by the three interactive components, the SeC-learning machine can i) assure lifetime safety (i.e., safety guarantee in any continual-learning stage, regardless of HP-Student's success or convergence), ii) address the Sim2Real gap, and iii) learn to tolerate unknown unknowns in real plants. The experiments on a cart-pole system and a real quadruped robot demonstrate the distinguished features of the SeC-learning machine, compared with continual learning built on state-of-the-art safe DRL frameworks with approaches to addressing the Sim2Real gap.
\end{abstract}

\section{Introduction}
Deep reinforcement learning (DRL) has been integrated into many autonomous systems (see exampels in \cref{fig:enter-label}) and have demonstrated breakthroughs in sequential and complex decision-making in broad areas, ranging from autonomous driving \cite{kendall2019learning,kiran2021deep} to chemical processes \cite{savage2021model,he2021deep} to robot locomotion \cite{ibarz2021train,levine2016end}. Such learning-integrated systems promise to revolutionize many processes in different industries with tangible economic impact \cite{market1,market2}. However, the public-facing AI incident database \cite{AID} reveals that machine learning (ML) techniques, including DRL, can deliver remarkably high performance but no safety assurance \cite{brief2021ai}. Hence, the high-performance DRL with verifiable safety assurance is even more vital today, aligning well with the market’s need for safe ML technologies.

\subsection{Related Work on Safe DRL} \label{sare}
Significant efforts have been devoted to promoting safe DRL in recent years, including developing safety-embedded rewards and residual action policies and deriving verifiable safety, as detailed below.     

The safety-embedded reward is crucial for a DRL agent to learn a high-performance action policy with verifiable safety. The control Lyapunov function (CLF) is the potential safety-embedded reward \cite{perkins2002lyapunov, berkenkamp2017safe,chang2021stabilizing, zhao2023stable}. Meanwhile, the seminal work \cite{westenbroek2022lyapunov} revealed that a CLF-like reward can enable DRL with verifiable stability. At the same time, enabling verifiable safety is achievable by extending CLF-like rewards with given safety conditions or regulations. However, systematic guidance for constructing such CLF-like rewards remains open.  

\begin{wrapfigure}{r}{0.646\textwidth}
\vspace{-0.00cm}
\begin{center}
\includegraphics[width=0.646\textwidth]{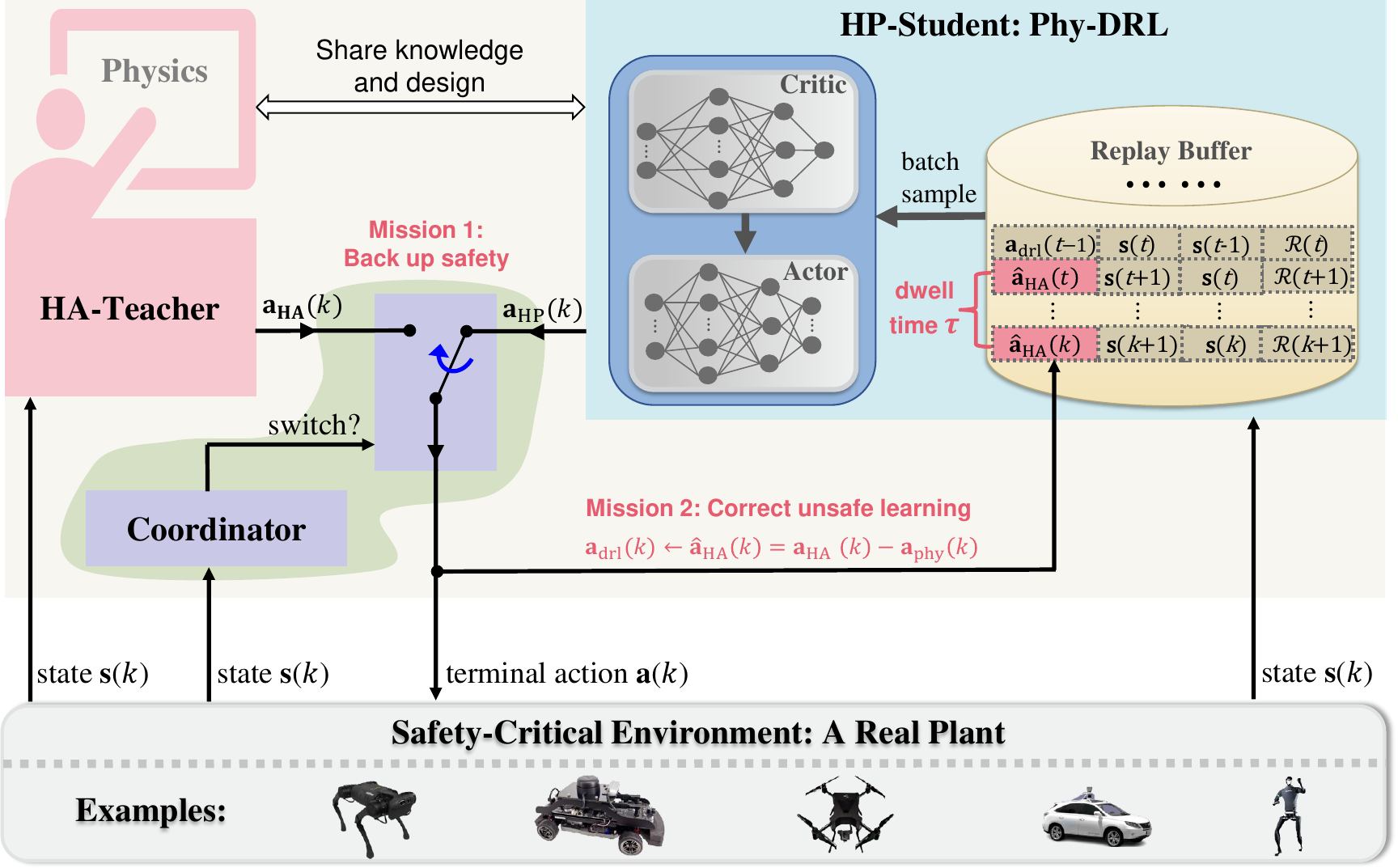}
  \end{center}
  \vspace{-0.1cm}
\caption{SeC-Learning Machine.}
  \vspace{0.cm}
  \label{fig:enter-label}
  \vspace{-0.55cm}
\end{wrapfigure}
The residual action policy is another shift in the focus of safe DRL, which integrates data-driven DRL action policy and physics-model-based action policy. The existing residual diagrams focus on stability guarantee \cite{rana2021bayesian,li2022equipping,cheng2019control,johannink2019residual}, with the exception being \cite{cheng2019end} on safety guarantee. However, the physics models considered are nonlinear and intractable, which thwarts delivering a verifiable safety guarantee or assurance, if not impossible.  

The verifiable safety (i.e., having verifiable conditions of safety guarantee) is enabled in  
the recently developed Phy-DRL (physics-regulated DRL) framework \cite{Phydrl1, Phydrl2}.  Meanwhile, Phy-DRL can address the aforementioned open problems. Summarily, Phy-DRL simplifies a nonlinear system dynamics model as an analyzable and tractable linear one. This linear model can then be a model-based guide for constructing the safety-embedded (CLF-like) reward and residual action policy. 

\subsection{Challenges and Open Problems}
Although safe DRL has developed significantly, DRL-enabled autonomous systems still face formidable safety challenges, rooting in the Sim2Real gap and unknown unknowns on real plants.  

\textbf{Challenge 1: Sim2Real Gap.} Due to the expense of data acquisition and potential safety concerns in real-world settings, the prevalent DRL involves training a policy within a simulator using synthetic data and deploying it onto the physical platforms. The discrepancy between the simulated environment and the real scenario thus leads to the Sim2Real gap that degrades the performance of the pre-trained DRL in a real plant. Numerous approaches have been developed to address Sim2Real gap \cite{Peng_2018, nagabandi2019learning, tan2018sim, yu2017preparing,cloudedge,imai2022vision, du2021autotuned,vuong2019pick, yang2022safe}.  The common aim of these approaches is to enhance realism in the simulator through, for example, domain randomization \cite{sadeghi2017cad2rl} and delay randomization (mimicking asynchronous communication and sampling) \cite{imai2022vision}. These approaches can address the Sim2Real gap to different degrees. However, the unrevealed gap still persistently impedes safety assurance if DRL (also other ML) is not trained or learning in the \underline{real} plant in \underline{real} environments, using \underline{real}-time data.

\textbf{Challenge 2: Unknown Unknowns.} The unknown unknowns generally refer to outcomes, events, circumstances, or consequences that are not known in advance and cannot be predicted in time and distributions \cite{bartz2019we}. The dynamics of many learning-enabled systems (e.g., autonomous vehicles \cite{rajamani2011vehicle} and airplanes \cite{roskam1995airplane}) are governed by conjunctive known knowns (e.g., Newton's laws of motion), known unknowns (e.g., Gaussian noise without knowing mean and variance), and unknown unknowns (due to, for example, unforeseen operating environments and DNNs' huge parameter space, intractable activation and hard-to-verify). The safety assurance also requires resilience to unknown unknowns, which is very challenging. The reasons root in characteristics of unknown unknowns: almost zero historical data and unpredictable time and distributions, leading to unavailable models for scientific discoveries and understanding.

Intuitively, enabling the continual learning of the DRL agent in the \underline{real} plant -- using its \underline{real}-time data generated in \underline{real} environments --  is the way to address Challenges 1 and 2. However, two safety problems related to the prospect of continual learning arise. 
\begin{problem}
\vspace{0.05cm}
\textbf{How do we teach or assist the DRL agent to correct his unsafe continual learning?}
\label{defa1a}
\vspace{0.00cm}
\end{problem}

\begin{problem}
\vspace{0.00cm}
\textbf{Facing an unsafe DRL agent, how can we guarantee the real-time safety of a system?}
\label{defa2a}
\vspace{0.00cm}
\end{problem}

\subsection{Contribution: Simplex-enabled Safe Continual Learning Machine} \label{cocondre}
If successful, continual learning in a real plant can directly address the Sim2Real gap and learn to tolerate unknown unknowns. However, the DRL's real-time action policy during continual learning cannot be fully verified and can have software faults. Continual learning shall run on a fault-tolerant architecture to address this safety concern. Simplex -- using simplicity to control complexity \cite{sha2001using,bak2014real} -- is a successful software architecture for complex safety-critical autonomous systems. The core of Simplex uses verified and simplified high-assurance controller to control the unverified high-performance and complex controller. Meanwhile, recently developed Phy-DRL theoretically and experimentally features fast training with verifiable safety \cite{Phydrl1, Phydrl2}. These motivate us to develop the \textbf{Simplex-enabled safe continual learning (SeC-learning) machine} to address \cref{defa1a} and \cref{defa2a}, which is built on Simplex architecture and Phy-DRL. As shown in \cref{fig:enter-label}, the SeC-learning machine constitutes HP (high performance)-Student, HA (high assurance)-Teacher, and coordinator. The HP-Student is a pre-trained Phy-DRL and continues to learn to tune the action policy to be safe in a real plant. The HA-Teacher action is a verified, mission-reduced, and physics-model-based design. As a complementary, HA-Teacher has two missions: backing up safety and teaching to correct unsafe learning of HP-Student. The coordinator triggers the switch and the interaction between the HP-Student and HA-Teacher to ensure lifetime safety (i.e., safety guarantee in any stage of continual learning, regardless of HP-Student's success or convergence).

\begin{table}[ht]
\centering
\caption{Notations throughout Paper}
\vspace{0.1cm}
\begin{tabular}{|l|l|}
\hline
$\mathbb{R}^{n}$  &   set of $\emph{n}$-dimensional real vectors        \\ \hline
$\mathbb{N}$   &  set of natural numbers       \\ \hline
$[\mathbf{x}]_{i}$ & $i$-th entry of vector $\mathbf{x}$       \\ \hline
$[\mathbf{W}]_{i,:}$ & $i$-th row of matrix $\mathbf{W}$ \\ \hline$[\mathbf{W}]_{i,j}$ & matrix $\mathbf{W}$'s element at row $i$ and column $j$\\ \hline
$\mathbf{P} \succ \!(\prec)~0$ & matrix $\mathbf{P}$ is positive (negative) definite \\ \hline  
$\top$ & matrix or vector transposition  \\ \hline
$\left|  \cdot  \right|$ & set cardinality,  or absolute value \\ \hline
$\mathbf{I}_{n}$  &   $n$-dimensional identity matrix       \\ \hline
$\mathbf{O}_{m \times n}$  &   $m \times n$-dimensional zero matrix       \\ \hline
$\lceil x \rceil$  &   $m \times n$-dimensional zero matrix       \\ \hline
\end{tabular} \label{notation}
\end{table}

\section{Preliminaries: Safety Definition}
The dynamics of a real plant can be described by 
\begin{align}
\mathbf{s}(k+1) = {\mathbf{A}} \cdot \mathbf{s}(k) + {\mathbf{B}} \cdot \mathbf{a}(k) + \mathbf{f}(\mathbf{s}(k), \mathbf{a}(k)), ~k \in \mathbb{N}  \label{realsys}
\end{align}
where $\mathbf{f}(\mathbf{s}(k), \mathbf{a}(k)) \in \mathbb{R}^{n}$ is the \underline{unknown} model mismatch, ${\mathbf{A}} \in \mathbb{R}^{n \times n}$ and ${\mathbf{B}} \in \mathbb{R}^{n \times m}$ denote \underline{known} system matrix and control structure matrix, respectively,  $\mathbf{s}(k) \in \mathbb{R}^{n}$ is real-time system state of real plant, $\mathbf{a}(k) \in \mathbb{R}^{m}$ is real-time action from SeC-learning machine. 

The actions of the SeC-learning machine aims to constrain the states of a real plant to the safety set:
\begin{align}
\text{Safety set}:~{\mathbb{X}} \triangleq \left\{ {\left. \mathbf{s} \in {\mathbb{R}^n} \right|\underline{\mathbf{v}} \le {\mathbf{D}} \cdot \mathbf{s} - \mathbf{v} \le \overline{\mathbf{v}}}, \!~\mathbf{D} \in \mathbb{R}^{h \times n},\right. \left. ~\text{with}~\mathbf{v}, \overline{\mathbf{v}}, \underline{\mathbf{v}} \in \mathbb{R}^{h}  \right\}. \label{aset2}
\end{align}
where $\mathbf{D}$, $\mathbf{v}$, $\overline{\mathbf{v}}$ and $\underline{\mathbf{v}}$ are given in advance for formulating $h \in \mathbb{N}$ safety conditions. In the SeC-learning machine, the safety set is not directly used to embed high-dimensional safety conditions (indicated by $h \in \mathbb{N}$ in \cref{aset2}) into the DRL reward since the reward is a real one-dimensional value. To address the problem, the concept of safety envelope was introduced in \cite{mao2023sl1, Phydrl1, Phydrl2}, whose condition is one-dimensional and which will be designed to be a subset of the safety set $\mathbb{X}$. 
\begin{align}
\text{Safety envelope:}~{\Omega} \triangleq \left\{ {\left. {\mathbf{s} \in {\mathbb{R}^n}} \right|{\mathbf{s}^\top}\cdot{\mathbf{P}}\cdot\mathbf{s} \le 1,~{\mathbf{P}} \succ 0} \right\}, \label{set3}
\end{align}
building on which safety definition is introduced below.  
\begin{definition}
Consider the safety envelope $\Omega$ \eqref{set3} and safety set $\mathbb{X}$ \eqref{aset2}. The real plant \eqref{realsys} is said to be safe, if given any $\mathbf{s}(1) \in \Omega  \subseteq \mathbb{X}$, the $\mathbf{s}(k) \in \Omega \subseteq \mathbb{X}$ holds for any time $k \in \mathbb{N}$.
\label{defsafety}
\end{definition}

\section{Design Overview: SeC-Learning Machine}
The proposed SeC-learning machine aims to address \cref{defa1a} and \cref{defa2a} with capabilities of assuring lifetime safety, addressing the Sim2Real gap, and tolerating unknown unknowns. To do so, as shown in \cref{fig:enter-label}, the learning machine is designed to have three critical interactive  components: 
\begin{itemize}
\vspace{-0.10cm}
    \item \textbf{HP-Student} is a pre-trained Phy-DRL (physics-regulated DRL \cite{Phydrl1, Phydrl2}) model and continues to learn in a safety-critical real plant to tune his action policy to be safe.  
    \item \textbf{HA-Teacher} is a verifiable and analyzable physics-model-based action policy with two missions: backing up the safety of a real plant and correcting unsafe learning of HP-Student. 
    \item \textbf{Coordinator} triggers the switch and interaction between HP-Student and HA-Teacher by monitoring the real-time system states. Specifically, when the real-time states of the real plant under the control of HP-Student approach the safety boundary, the coordinator triggers the switch to HA-Teacher and the correction of unsafe actions in learning. In other words, the HA-Teacher takes over the HP-Student to control the real plant to safe (i.e., backing up safety). Meanwhile, the HA-Teacher uses his safe actions to correct the HP-Student's unsafe actions in the replay buffer for learning. Once the real-time states return to a safe region, the coordinator triggers the switch back to HP-Student and terminates the learning correction. 
\end{itemize}
Next, we detail the designs of the three interactive components in \cref{opijk112345,opijk112346,ttacher}, respectively.

\section{SeC-Learning Machine: HP-Student Component} \label{opijk112345}
The HP-Student builds on Phy-DRL (physics-regulated deep reinforcement learning) proposed in \cite{Phydrl1, Phydrl2}. The critical reason is that Phy-DRL's training mission is pre-defined: searching for an action policy that renders the assigned safety envelope invariant. In this way, HP-Student can share his mission with HP-Student and Coordinator, so they can have a common goal in the learning machine: rendering the safety envelope invariant in a safety-critical real plant in the face of Sim2Real gap and unknown unknowns. We next detail the designs of HP-Student. 

\subsection{HP-Student: Residual Action Policy and Safety-embedded Reward}
Following Phy-DRL in \cite{Phydrl1, Phydrl2}, the HP-Student adopts the concurrent residual action policy and safety-embedded reward, as they can offer fast and stable training and successfully encode the safety envelope ${\Omega}$. The residual action formula is
\begin{align}
\mathbf{a}_{\text{HP}}(k) = \underbrace{\mathbf{a}_{\text{drl}}(k)}_{\text{data-driven}} + \underbrace{\mathbf{a}_{\text{phy}}(k) ~(:=  \mathbf{F} \cdot \mathbf{s}(k))}_{\text{model-based}},\label{residual}
\end{align}
where $\mathbf{a}_{\text{drl}}(k)$ denotes a date-driven action from DRL, while $\mathbf{a}_{\text{phy}}(k)$ is a physics-model-based action ($\mathbf{F}$ is our design). Meanwhile, the safety-embedded reward: 
\begin{align}
\mathcal{R}( {\mathbf{s}(k),\mathbf{a}_{\text{drl}}}(k)) = \underbrace{\mathbf{s}^\top(k) \cdot  \mathbf{H}  \cdot \mathbf{s}(k)  - {\mathbf{s}^\top(k+1)} \cdot \mathbf{P} \cdot \mathbf{s}(k+1)}_{\triangleq ~ r(\mathbf{s}(k),~\mathbf{s}(k+1))}   ~+~ w( \mathbf{s}(k),\mathbf{a}_{\text{HP}}(k)), \label{reward}
\end{align}
where the sub-reward $w( \mathbf{s}(k),\mathbf{a}(k))$ aims at high-performance operations (e.g., minimizing energy consumption of resource-limited robots \cite{yang2021learning,gangapurwala2020guided}). In contrast, the sub-reward $r(\mathbf{s}(k),\mathbf{s}(k+1))$ is safety-critical. \cref{reward} also defines: 
\begin{align}
\mathbf{H} \triangleq  {{\overline{\mathbf{A}}^\top} \cdot \mathbf{P} \cdot \overline{\mathbf{A}}}, ~~~~\text{with}~~~\overline{\mathbf{A}} \triangleq \mathbf{A} + {\mathbf{B} } \cdot {\mathbf{F}}~~\text{and}~~0 ~\prec~  \mathbf{H} ~\prec~ \alpha \cdot \mathbf{P},~~\alpha \in (0,1). \label{th10000br}
\end{align}
The matrices $\mathbf{P}$ and $\mathbf{F}$ are design variables, using the available physics-model knowledge $\left({\mathbf{A},~ \mathbf{B}} \right)$. Their automatic computations will be discussed in \cref{controlable}.

\subsection{HP-Student: Controllable Contribution Ratio} \label{controlable}
\begin{wrapfigure}{r}{0.40\textwidth}
\vspace{-0.6cm}
  \begin{center}
\includegraphics[width=0.40\textwidth]{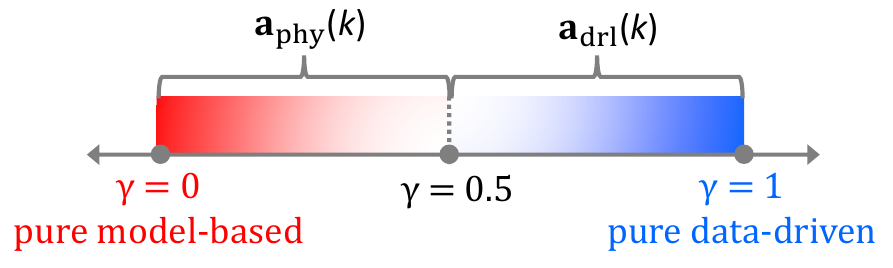}
  \end{center}
  \vspace{-0.35cm}
  \caption{Contribution ratio $\gamma$.}
  \vspace{-0.30cm}
  \label{rartio}
 \vspace{-0.20cm}
\end{wrapfigure}

In residual diagram \eqref{residual}, the contribution ratio between data-driven and model-based policies controls HP-Student's safety and system performance. To understand this, we define the contribution ratio as $\gamma  \triangleq \frac{{space \left\{ {\left| {{\mathbf{a}_{\text{drl}}}\left( k \right)} \right|,\forall k \in \mathbb{N}} \right\}}}{{space \left\{ {\left| {{\mathbf{a}_{\text{drl}}}\left( k \right)} \right|,\forall k \in \mathbb{N}} \right\} + space\left\{ {\left| {{\mathbf{a}_{\text{phy}}}\left( k \right)} \right|,\forall k \in \mathbb{N}} \right\}}}$. As depicted in \cref{rartio}, if $\gamma$ approaches 0, i.e., the physics-model-based action policy dominates the integration, featuring analyzable and verifiable behavior but limited performance. If $\gamma$ approaches 1, i.e., the data-driven policy dominates the integration, featuring high performance but hard-to-analyze and hard-to-verify. Therefore, a controllable contribution ratio is desired. The controllable space of data-driven actions is innate, as Phy-DRL is built on DDPG \cite{lillicrap2015continuous}, which directly maps states to actions within the interval $[-1, 1]$ using the Tanh activation function. The action space can be rescaled with a controllable magnitude factor $m$ to expand the space to $[-m, m]$. So, the remaining job is to enable the controllable space of physics-model-based actions: 
\begin{align}
\!\!\textbf{Action Space:}~{\mathbb{A}_{\text{phy}}} \triangleq \left\{ {\left. {\mathbf{a}_{\text{phy}} \in {\mathbb{R}^m}} ~\right|~\underline{\mathbf{z}} \le {\mathbf{C}} \cdot \mathbf{a}_{\text{phy}} - \mathbf{z} \le \overline{\mathbf{z}}, \!~\text{with}~\mathbf{C} \in \mathbb{R}^{g \times m}, ~\mathbf{z}, \overline{\mathbf{z}}, \underline{\mathbf{z}} \in \mathbb{R}^{m}  }\right\}. \label{org}
\end{align}
where $\mathbf{C}, \mathbf{z}, \overline{\mathbf{z}}$ and $\underline{\mathbf{z}}$ are users' options for controlling the space of model-based actions. \cref{residual} shows the model-based action completely depends on $\mathbf{F}$. So,we shall redesign $\mathbf{F}$ to control model-based action to space \eqref{org}. Due to the page limit, the proposed redesign for delivering $\mathbf{F}$, reward \eqref{reward} and controllable action space \eqref{org} is presented in \cref{Auxmovcma}.

\subsection{HP-Student: Continual Learning}
HP-Student is a Phy-DRL model pre-trained in a simulator or another domain, which takes the \textit{actor-critic} architecture-based DRLs such as \cite{lillicrap2015continuous} \cite{sac} for training. The pre-trained Phy-DRL model has an action policy and an action-value function. When deployed to a new safety-critical environment, the SeC-learning machine enables the pre-trained policy to continually and safely search for a safe policy that maximizes the expected return. 

Sampling efficiency is one of the important considerations for continual learning in the real world. Experience replay (ER) \cite{andrychowicz2017hindsight} allows off-policy algorithms to reuse the experience collected in the past, greatly improving the sampling efficiency and avoiding forgetting the learned knowledge \cite{khetarpal2022towards}. ER is also beneficial for breaking the correlation between adjacent transitions to avoid sampling bias for a stable learning process. Those features are very important in continual learning, where online data is limited due to the expensive interaction on the physical system. During the online inference, we continuously store the real transitions realized by safe high-performance action of HP-Student or corrected unsafe data-driven action by HA-Teacher to the replay buffer. Specifically, as illustrated in \cref{fig:enter-label}, if the HP-Student's action $\mathbf{a}_{\text{HP}}(k)$ leads to unsafe behavior of a real plant, HA-Teacher takes over his role of controlling real plant to be safe, and corrects his unsafe data-driven action $\mathbf{a}_{\text{drl}}(k)$ to $\widehat{\mathbf{a}}_{\text{HA}}(k)$ according to 
\begin{align}
\mathbf{a}_{\text{drl}}(k) \leftarrow \widehat{\mathbf{a}}_{\text{HA}}(k) \triangleq {\mathbf{a}}_{\text{HA}}(k) - \mathbf{a}_{\text{phy}}(k), \label{correctdef}
\end{align}
where $ \mathbf{a}_{\text{phy}}(k)$ is HP-student's model-based action in residual action policy \eqref{residual}, and ${\mathbf{a}}_{\text{HA}}(k)$ is the action from HA-Teacher, whose design is presented in \cref{ttacher}. Meanwhile, the online learning process will uniformly sample a minibatch of transitions for learning or training \cite{fujimoto2018addressing}. 

\begin{remark}
\cref{correctdef} indicates that for HP-Student's residual action policy \eqref{residual}, the correction in performed only on data-driven action $\mathbf{a}_{\text{drl}}(k)$, i.e., not including model-based action $\mathbf{a}_{\text{phy}}(k)$. The reason is that although the model-based design has limited performance and a small safe operation region due to model mismatch, it is analyzable and verifiable, and its policy is invariant because of his invariant physics-model knowledge $\left({\mathbf{A},~ \mathbf{B}} \right)$.  
\end{remark}

\section{SeC-Learning Machine: Coordinator Component} \label{opijk112346}
The Coordinator triggers the switch between HP-Student and HA-Teacher to control the real plant by monitoring the system's state in real-time. The switching logic of terminal action applied to a real plant is described below. 
\begin{align}
\mathbf{a}(k) = \begin{cases}
		\mathbf{a}_{\text{HA}}(k), &\text{if}~{\mathbf{s}^\top}(t) \cdot \mathbf{P} \cdot \mathbf{s}(t) \ge \varepsilon < 1 ~\text{and}~ t \le k \le t + \tau,  ~\tau \in \mathbb{N}\\ 
        \mathbf{a}_{\text{HP}}(k),      &\text{otherwise}  
	\end{cases} \label{coordinator}
\end{align}
synchronizing with which is the correcting logic of HP-Student's unsafe actions in his replay buffer: 
\begin{align}
\mathbf{a}_{\text{drl}}(k) = \begin{cases}
		\widehat{\mathbf{a}}_{\text{HA}}(k), &\text{if}~{\mathbf{s}^\top}(t) \cdot \mathbf{P} \cdot \mathbf{s}(t) \ge \varepsilon < 1 ~\text{and}~ t \le k \le t + \tau,  ~\tau \in \mathbb{N}\\ 
        \mathbf{a}_{\text{drl}}(k),      &\text{otherwise}  
	\end{cases} \label{coordinatorcl}
\end{align}
where $\widehat{\mathbf{a}}_{\text{HA}}(k)$ is the corrected unsafe action by HA-Teacher, defined in \cref{correctdef}. Noting that 
\begin{wrapfigure}{r}{0.58\textwidth}
\vspace{-0.45cm}
  \begin{center}
\includegraphics[width=0.58\textwidth]{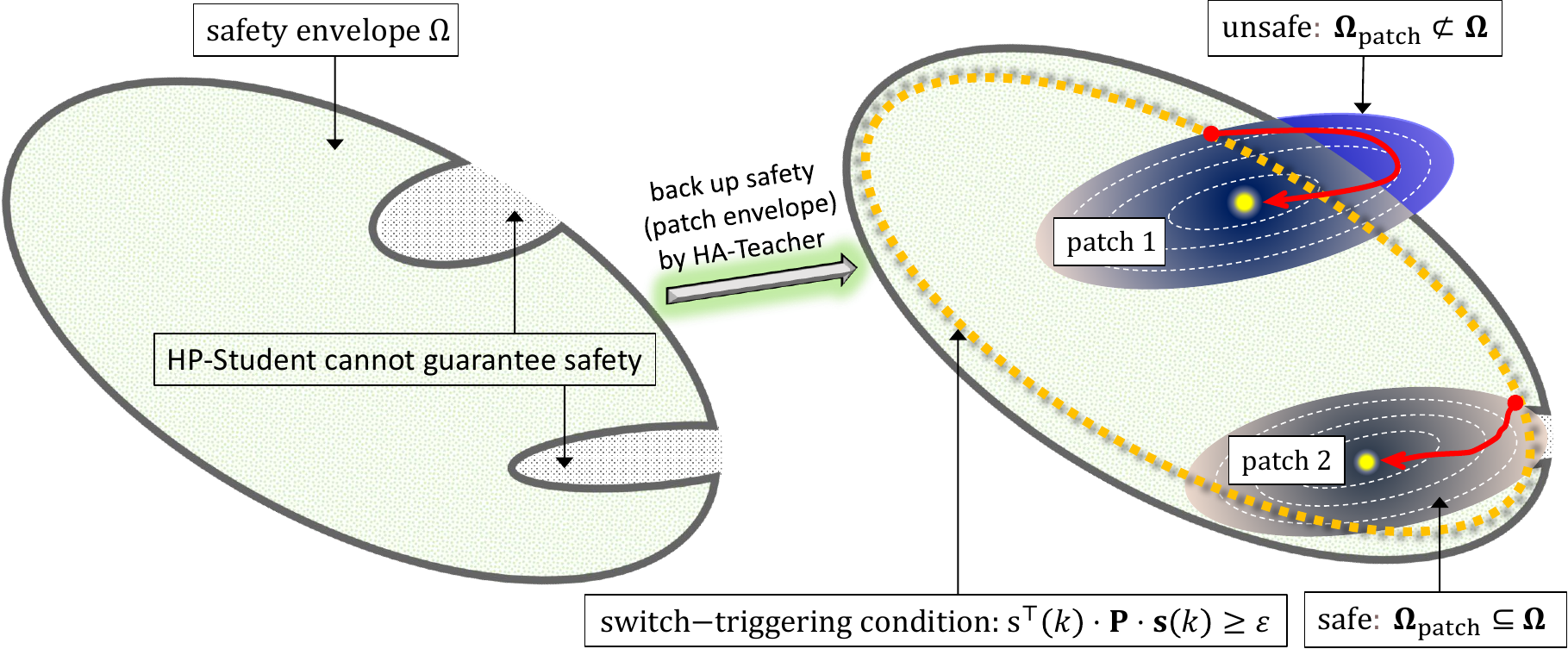}
  \end{center}
  \vspace{-0.10cm}
\caption{System phase behavior.}
  \vspace{-0.10cm}
  \label{behavior}
  \vspace{-0.4cm}
\end{wrapfigure}
real plant initially operates from a safety envelope, we can observe from \cref{coordinator} and \cref{coordinatorcl} that the triggering condition of the switch from HP-Student to HA-Teacher and the unsafe action correction is ${\mathbf{s}^\top}(t) \cdot \mathbf{P} \cdot \mathbf{s}( t) \ge \varepsilon$. In addition, the requirement `$t \le k \le t + \tau$' indicates that once HA-Teacher takes over the role of HP-Student, he has a dwell time $\tau \in \mathbb{N}$; after the dwell time, he returns the control role to HP-Student. The dwell time has two considerations:
\begin{itemize}
\vspace{-0.21cm}
\item Constraining the states of a real plant to a very safe space (i.e., approaching the center of a safety envelope patch, illustrated in \cref{behavior}) to preserve sufficient fault-tolerant space for HP-Student's continual learning. 
\vspace{-0.05cm}
\item Collecting sufficient data of safe actions and states to correct the unsafe continual learning of HP-Student (see corrected $\tau \in \mathbb{N}$ unsafe actions in replay buffer in \cref{fig:enter-label}). 
\vspace{-0.15cm}
\end{itemize}
The dwell time $\tau$ is a design of HA-Teacher, which is carried out in \cref{dwelltime} in \cref{sudell}. 

\begin{remark}[\textbf{Parallel Running}] \label{parlasdffg}
Finally, we note that the HA-Teacher and HP-Student in the SeC-learning machine run in parallel. This configuration guarantees
that when the Coordinator activates the HA Teacher to back up safety and correct unsafe learning, his actions can be available immediately. If operating without parallel running, the system can lose control due to the time delay in sampling, communication, and computation. 
\end{remark}

\section{SeC-Learning Machine: HA-Teacher Component}\label{ttacher}
HA-Teacher has tasks in the SeC-leaning machine: 
\begin{itemize}
\vspace{-0.15cm}
    \item \textbf{Back up Safety via Patching Safety Envelope}. As soon as the HP-Student leads to an unsafe real-time system behavior, the HA-Teacher intervenes to safely control the system through patching the safety envelope, depicting in \cref{behavior}. 
    \item \textbf{Correct Unsafe Learning}. The HA-Teacher uses safe actions to correct the real-time and potentially (in dwell time $\tau$ horizon) unsafe actions of the HP-Student for learning.
\end{itemize}

HA-Teacher is a physics-model-based design whose function is reduced to be safety-critical only. Compared with the HP-Student, the HA-Teacher has relatively rich dynamics knowledge about real plants. Hereto, the dynamics model leveraged by HA-Teacher updates from \eqref{realsys} as 
\begin{align}
\mathbf{s}(k+1) = {\mathbf{A}}(\mathbf{s}(k)) \cdot \mathbf{s}(k) + {\mathbf{B}}(\mathbf{s}(k)) \cdot \mathbf{a}_{\text{HA}}(k) + \mathbf{g}(\mathbf{s}(k)), ~k \in \mathbb{N}  \label{realsyshadesign}
\end{align}
where $\mathbf{g}(\mathbf{s}(k)) \in \mathbb{R}^{n}$ is the \underline{unknown} model mismatch for HA-Teacher. The physics-model knowledge available to HA-Teacher is thus $({\mathbf{A}}(\mathbf{s}(k)), {\mathbf{B}}(\mathbf{s}(k)))$. Because of the \textbf{Parallel Running} configuration, HA-Teacher is always active to compute a safe action policy with control goal as 
\begin{align}
\text{Patch center}:~~\overline{\mathbf{s}}^* = \chi \cdot \mathbf{s}(k)  ~~~\text{with}~~\chi \in (-1,1). \label{goal}
\end{align} 
In other words, with real-time sensor data and physics-model knowledge, the HA-Teacher is always in `active' status to compute a model-based action policy to control the real plant to reach the goal $\overline{\mathbf{s}}^*$. To achieve this, HA-Teacher first obtains tracking-error dynamics from \cref{realsyshadesign} as   
\begin{align}
\mathbf{e}(k+1) = {\mathbf{A}(\mathbf{s}(k))} \cdot \mathbf{e}(k) + {\mathbf{B}}(\mathbf{s}(k)) \cdot \mathbf{a}_{\text{HA}}(k) + \mathbf{h}(\mathbf{e}(k)),  ~\text{with}~~\mathbf{e}(k) \triangleq \mathbf{s}(k) - \overline{\mathbf{s}}^{*} \label{realsyserror}
\end{align} 
where $\mathbf{h}(\mathbf{e}(k)) \in \mathbb{R}^{n}$ denotes unknown model mismatch, and HA-Teacher's action policy  is   
\begin{align}
\mathbf{a}_{\text{HA}}(k) = \widehat{\mathbf{F}} \cdot \mathbf{e}(k),   \label{hacteacherpolicy}
\end{align}
whose aim is to track the goal $\overline{\mathbf{s}}^*$ while constraining system states to the envelope patch: 
\begin{align}
\text{Envelope patch:}~{\Omega_{\text{patch}}} \!\triangleq\! \left\{ {\left. \mathbf{s} ~\right|{{( {\mathbf{s} \!-\! \overline{\mathbf{s}}^*})}^\top} \cdot \widehat{\mathbf{P}} \cdot ( {\mathbf{s} \!-\! \overline{\mathbf{s}}^*}) \le {{( {1 \!-\! \chi })}^2} \cdot {{\mathbf{s}^\top(k)}} \cdot \widehat{\mathbf{P}} \cdot \mathbf{s}(k)}, ~\widehat{\mathbf{P}} \!\succ\! 0 \right\}
\!. \label{hacset}
\end{align}

The matrices $\widehat{\mathbf{F}}$ and $\widehat{\mathbf{P}}$ in \cref{hacteacherpolicy} and \cref{hacset} are HA-Teacher's design variables for backing up safety and correcting unsafe learning. To have them, we present a practical and common assumption on unknown model mismatch for computing them. 
\begin{assumption}
The model mismatch in $\mathbf{h}(\cdot)$ in \cref{realsyserror} is locally Lipschitz in set $\Omega_{\text{patch}}$, i.e., 
\begin{align}
{\left( {{\bf{h}}\left( {{{\bf{e}}_1}} \right) - {\bf{h}}\left( {{{\bf{e}}_2}} \right)} \right)^\top} \cdot \mathbf{P} \cdot \left( {{\bf{h}}\left( {{{\bf{e}}_1}} \right) - {\bf{h}}\left( {{{\bf{e}}_2}} \right)} \right) \le \kappa  \cdot {\left( {{{\bf{e}}_1} - {{\bf{e}}_2}} \right)^\top} \cdot \mathbf{P} \cdot \left( {{{\bf{e}}_1} - {{\bf{e}}_2}} \right),~\forall {\mathbf{e}_1}, {\mathbf{e}_2} \in \Omega_{\text{patch}},  \nonumber
\end{align}
where $\mathbf{P} \succ 0$ is shared by HP-Student, which defines his safety envelope \eqref{set3} and safety-embedded reward \eqref{reward}. \label{assm}
\end{assumption}

The designs of $\widehat{\mathbf{F}}$ and $\widehat{\mathbf{P}}$ for delivering HA-Teacher's capabilities of backing up safety and correcting unsafe learning are formally presented in the following theorem, whose proof appears in \cref{thm10007paux}. 
\begin{theorem}
Consider the HA-Teacher's action policy \eqref{hacteacherpolicy} and the envelope patch $\Omega_{\text{patch}}$ \eqref{hacset}, whose matrices $\widehat{\mathbf{F}}$ and $\widehat{\mathbf{P}}$ are computed according to 
\begin{align}
\widehat{\mathbf{F}} = \widehat{\mathbf{R}} \cdot \widehat{\mathbf{Q}}^{-1}, ~~~~ \widehat{\mathbf{P}} = \widehat{\mathbf{Q}}^{-1}, \label{acc0}
\end{align}
with the matrices $\widehat{\mathbf{R}}$ and $\widehat{\mathbf{Q}}$ satisfying
\begin{align}
&\widehat{\mathbf{Q}} \cdot {\mathbf{P}} \prec {\mathbf{I}_n} \prec \eta  \cdot \widehat{\mathbf{Q}} \cdot {\mathbf{P}}, ~~\text{with a given}~\eta > 1   \label{cc0}\\
&\left[ {\begin{array}{*{20}{c}}
(\beta - \kappa \cdot \eta \cdot (1 + \frac{1}{\omega})) \cdot \widehat{\mathbf{Q}} & \widehat{\mathbf{Q}} \cdot \mathbf{A}^\top(\mathbf{s}(k)) + \widehat{\mathbf{R}}^\top \cdot \mathbf{B}^\top(\mathbf{s}(k))\\
\mathbf{A}(\mathbf{s}(k)) \cdot \widehat{\mathbf{Q}} + \mathbf{B}(\mathbf{s}(k)) \cdot \widehat{\mathbf{R}} &  \frac{\widehat{\mathbf{Q}}}{1+\omega}
\end{array}} \right] \succ 0,\label{cc3}
\end{align}
where $\beta \in (0,1)$ and $\omega > 0$ are given parameters. Under \cref{assm}, the system \eqref{realsyserror} controlled by HA-Teacher has the following properties: 
\begin{enumerate}
\item The ${\mathbf{e}^\top}\left( k+1 \right) \cdot \widehat{\mathbf{P}} \cdot \mathbf{e}\left( k+1 \right) \le \beta \cdot {\mathbf{e}^\top}\left( k \right) \cdot \widehat{\mathbf{P}} \cdot \mathbf{e}\left( k \right)$ holds for any $k \in \mathbb{N}$. \label{stat1}
\item The $\Omega_{\text{patch}} \subseteq {\Omega}$ holds if the parameters $\eta$ in \cref{cc0} and $\chi$ in \cref{goal} satisfy
\begin{align}
{\left( {1 - \chi } \right)^2} \cdot \eta \cdot \varepsilon + {\chi ^2}\cdot \varepsilon \le 0.5.\label{cc300}
\end{align}
\label{stat2}
\end{enumerate}
\label{thm10007p}
\end{theorem}

\begin{remark}[\textbf{Suggestion from \cref{stat1}: Dwell time $\tau$ of HA-Teacher}] We obtain from \cref{stat1} that  
\begin{align}
{{\bf{e}}^ \top }\left( k \right) \cdot \widehat {\bf{P}} \cdot {\bf{e}}\left( k \right) \le {\beta ^{k - t}} \cdot {\left( {{{\bf{e}}^*}} \right)^ \top } \cdot \widehat {\bf{P}} \cdot \left( {{{\bf{e}}^*}} \right), \label{ccrm}
\end{align}
wherein the $t$ and ${\bf{e}}^*$ denote the activation time of HA-Teacher and the initial distance with goal $\overline{\mathbf{s}}^*$ \eqref{goal}, respectively. The real-time tracking error ${\bf{e}}\left( k \right)$ can be understood as the distance to the goal $\overline{\mathbf{s}}^*$ (i.e., the center of envelope patch). Meanwhile, we can use ${{\bf{e}}^ \top }\left( k \right) \cdot \widehat {\bf{P}} \cdot {\bf{e}}\left( k \right)$ as distance function or performance metric. Hereto, we consider a safety criteria as ${{\bf{e}}^ \top }\left( k \right) \cdot \widehat {\bf{P}} \cdot {\bf{e}}\left( k \right) \le \delta$. Illustrated in \cref{behavior}, a very small $\delta$ means ``being very close to patch center" and that HA-Teacher can preserve sufficient fault-tolerance space for HA-Teacher's continual learning. Meanwhile, when the preset safety criteria hold, HP-Student takes back the control role from HA-Teacher. According to \cref{ccrm}, the condition of HA-Teacher's dwell time for satisfying the safety criteria is
\begin{align}
\tau  \ge \left\lceil {\frac{{\ln \delta  - \ln {{\left( {{{\bf{e}}^*}} \right)}^ \top } \cdot \widehat {\bf{P}} \cdot \left( {{{\bf{e}}^*}} \right)}}{{\ln \beta }}} \right\rceil. \label{dwelltime}
\end{align}
In other words, if $k-t \ge \tau$ and $\tau$ satisfies \cref{dwelltime}, we have ${{\bf{e}}^ \top }\left( k \right) \cdot \widehat {\bf{P}} \cdot {\bf{e}}\left( k \right) \le \delta$. \label{sudell}
\end{remark}

\begin{remark}[\textbf{Suggestion from \cref{stat2}: Backing up safety by envelope patches}] 
Condition \eqref{cc3} is for backing up safety when the HP-Student's actions cannot guarantee real plants' real-time safety. As depicted in \cref{behavior}, if only the property in \cref{stat1} of \cref{thm10007p} holds, the HA-Teacher cannot back up safety due to the unsafe region, highlighted in blue color. Only after the \cref{stat2} of \cref{thm10007p} holds can HA-Teacher achieve the concurrent missions of backing up safety and correcting learning, as the unsafe regions (i.e., regions outside the safety envelope) disappear.
\end{remark}

\begin{remark}[\textbf{Fast Computation}] The $\widehat{\mathbf{F}}$ and $\widehat{\mathbf{P}}$ are automatically computed from inequalities \eqref{cc0} and \eqref{cc3} by LMI toolbox \cite{gahinet1994lmi, boyd1994linear}. Its computation time is usually significantly small (e.g., 0.01 sec -- 0.04 sec), and its influence can be ignored in the configuration of parallel running (see \cref{parlasdffg}). 
\end{remark}

\section{Experiment} \label{exppjus}
We perform the experiments on a cart-pole system (simulator) and a real quadruped robot. 

\subsection{Cart-Pole System}
This experiment aims to demonstrate the effectiveness of the SeC-Learning Machine from perspectives of concurrent safety and training performances in the face of the Sim2Real gap. The pre-training of HP-Student (i.e., Phy-DRL) is performed on the simulator provided in Open-AI Gym~\cite{OpenAI-gym}. To address the Sim2Real gap, the pre-training adopts domain randomization \cite{sadeghi2017cad2rl,nagabandi2019learning} through introducing random force disturbances and randomizing friction force. We also use the simulator to mimic a real plant whose Sim2Real gap is intentionally created by inducing a friction force that is out of the distribution of the random friction force used in pre-training. 

The system's mechanical analog is characterized by the pendulum's angle $\theta$, the cart's position $x$, and their velocities $\omega = \dot \theta$ and $v = \dot x$. The mission of HP-Student is to stabilize the pendulum at equilibrium $\mathbf{s}^* = [0,0,0,0]^\top$ while constraining the system state to safety set: 
\vspace{-0.0cm}
\begin{align} 
{\mathbb{X}} =  \left\{ {\left. {\mathbf{s} \in {\mathbb{R}^4}} \right| -0.9 \le x \le 0.9, ~-0.8 <  \theta  < 0.8}\right\}. \label{safetysetexp}
\end{align}
Given the safety set, we set the space of physics-model-based action policy as 
\begin{align}
{\mathbb{A}_{\text{phy}}} \triangleq \left\{ {\left. {\mathbf{a}_{\text{phy}} \in {\mathbb{R}}} ~\right|~-25 \le \mathbf{a}_{\text{phy}} \le 25  }\right\}. \label{orgact}
\end{align}
With them, the designs of HP-Student and HA-Teacher are presented in \cref{HP} and \cref{HA}, respectively.

\begin{wrapfigure}{r}{0.35\textwidth}
\vspace{-0.1cm}
  \begin{center}
\includegraphics[width=0.35\textwidth]{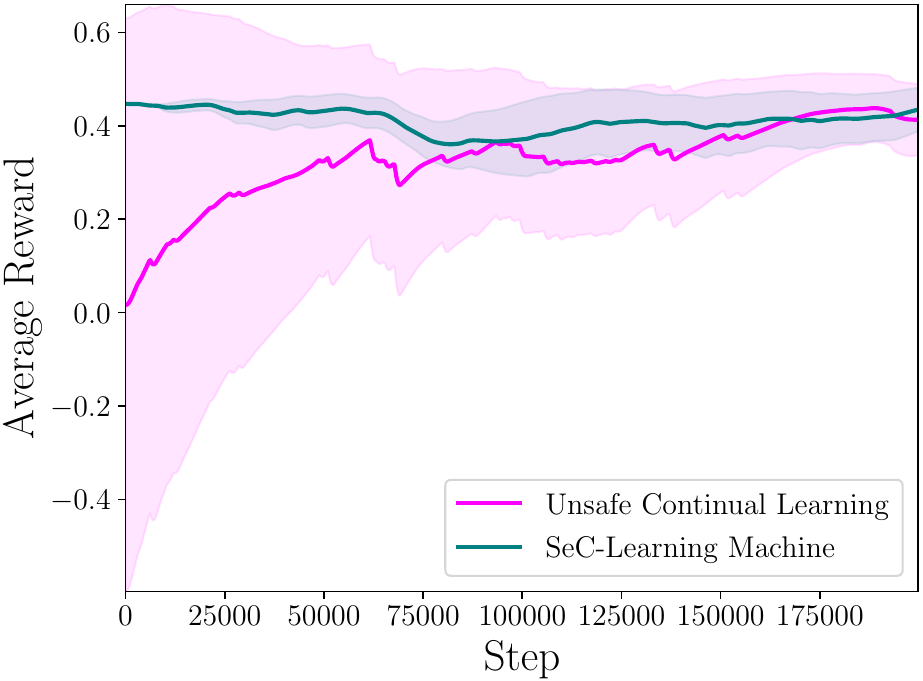}
  \end{center}
  \vspace{-0.40cm}
  \caption{Episode reward.}
  \vspace{-0.4cm}
  \label{rewardoo}
\end{wrapfigure}
For continual learning in mimicked real plants, the maximum length of one episode is 1500 steps. For the comparisons, we consider three models: \textit{`Pre-trained Policy'} (i.e., Phy-DRL pre-trained in a simulator, using randomization approaches for addressing Sim2Real gap), \textit{`Unsafe Continual Learning'} (i.e., pre-trained Phy-DRL continues to learn in the real plant but no Simplex support), and our proposed \textit{`SeC-Learning Machine'}. A very distinguished feature of our SeC-Learning Machine is lifetime safety, i.e., a safety guarantee in any stage of continual learning, regardless of the success of HP-Student. To demonstrate this and have fair comparisons, HP-Student in the `SeC-Learning Machine' and the `Unsafe Continual Learning' model are picked after training for \underline{\bf{only two episodes}}. Given the three different initial conditions, phases plots of these three models are shown in \cref{quaapcv}.  To further convincingly demonstrate the feature, additional phase plots are shown in \cref{ep3}, \cref{ep4}, and \cref{ep5}, where the models are picked after training for only \textbf{three episodes}, \textbf{four episodes}, \textbf{five episodes}, respectively. Meanwhile, the reward's training curves (five random seeds) are shown in \cref{rewardoo}. Observing \cref{quaapcv,rewardoo,ep3,ep4,ep5}, we conclude: 
\begin{itemize}
\vspace{-0.20cm}
    \item In the face of a large Sim2Real gap with a pre-training environment, the SeC-Learning Machine can always guarantee safety (system states never leave the safety envelope; see red curves in \cref{quaapcv,ep3,ep4,ep5}) in any stage of continual learning. In contrast, the pre-trained model and continual learning without Simplex cannot guarantee safety (system states left the safety envelope; see blue and green curves in \cref{quaapcv,ep3,ep4,ep5}). 
\vspace{-0.05cm}
    \item Unsafe action correction and safety backup from HA-Teacher lead to remarkably stable and fast training, compared with continual learning without Simplex logic (see \cref{rewardoo}).
\end{itemize}

\begin{figure}
    \centering
    \subfloat[$\text{Initial Condition 1}$]{\includegraphics[width=0.322\textwidth]{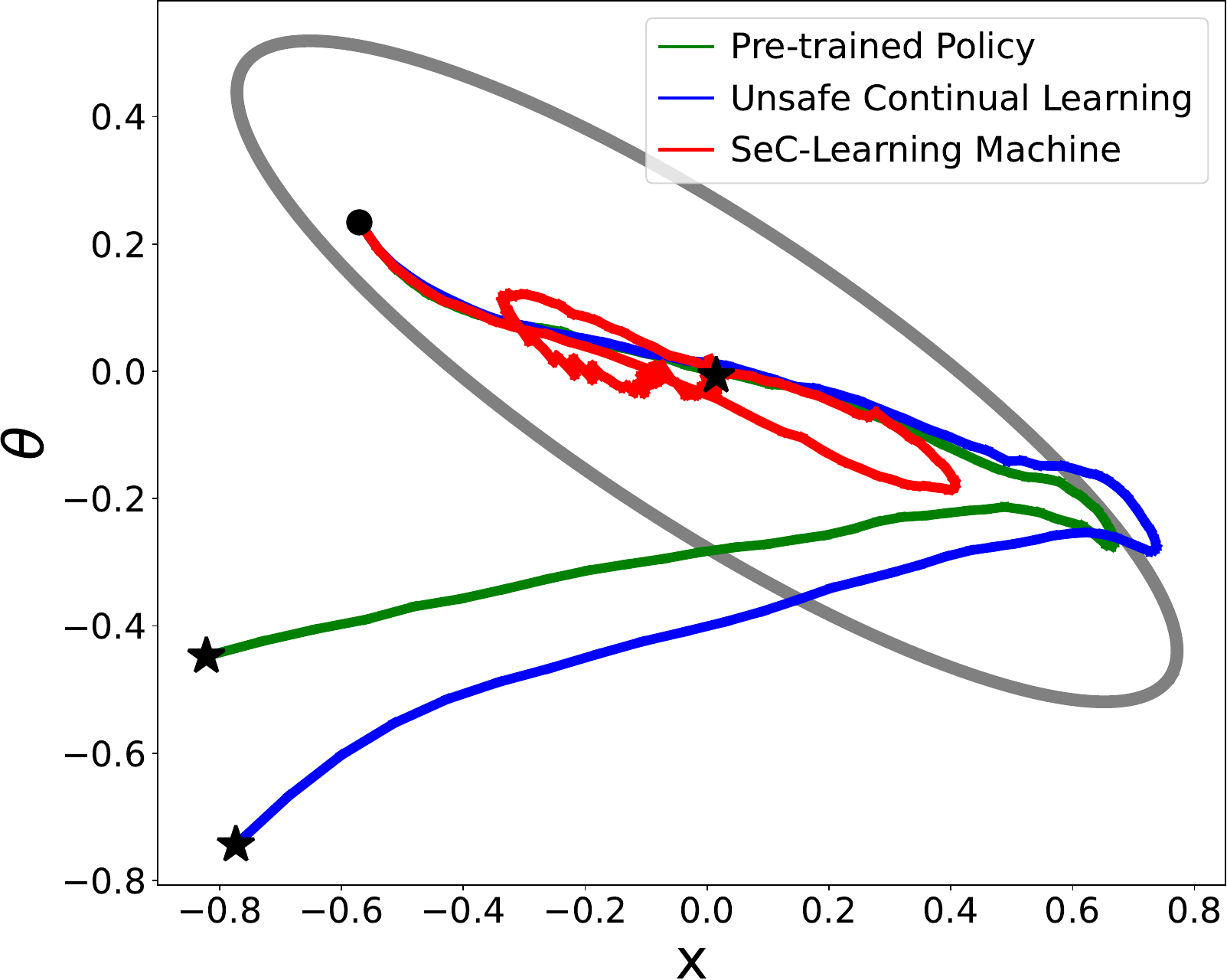}} \hspace{0.15cm}
    \centering
    \subfloat[$\text{Initial Condition 2}$]{\includegraphics[width=0.322\textwidth]{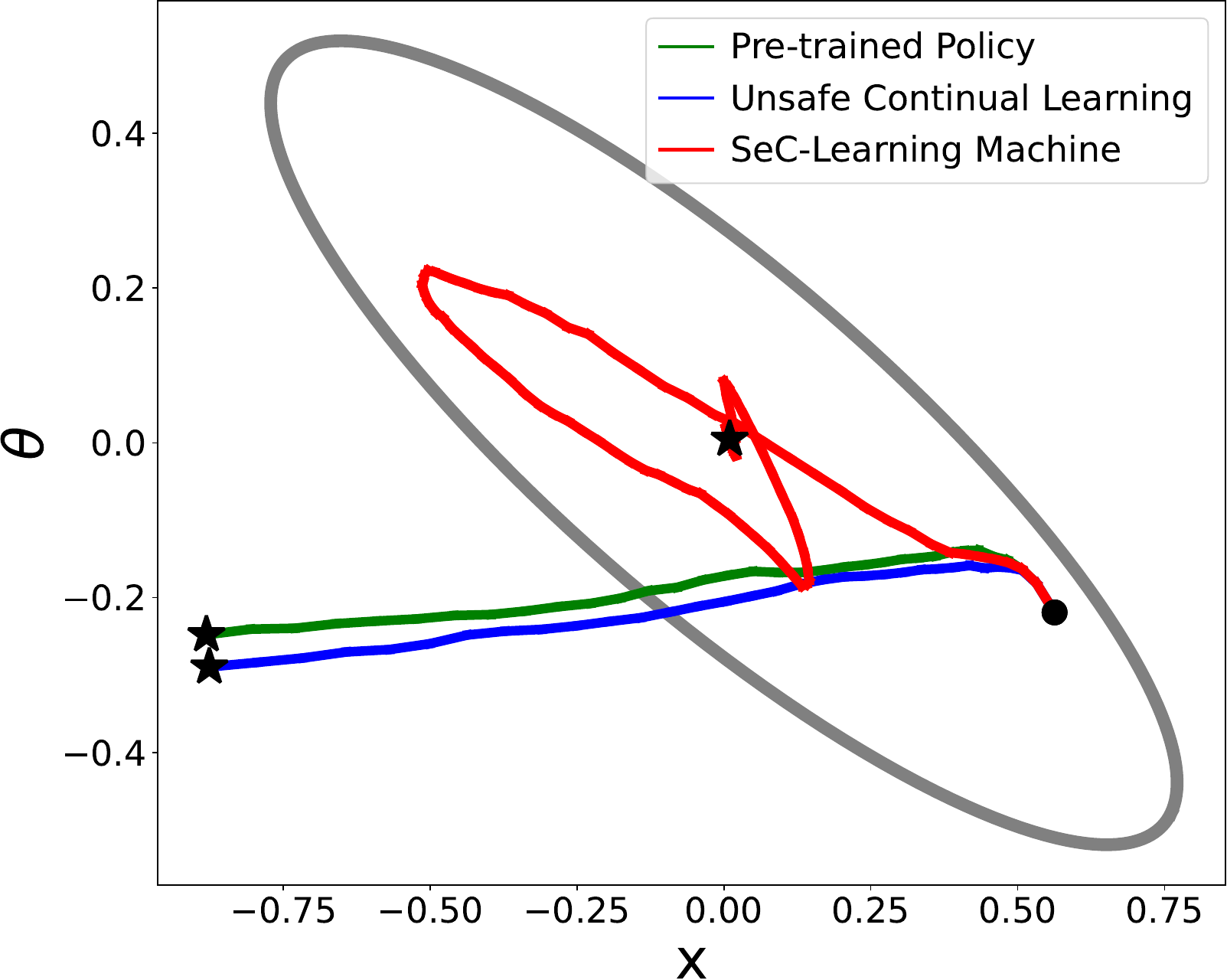}}   \hspace{0.15cm}  
    \centering
    \subfloat[$\text{Initial Condition 3}$]{\includegraphics[width=0.322\textwidth]{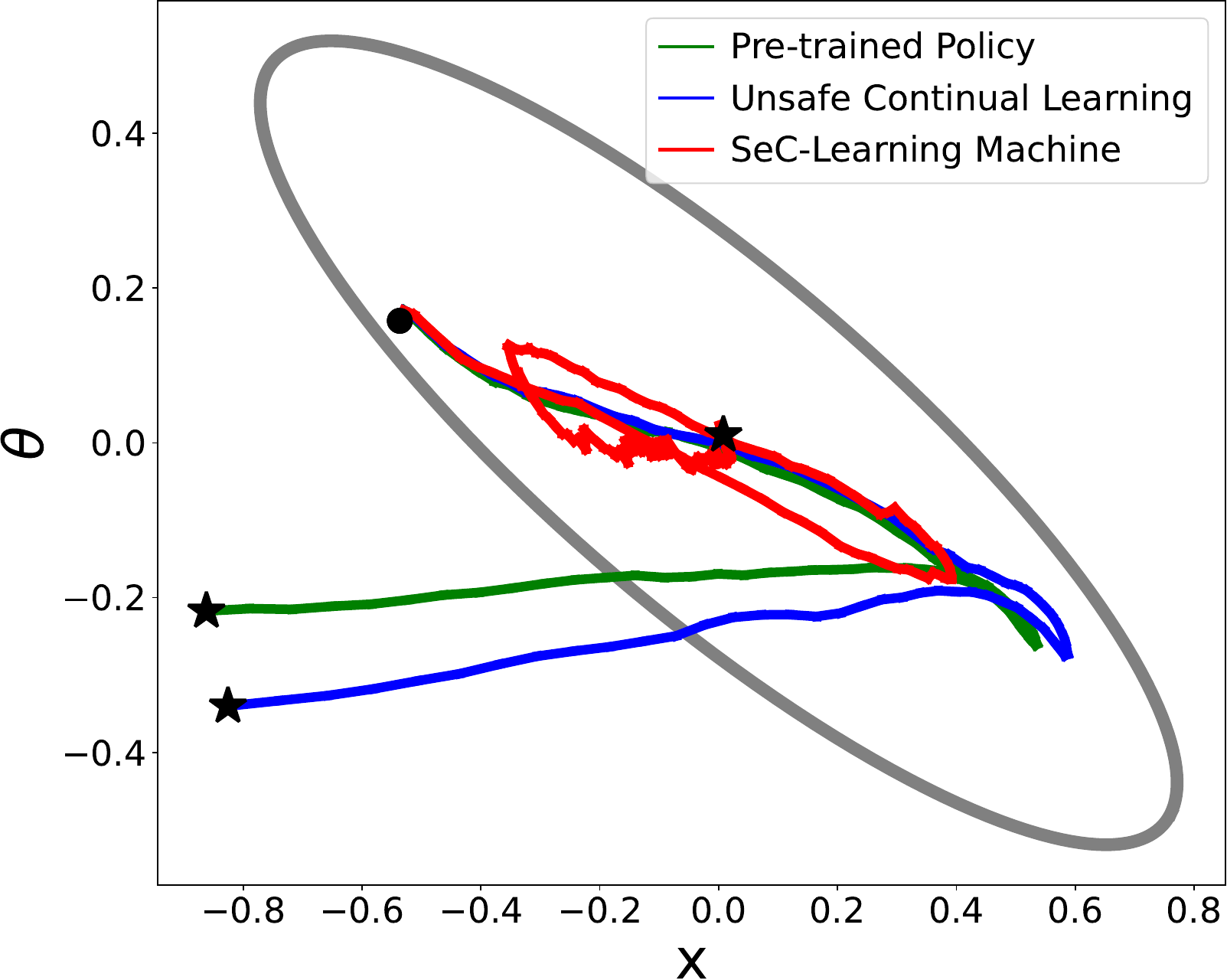}} 
    \vspace{-0.14cm}
\caption{\textbf{Two Episodes}. Phase plots, given the same initial condition. The black dot and star denote the initial condition and final location, respectively.}
\vspace{-0.2cm}
\label{quaapcv}
\end{figure}

\subsection{Real Quadruped Robot} \label{ccrobnt}
The mission of the action policy is to control the robot's COM height, COM x-velocity, and other states to track the corresponding commands $r_{v_x}$,  $r_{h}$, and zeros, under safety constraints: $\left| {\text{yaw}} \right| \le 0.2$ rad, $\left| {\text{CoM x-velocity} - r_{v_x}} \right| \le |r_{v_x}|$, $\left| {\text{CoM z-height} - r_{h}} \right| \le 0.12$ $\text{m}$, and  $\left| {\text{CoM yaw velocity}} \right| \le 0.3$ rad/s. For the Coordinator's switching logic \eqref{coordinator} and correcting logic \eqref{coordinatorcl}, we let $ \varepsilon = 0.65$ and $ \tau = 10$. The designs of HP-Student and HA-Teacher are presented in \cref{HPdog} and \cref{HAdog}, respectively. For HP-Student's pre-training of addressing the Sim2Real gap, we consider the approaches of delay randomization proposed in \cite{imai2022vision} and force randomization. During pre-taining in the simulator, the ground friction is set as 0.7, and $r_{v_x}$ = 0.6 m/s and $r_{h}$ = 0.24 m. In the real quadruped robot, one episode is defined as ``\underline{running the robot for 15 sec}." To better demonstrate the performance of the SeC-Learning Machine, the velocity command for the real robot is $r_{v_x}$ = 0.35 m/s, which very different from the one for HP-Student's training in the simulator. 

We compare three models in the real quadruped robot: \textit{`Phy-DRL'} (a pre-trained Phy-DRL model in the simulator, directly deployed on the real robot), \textit{`Continual Learning'} (a pre-trained Phy-DRL model in simulator that continues learning in the real robot for \underline{20 episodes} but without Simplex logic), and our proposed \textit{`SeC-Learning Machine'} in the \underline{1st episode} in real robot. We consider the 1st episode for the `SeC-Learning Machine' model, owning to its claimed feature of lifetime safety, i.e., safety guarantee in any stage of continual learning in a real plant regardless of HP-Student’s convergence or success. Therefore, showing system trajectories in the 1st episode will be most convincing, as HP-Student cannot converge so fast for a safe action policy within only one episode. We also note that Phy-DRL or HP-Student is pre-trained well in the simulator.

The comparisons are first viewed from the trajectories of the robot's CoM height and x-velocity regarding tracking performance and safety guarantee, which are shown in \cref{quaapcvopt} (a) and (b). Meanwhile, the comparison video of the three models in real robots is available at \href{https://www.youtube.com/watch?v=ZNpJULgLnh0}{\color{blue} real-robot-video-1 link}.
In addition, the real robot's trajectories under the control of the SeC-Learning Machine in the 5th episode, 10th episode, 15th episode, and 20th episode are shown in \cref{ep5sec,ep10sec,ep15sec,ep20sec} in \cref{addtextsim}, respectively. These figures straightforwardly depict that the SeC-Learning Machine guarantees the safety of real robots across all selected episodes or stages of continual learning. Observing \cref{quaapcvopt} (a) and (b), \cref{ep5sec,ep10sec,ep15sec,ep20sec}, and the comparison video, we discover that Phy-DRL (i.e., HP-Student), a well pre-trained model in simulator, cannot guarantee the safety of the real robot due to the existing Sim2Real gap or unknown unknowns in the physical environment that the delay and force randomization failed to capture. Thus, continuous real-time learning is needed for the real robot. However, without Simplex logic, the safety of real robot during continual learning is not guaranteed. The SeC-Learning Machine successfully addresses these challenges by ensuring lifetime safety for continual real-time learning.  

\begin{figure}
    \centering
    \subfloat[$\text{CoM x-Velocity}$]{\includegraphics[width=0.300\textwidth]{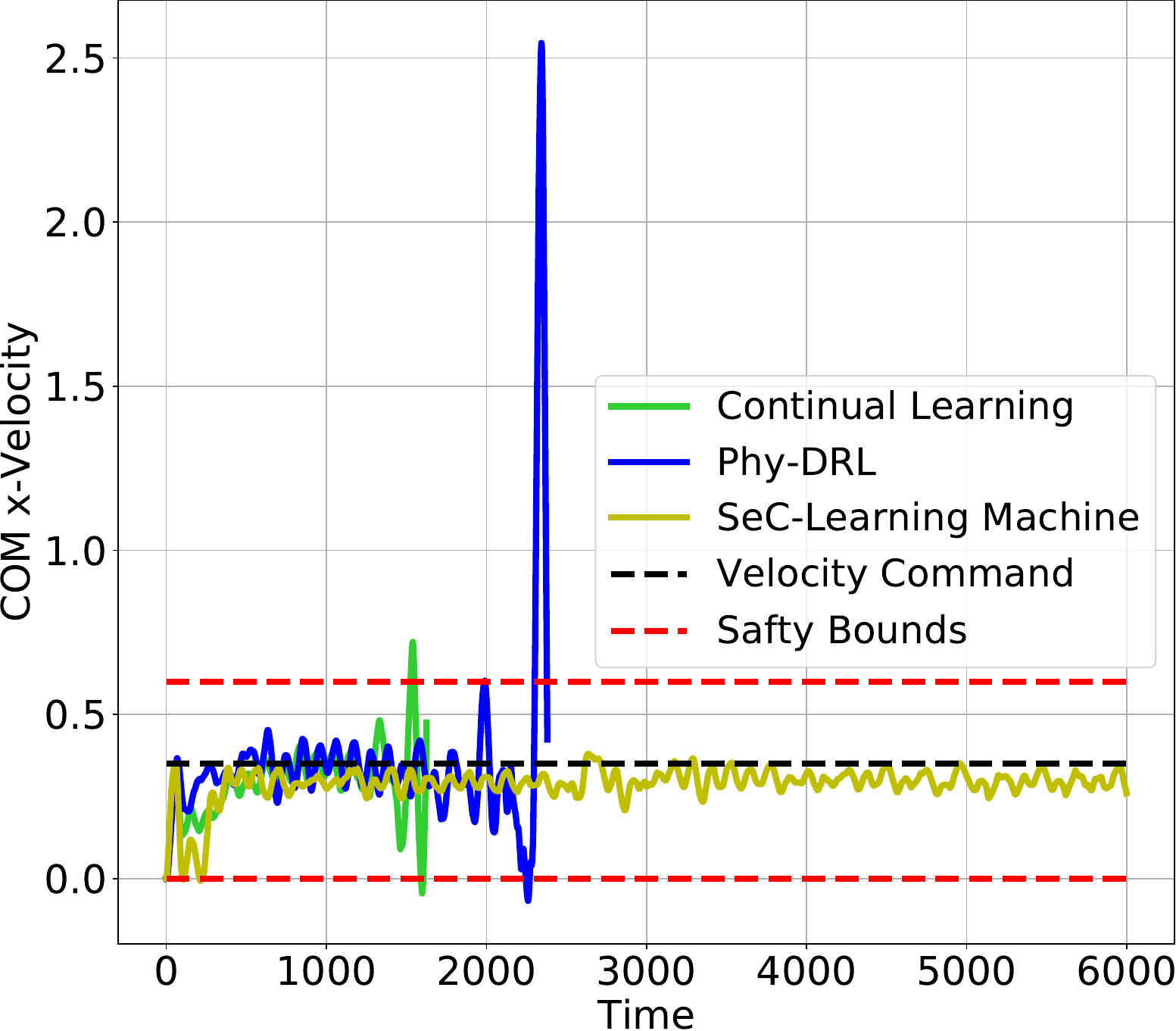}} \hspace{0.15cm}
    \centering
    \subfloat[$\text{CoM z-height}$]{\includegraphics[width=0.300\textwidth]{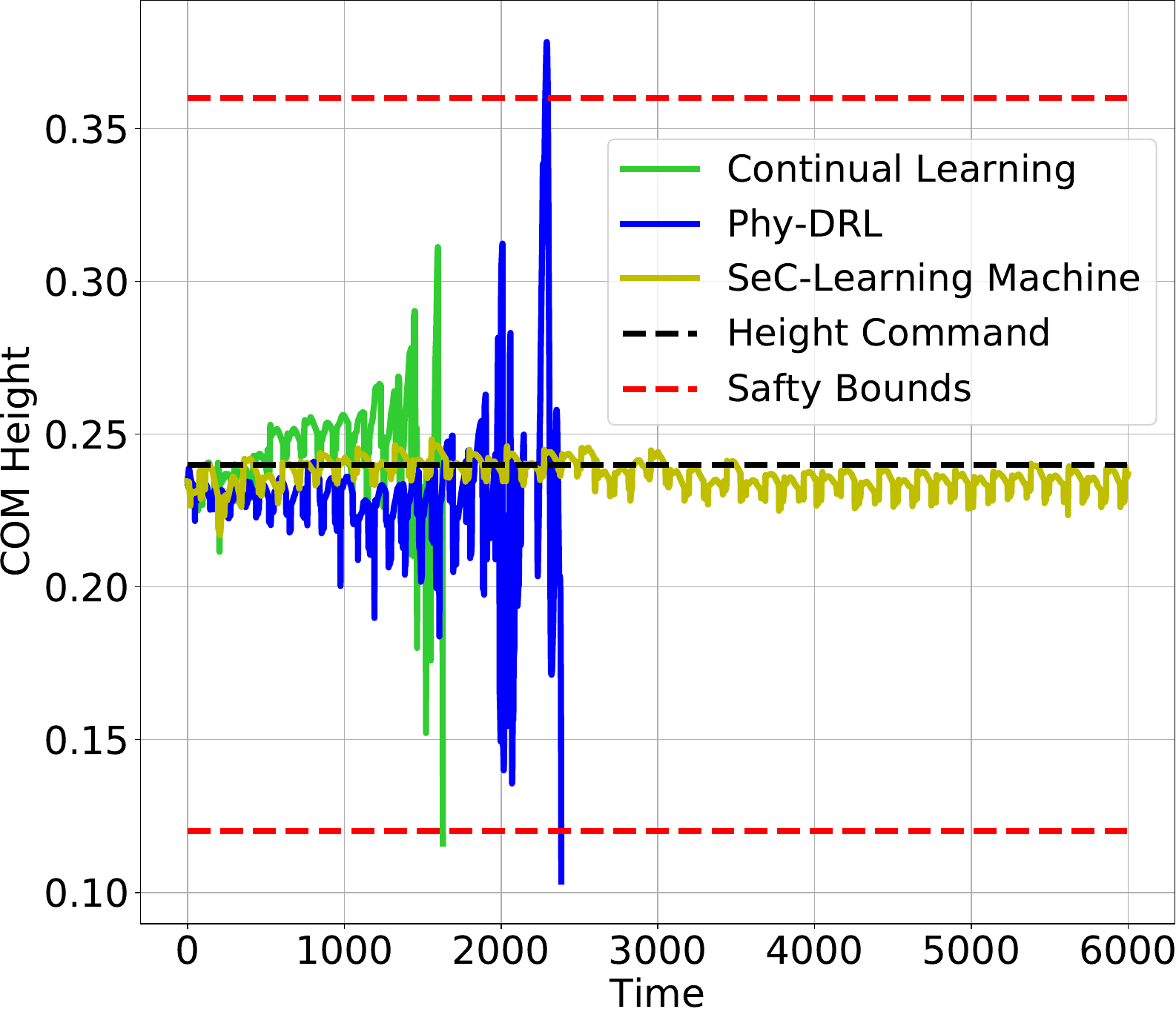}}   \hspace{0.15cm}  
    \centering
    \subfloat[$\text{Episode-average rewards}$]{\includegraphics[width=0.325\textwidth]{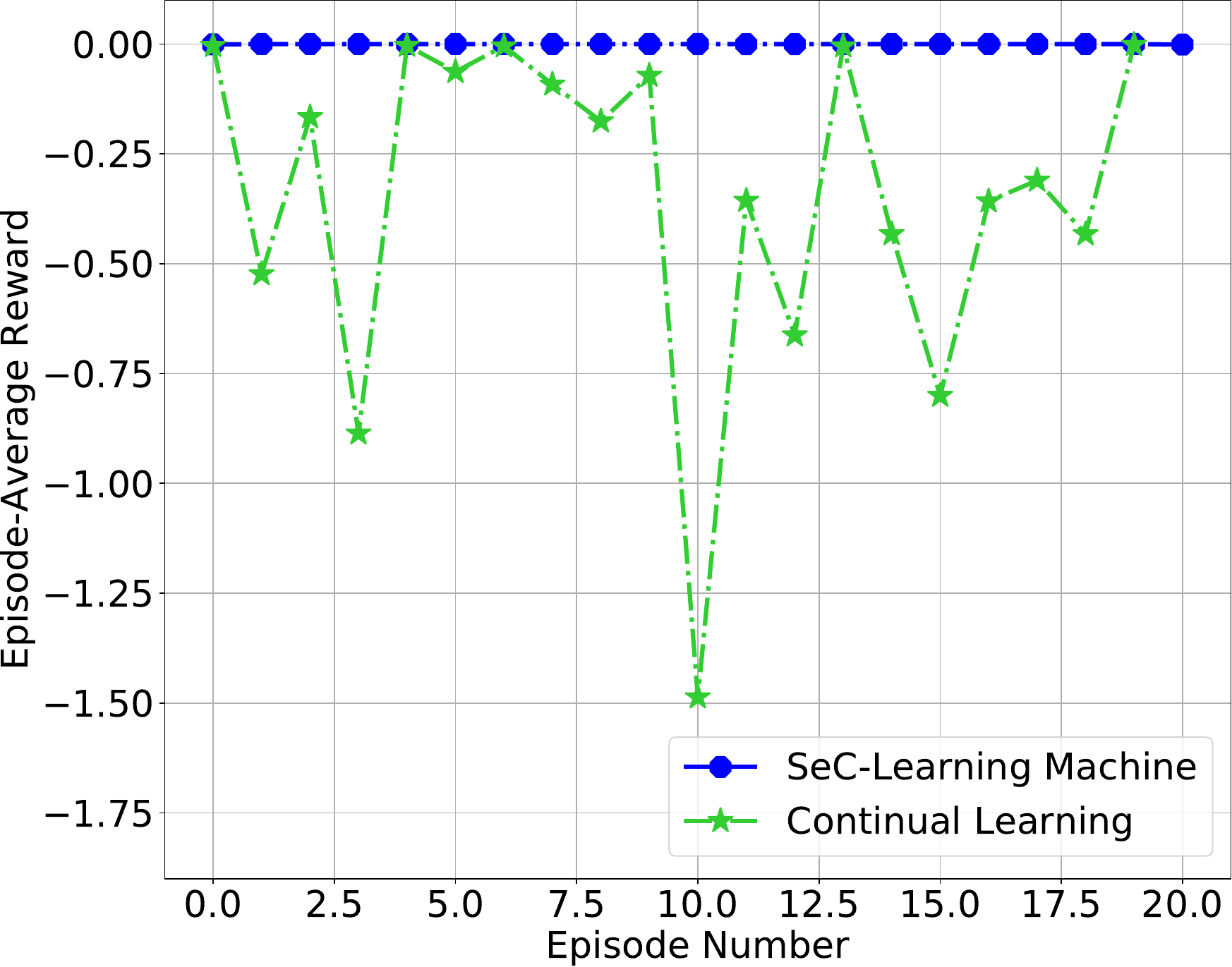}} 
    \vspace{-0.12cm}
\caption{Comparisons of trajectories and episode-average reward.}
\vspace{-0.3cm}
\label{quaapcvopt}
\end{figure}

Finally, we emphasize that another critical mission of HA-Teacher is correcting HP-Student's unsafe behavior. In other words, HP-Student learns from HA-Teacher to be safe when his actions are unsafe. To demonstrate this, we compare the episode-average reward curves of HP-Student (with HA-Teacher support) against the one (without Simplex logic or HA-Teacher support) in continual learning, as depicted in \cref{quaapcvopt} (c). In addition, the reward curves in the iteration step for 20 episodes are shown in \cref{addewwards} in \cref{addtextsimppo}. Meanwhile, We also deactivate HA-Teacher during the 20th episode of HP-Student's continual learning to verify if HP-Student has quickly assimilated safe behaviors from HA-Teacher. The demonstration video for comparison with the initial HP-Student is available at \href{https://www.dropbox.com/scl/fi/qughself4wqt8qpanhajx/com2.mp4?rlkey=get440fncfnqakwrgeymw1y2z&st=93znk3fw&dl=0}{\color{blue} real-robot-video-2 link}. The demonstration video, \cref{quaapcvopt} (c) and \cref{addewwards} prove that compared with Continual Learning, HA-Teacher enables stable, fast, and safe learning for HP-Student. Notably, the video shows that after only \underline{20 episodes} (i.e., \underline{300 sec}), the HP-Student has successfully learned from HA-Teacher to be safe.

\section{Conclusion and Discussion} \label{conljus}
This paper develops the SeC-learning machine for safety-critical autonomous systems. The SeC-learning machine constitutes HP-Student  (a pre-trained Phy-DRL, continuing to learn in a real plan), HA-Teacher (verified physics-model-based design), and Coordinator. In the learning machine, the HA-Teacher backs up safety and corrects unsafe learning of HP-Student. The SeC-learning machine aims to assure lifetime safety, address the Sim2Real gap, and learn to tolerate unknowns unknowns in real plants. Experiments also demonstrate that the SeC-learning machine features remarkably fast, stable, and safe training compared with continual learning without Simplex logic. 

Intuitively, the SeC-learning machine is also an automatic hierarchy learning machine. Specifically, the HP-student first learns from the HA-teacher to be safe. If safe enough, i.e., the HP-student can be independent of the HA-teacher, the HP-student will learn by himself to achieve the goal of high performance with verifiable safety. This investigation continues our future research direction. 

\section{Computation Resources} 
\label{computation}
In all case studies, we train and test the deep reinforcement learning (DRL) algorithm on a desktop equipped with Ubuntu 22.04, a 12th Gen Intel(R) Core(TM) i9-12900K 16-core processor, 64 GB of RAM, and an NVIDIA GeForce GTX 3090 GPU. The DRL algorithm was implemented in Python using the TensorFlow framework. We used the open-source Python CVX solver to solve LMIs problems. 

In our architecture, the computation of $\widehat{\mathbf{F}}$ and $\widehat{\mathbf{P}}$ for HA-Teacher at each patch need to be done in time when the Safety Coordinator is triggered. To enable real-time computation of CVX and interaction with the environment, we implement a multi-processing pipeline to control the robot and solve LMIs in parallel in real-time. For solving LMIs, we always let the solver take the latest state so that whenever the safety coordinator is triggered, the latest $\widehat{\mathbf{F}}$ and $\widehat{\mathbf{P}}$ will always be ready, where the delay issue was considered and formulated in the LMIs problems. 

We also noticed that the MATLAB-based CVX solver could solve the LMIs problem more consistently than the Python-based one, yielding more reliable solutions. However, the data interfacing overhead between Matlab and Python will introduce extra delay when updating $\widehat{\mathbf{F}}$ and $\widehat{\mathbf{P}}$ for HA-Teacher. Besides, the multiprocessing implementation for Matlab and Python is another technical challenge due to software compatibility issues. Therefore, we took the Python-based CVX solver for real-time real-world experiments, and we suggest that the Matlab-based solver is preferable for the less real-time sensitive applications.

\bibliographystyle{unsrt}
\bibliography{reference}

\newpage

\appendix
\appendixpage

\startcontents[sections]
\printcontents[sections]{}{0}{\setcounter{tocdepth}{1}}

\newpage
\section{Auxiliary Lemmas} \label{Aux}
\begin{lemma} [Schur Complement \cite{zhang2006schur}]
For any symmetric matrix $\mathbf{M} = \left[\!\! {\begin{array}{*{20}{c}}
\mathbf{A}&\mathbf{B}\\
{{\mathbf{B}^\top}}&\mathbf{C}
\end{array}} \!\!\right]$, then $\mathbf{M} \succ 0$ if and only if $\mathbf{C} \succ 0$ and $\mathbf{A} - \mathbf{B}\mathbf{C}^{-1}\mathbf{B}^\top \succ 0$. \label{auxlem2}
\end{lemma}

\begin{lemma} \cite{Phydrl1}
Consider the sets defined in \cref{aset2} and \cref{set3}. We have ${\Omega} \subseteq \mathbb{X}$, if
\begin{align}
&\left[ \underline{\mathbf{D}} \right]_{i,:} \cdot {\mathbf{P}^{-1}} \cdot \left[ \underline{\mathbf{D}}^\top \right]_{:,i} = \begin{cases}
		\ge 1, &[{\mathbf{d}}]_i = 1\\
		\le 1, &[{\mathbf{d}}]_i = -1
	\end{cases}, \!\!~~\left[ \overline{\mathbf{D}} \right]_{i,:} \cdot {\mathbf{P}^{-1}} \cdot \left[ \overline{\mathbf{D}}^\top \right]_{:,i} \le 1, ~i \in \{1, \ldots, h\}, \label{set5}
\end{align} 
where $\overline{\mathbf{D}} = \frac{{\mathbf{D}}}{\overline{\Lambda}}$, $\underline{\mathbf{D}} = \frac{{\mathbf{D}}}{\underline{\Lambda}}$, and for $i,j \in \{1, \ldots, h\}$, 
\begin{align}
 [{\mathbf{d}}]_{i} \!\triangleq\! \begin{cases}
		1, \!\!&\left[\underline{\mathbf{v}} \!+\!  \mathbf{v}\right]_{i} \!>\! 0 \\
        1, \!\!&\left[\overline{\mathbf{v}} \!+\!  \mathbf{v}\right]_{i} \!<\! 0 \\
		-1, \!\!& \text{otherwise}
	\end{cases}\!, ~~[\overline{\Lambda}]_{i,j} \!\triangleq\! \begin{cases}
		0, \!\!&i \!\ne\! j\\
		\left[\overline{\mathbf{v}} \!+\!  \mathbf{v}\right]_{i}, \!\!&\left[\underline{\mathbf{v}} \!+\!  \mathbf{v}\right]_{i} \!>\! 0 \\
        \left[\underline{\mathbf{v}} \!+\!  \mathbf{v}\right]_{i}, \!\!&\left[\overline{\mathbf{v}} \!+\!  \mathbf{v}\right]_{i} \!<\! 0 \\
		\left[\overline{\mathbf{v}} \!+\!  \mathbf{v}\right]_{i}, \!\!\!&\text{otherwise}
	\end{cases}\!, ~~[\underline{\Lambda}]_{i,j} \!\triangleq\! \begin{cases}
		0, \!\!&i \!\ne\! j\\
		\left[\underline{\mathbf{v}} \!+\!  \mathbf{v}\right]_{i}, \!\!&\left[\underline{\mathbf{v}} \!+\!  \mathbf{v}\right]_{i} \!>\! 0 \\
        \left[\overline{\mathbf{v}} \!+\!  \mathbf{v}\right]_{i}, \!\!&\left[\overline{\mathbf{v}} \!+\!  \mathbf{v}\right]_{i} \!<\! 0 \\
		\left[-\underline{\mathbf{v}} -  \mathbf{v}\right]_{i}, \!\!\!& \text{otherwise} 
	\end{cases}\!. \nonumber
\end{align} \label{safelemmaorg}
\end{lemma}

\begin{lemma} 
Consider the action set $\mathbb{A}_{\text{phy}}$ defined in \cref{org}, and 
\begin{align}
{\Phi} \triangleq \left\{ {\left. {\mathbf{a} \in {\mathbb{R}^m}} ~\right|~{\mathbf{a}^\top} \cdot {\mathbf{V}} \cdot \mathbf{a} \le 1,~{\mathbf{V}} \succ 0} \right\}.\label{set4}
\end{align}
We have ${\Phi} \subseteq \mathbb{A}_{\text{phy}}$, if 
\begin{align}
\left[ \overline{\mathbf{C}} \right]_{i,:} \cdot {\mathbf{V}^{-1}} \cdot \left[ \overline{\mathbf{C}}^\top \right]_{:,i} \le 1,~~\left[ \underline{\mathbf{C}} \right]_{i,:} \cdot {\mathbf{V}^{-1}} \cdot \left[ \underline{\mathbf{C}}^\top \right]_{:,i} = \begin{cases}
		\ge 1, \!\!&\text{if}~~[{\mathbf{c}}]_i = 1\\
		\le 1, \!\!&\text{if}~~[{\mathbf{c}}]_i = -1
	\end{cases}\!, ~i \in \{1, \ldots, g\} \label{set6}
\end{align}
where $\overline{\mathbf{C}} = \frac{{\mathbf{C}}}{\overline{\Delta}}$, $\underline{\mathbf{C}} = \frac{{\mathbf{C}}}{\underline{\Delta}}$, and for $i,j \in \{1, \ldots, g\}$,
\begin{align}
 [{\mathbf{c}}]_{i} \!\triangleq\! \begin{cases}
		1, \!\!&\left[\underline{\mathbf{z}} +  \mathbf{z}\right]_{i} > 0 \\
        1, \!\!&\left[\overline{\mathbf{z}} +  \mathbf{z}\right]_{i} < 0 \\
		-1, \!\!& \text{otherwise}
	\end{cases}\!, ~[\overline{\Delta}]_{i,j} \!\triangleq\! \begin{cases}
		0, \!\!&i \ne j\\
		\left[\overline{\mathbf{z}} +  \mathbf{z}\right]_{i}, \!\!&\left[\underline{\mathbf{z}} \!+\!  \mathbf{z}\right]_{i} \!>\! 0 \\
        \left[\underline{\mathbf{z}} \!+\!  \mathbf{z}\right]_{i}, \!\!&\left[\overline{\mathbf{z}} \!+\!  \mathbf{z}\right]_{i} \!<\! 0 \\
		\left[\overline{\mathbf{z}} +  \mathbf{z}_{\sigma}\right]_{i}, &\text{otherwise}
	\end{cases}\!, 
 ~[\underline{\Delta}]_{i,j} \!\triangleq\! \begin{cases}
		0, \!\!&i \ne j\\
		\left[\underline{\mathbf{z}} \!+\!  \mathbf{z}\right]_{i}, \!\!&\left[\underline{\mathbf{z}} \!+\!  \mathbf{z}\right]_{i} \!>\! 0 \\
        \left[\overline{\mathbf{z}} \!+\!  \mathbf{z}\right]_{i}, \!\!&\left[\overline{\mathbf{z}} \!+\!  \mathbf{z}\right]_{i} \!<\! 0 \\
		\left[-\underline{\mathbf{z}} \!-\!  \mathbf{z}\right]_{i}, & \text{\text{otherwise}}
	\end{cases}\!. \nonumber  
\end{align} 
 \label{safelemma}
\end{lemma}

\begin{proof}
The proof path of \cref{safelemma} is the same as that of \cref{safelemmaorg} in \cite{Phydrl1}. So, it is omitted here. 
\end{proof}

\newpage
\section{Control Contribution Ratio} \label{Auxvvbbggffd}
\subsection{Controllable Data-driven Action Space} \label{Auxmovcda}
In continuous control tasks, actions of the data-driven approaches are typically generated by policies parameterized with deep neural networks, which map states to actions. 
In DRL, deterministic policies, as described in DDPG \cite{lillicrap2015continuous}, directly map states to actions within the interval 
$[-1, 1]$ using the Tanh activation function at the output layer. On the other hand, stochastic policies, such as those proposed in the Soft Actor-Critic (SAC)\cite{sac} algorithm and Proximal Policy Optimization (PPO) \cite{schulman2017proximal}, sample actions from a learned Gaussian distribution with reparameterization trick, followed by Tanh activation to ensure actions remain within the $[-1, 1]$ interval. To adapt the generated action for actual control commands, the action is rescaled by a magnitude factor $M_l$, expanding the action space to $[-M_l, M_l]$. This allows for the flexible adaptation of the action space to the specific requirements of different control tasks.

\subsection{Controllable Physics-model-based Action Space} \label{Auxmovcma}
The controllable physics-model-based action space is delivered by incorporating two additional conditions \eqref{set6} and \eqref{th3} into Theorem 5.3's LMIs in \cite{Phydrl1}. The updated version is formally stated in the following theorem.  

\begin{theorem}
Consider the system \eqref{realsys} under control of Phy-DRL, consisting of residual policy \eqref{residual} and safety-embedded reward \eqref{reward}, wherein  matrices $\mathbf{F}$ and $\mathbf{P}$ are computed according to 
\begin{align}
\mathbf{F} = \mathbf{R} \cdot \mathbf{Q}^{-1}, ~~~~ \mathbf{P} = \mathbf{Q}^{-1}, \label{th10000}
\end{align}
and $\mathbf{R}$ and $\mathbf{Q}^{-1}$ satisfying the inequalities \eqref{set5}, \eqref{set6} with $\mathbf{V} = \beta \cdot \mathbf{I}_{m} \succ 0$,  and 
\begin{align}
&\left[ {\begin{array}{*{20}{c}}
\mathbf{Q} &  \mathbf{R}^\top\\
\mathbf{R} & \frac{1}{\beta } \cdot \mathbf{I}_{m}
\end{array}}\right] \succ 0,  \label{th3}\\
&\left[ {\begin{array}{*{20}{c}}
\alpha \cdot \mathbf{Q} & \mathbf{Q} \cdot \mathbf{A}^\top + \mathbf{R}^\top \cdot \mathbf{B}^\top\\
\mathbf{A} \cdot \mathbf{Q} + \mathbf{B} \cdot \mathbf{R} & \mathbf{Q}
\end{array}} \right] \succ 0, ~~~~\text{with a given}~\alpha \in (0,1).\label{th1}
\end{align}
With the sets ${\Omega}$ \eqref{set3}, $\mathbb{X}$ \eqref{aset2}, $\mathbb{A}$ \eqref{org}, and  
${\Phi}$ \eqref{set4} at hand, the system \eqref{realsys} has the following properties: 
\begin{enumerate}
\item Given any $\mathbf{s}(1) \in {\Omega}$, the system state $\mathbf{s}(k) \in {\Omega} \subseteq \mathbb{X}$ holds for any $k \in \mathbb{N}$ (i.e., the safety of system (1) is guaranteed), if the sub-reward $r(\mathbf{s}(k), \mathbf{s}(k+1))$ in \cref{reward} satisfies $r(\mathbf{s}(k), \mathbf{s}(k+1)) \ge \alpha - 1,~\forall k \in \mathbb{N}$. \label{updateth1}
\item If the state $\mathbf{s}(k) \in {\Omega} \subseteq \mathbb{X}$, the physics-model-based control command in \cref{residual} satisfies $\mathbf{a}_{\text{phy}}(k) \in {\Phi} \subseteq \mathbb{A}_{\text{phy}}$.  \label{updateth2}
\end{enumerate}
\label{thm10007p}
\end{theorem}

\begin{proof}
\cref{updateth1} is a direct result of Theorem 5.3 in \cite{Phydrl1}, which is not influenced by the additional conditions \eqref{th3} and \cref{set6}. So, its proof is omitted. 

We now focus on the proof of \cref{updateth2} of \cref{thm10007p}. Since $\beta > 0$ and ${\mathbf{V}} = \beta \cdot \mathbf{I}_{m}$, according to Lemma \ref{auxlem2}, the condition \eqref{th3} implies that 
\begin{align}
\mathbf{Q} - \beta \cdot \mathbf{R}^\top \cdot \mathbf{R} = \mathbf{Q} -  \mathbf{R}^\top \cdot \mathbf{V} \cdot \mathbf{R} \succ 0. \label{addpf1}
\end{align}
Substituting  $\mathbf{F} \cdot \mathbf{Q} = \mathbf{R}_i$ into \eqref{addpf1} leads to $\mathbf{Q} - (\mathbf{F} \cdot \mathbf{Q})^\top \cdot \mathbf{V} \cdot (\mathbf{F} \cdot \mathbf{Q}) \succ 0$, multiplying both left-hand and right-hand sides of which by $\mathbf{Q}^{-1}$ yields $\mathbf{Q}^{-1} - \mathbf{F}^\top \cdot \mathbf{V} \cdot \mathbf{F} \succ 0$. We thus have 
\begin{align}
\mathbf{s}^{\top}(k) \cdot \mathbf{Q}^{-1} \cdot \mathbf{s}(k) - \mathbf{s}^{\top}(k) \cdot \mathbf{F}^\top \cdot \mathbf{V} \cdot \mathbf{F} \cdot \mathbf{s}(k) =  \mathbf{s}^{\top}(k) \cdot \mathbf{P} \cdot \mathbf{s}(k) - \mathbf{a}_{\text{phy}}^{\top}(k) \cdot \mathbf{V} \cdot \mathbf{a}_{\text{phy}}(k) > 0, \nonumber
\end{align}
which is obtained via considering $\mathbf{P} = \mathbf{Q}^{-1}$ and $\mathbf{F} = \mathbf{R} \cdot \mathbf{Q}^{-1} = \mathbf{R} \cdot \mathbf{P}$. The inequality  means $\mathbf{s}^{\top}(k) \cdot \mathbf{P} \cdot \mathbf{s}(k) > \mathbf{a}_{\text{phy}}^{\top}(k) \cdot \mathbf{V} \cdot \mathbf{a}_{\text{phy}}(k)$. Therefore, in light of the definition of safety envelope \eqref{set3}, we conclude that  if $s(k) \in {\Omega}$, i.e., $\mathbf{s}^{\top}(k) \cdot \mathbf{P} \cdot \mathbf{s}(k) < 1$, then $\mathbf{a}_{\text{phy}}^{\top}(k) \cdot \mathbf{V} \cdot \mathbf{a}_{\text{phy}}(k) < 1$. Furthermore, considering \eqref{set6} in Lemma \ref{safelemma} 
in conjunction with defined set \eqref{set4}, we conclude $\mathbf{a}_{\text{phy}}(k) \in {\Phi} \subseteq  \mathbb{A}$, which completes the proof. 
\end{proof}

\subsection{Automatic Constructions of Residual Action Policy and Safety-embedded Reward} \label{aut}
Given the matrices $\mathbf{F}$ and $\mathbf{P}$, the residual policy \cref{residual} and safety-embedded reward \cref{reward} are delivered immediately. With the available physics-model knowledge $(\mathbf{A}, \mathbf{B})$ at hand, the $\mathbf{F}$ and $\mathbf{P}$ can be automatically computed from LMIs in \cref{set5}, \cref{set6}, \cref{th3}, and \cref{th1}, by LMI toolbox \cite{gahinet1994lmi, boyd1994linear}.

We can also find optimal $\mathbf{R}$ and $\mathbf{Q}$ that can maximize the safety envelope. To achieve this, we recall the volume of a safety envelope \eqref{set3} is proportional to $\sqrt {\det \left( {\mathbf{P}^{ - 1}} \right)}$, the interested problem is thus a typical analytic centering problem, formulated as given a $\alpha \in (0,1)$,
\begin{align} 
\mathop {\text{argmin} }\limits_{{\mathbf{Q}},~{\mathbf{R}}} \left\{ {\log \det \left( {\mathbf{Q}^{ - 1}} \right)} \right\} = \mathop {\text{argmin}}\limits_{{\mathbf{Q}},~{\mathbf{R}}} \left\{ {\log \det \left( {\mathbf{P}^{ - 1}} \right)} \right\}, ~\text{subject to LMIs}~\eqref{set5}, \eqref{set6}, \eqref{th3}, \eqref{th1}. \nonumber
\end{align}

\newpage
\section{Proof \cref{thm10007p}} \label{thm10007paux}
\subsection*{Proof of \cref{stat1}}
We define a Lyapunov candidate for a real plant described in \cref{realsyserror}: 
\begin{align}
V\left( k \right) = {\mathbf{e}^\top}\left( k \right) \cdot \widehat{\mathbf{P}} \cdot \mathbf{e}\left( k \right),  \label{lipassm}
\end{align}
which along the dynamics \eqref{realsyserror}, in conjunction with the action policy \eqref{hacteacherpolicy}, results in 
\begin{align}
&V\left( {k + 1} \right) - \beta  \cdot V\left( k \right) \nonumber\\
&= {\mathbf{e}^\top}\left( {k + 1} \right) \cdot \widehat{\mathbf{P}} \cdot \mathbf{e}\left( {k + 1} \right) - \beta  \cdot {\mathbf{e}^\top}\left( k \right) \cdot \widehat{\mathbf{P}} \cdot \mathbf{e}\left( k \right) \nonumber\\
& = {\mathbf{e}^\top}\left( k \right) \cdot \left( {{\overline{\mathbf{A}}^\top(\mathbf{s}(k))} \cdot \widehat{\mathbf{P}} \cdot \overline{\mathbf{A}}(\mathbf{s}(k)) - \beta  \cdot \widehat{\mathbf{P}}} \right) \cdot \mathbf{e}\left( k \right) + {\mathbf{h}^\top}\left( {\mathbf{e}\left( k \right)} \right) \cdot \widehat{\mathbf{P}} \cdot \mathbf{h}\left( {\mathbf{e}\left( k \right)} \right) \nonumber\\
& \hspace{7.0cm} + 2{\mathbf{e}^\top}\left( k \right) \cdot \left( {\overline{\mathbf{A}}(\mathbf{s}(k)) \cdot \widehat{\mathbf{P}}} \right) \cdot \mathbf{h}\left( {\mathbf{e}\left( k \right)} \right), \label{frstderiv}
\end{align}
where we define: 
\begin{align}
\overline{\mathbf{A}}(\mathbf{s}(k)) \buildrel \Delta \over = \mathbf{A}(\mathbf{s}(k)) + \mathbf{B}(\mathbf{s}(k)) \cdot \widehat{\mathbf{F}}.  \label{lipassm}
\end{align}

It is straightforward to verify the inequality 
\begin{align}
&2{\mathbf{e}^\top}\left( k \right) \cdot \left( {\overline{\mathbf{A}}(\mathbf{s}(k)) \cdot \widehat{\mathbf{P}}} \right) \cdot \mathbf{h}\left( {\mathbf{e}\left( k \right)} \right) \nonumber\\
&\le \omega \cdot {\mathbf{e}^\top}\left( k \right) \cdot {{\overline{\mathbf{A}}}^\top} (\mathbf{s}(k)) \cdot \widehat{\mathbf{P}} \cdot \overline{\mathbf{A}}(\mathbf{s}(k)) \cdot \mathbf{e}\left( k \right) + \frac{1}{\omega} \cdot {\mathbf{h}^\top}\left( {\mathbf{e}\left( k \right)} \right) \cdot \widehat{\mathbf{P}} \cdot \mathbf{h}\left( {\mathbf{e}\left( k \right)} \right),  \label{lipassm1}
\end{align}
holds for any $\omega > 0$.

Meanwhile, noting $\widehat{\mathbf{Q}} = \widehat{\mathbf{P}}^{-1}$, the inequality in \cref{cc0} is equivalent to $\eta \cdot {\mathbf{P}} \succ  \widehat{\mathbf{P}} \succ {\mathbf{P}}$, which in light of \cref{assm} then results in 
\begin{align}
{\mathbf{h}^\top}\left( {\mathbf{e}\left( k \right)} \right) \cdot \widehat{\mathbf{P}} \cdot \mathbf{h}\left( {\mathbf{e}\left( k \right)} \right) \le  {\mathbf{h}^\top}\left( {\mathbf{e}\left( k \right)} \right) \cdot \eta \cdot {\mathbf{P}} \cdot \mathbf{h}\left( {\mathbf{e}\left( k \right)} \right) &\le {\mathbf{e}^\top\left( k \right)} \cdot \kappa \cdot \eta \cdot {\mathbf{P}} \cdot {\mathbf{e}\left( k \right)} \nonumber\\
&\le {\mathbf{e}^\top\left( k \right)} \cdot \kappa \cdot \eta \cdot \widehat{{\mathbf{P}}} \cdot {\mathbf{e}\left( k \right)}. \label{lipassm2}
\end{align}

Substituting inequalities in \cref{lipassm1} and \cref{lipassm2} into \cref{frstderiv} yields 
\begin{align}
&V\left( {k + 1} \right) - \beta  \cdot V\left( k \right) \nonumber\\
& \le {\mathbf{e}^\top}\left( k \right) \cdot \left( {(1 + \omega)\cdot {{\overline{\mathbf{A}}}^\top}(\mathbf{s}(k)) \cdot \widehat{\mathbf{P}} \cdot \overline{\mathbf{A}}(\mathbf{s}(k)) - (\beta - \kappa \cdot \eta \cdot (1 + \frac{1}{\omega}))  \cdot \widehat{\mathbf{P}}}\right) \cdot \mathbf{e}\left( k \right).\label{lipassm3}
\end{align}

Recalling Schur Complement in \cref{auxlem2} in \cref{Aux} and considering $\widehat{\mathbf{P}} \succ 0$, we conclude that the inequality in \cref{cc3} is equivalent to 
\begin{align}
&(\beta - \kappa \cdot \eta \cdot (1 + \frac{1}{\omega})) \cdot \widehat{\mathbf{Q}} \nonumber\\
&\hspace{0.2cm}-  (1 + \omega) \cdot (\widehat{\mathbf{Q}} \cdot \mathbf{A}^\top(\mathbf{s}(k)) + \widehat{\mathbf{R}}^\top \cdot \mathbf{B}^\top(\mathbf{s}(k))) \cdot \widehat{\mathbf{Q}}^{-1} \cdot (\mathbf{A}(\mathbf{s}(k)) \cdot \widehat{\mathbf{Q}} + \mathbf{B}(\mathbf{s}(k)) \cdot \widehat{\mathbf{R}})  \succ 0.  \label{lipassm4}
\end{align}
Multiplying both the left-hand side and the right-hand side of inequality \eqref{lipassm4} by $\widehat{\mathbf{Q}}^{-1}$ yields 
\begin{align}
&(\beta - \kappa \cdot \eta \cdot (1 + \frac{1}{\omega})) \cdot \widehat{\mathbf{Q}}^{-1}  \nonumber\\
&\hspace{0.3cm} -  (1 + \omega) \cdot (\mathbf{A}^\top(\mathbf{s}(k)) + \widehat{\mathbf{Q}}^{-1} \cdot \widehat{\mathbf{R}}^\top \cdot \mathbf{B}^\top(\mathbf{s}(k))) \cdot \widehat{\mathbf{Q}}^{-1} \cdot (\mathbf{A}(\mathbf{s}(k)) + \mathbf{B}(\mathbf{s}(k)) \cdot \widehat{\mathbf{R}} \cdot \mathbf{Q}^{-1})  \succ 0, \nonumber
\end{align}
substituting the definitions in \cref{acc0} into which leads to  
\begin{align}
&(\beta - \kappa \cdot \eta \cdot (1 + \frac{1}{\omega})) \cdot \widehat{\mathbf{P}} \nonumber\\
&\hspace{1.5cm}-  (1 + \omega)\cdot(\mathbf{A}^\top(\mathbf{s}(k)) +  \widehat{\mathbf{F}}^\top \cdot \mathbf{B}^\top(\mathbf{s}(k))) \cdot \widehat{\mathbf{P}} \cdot (\mathbf{A}(\mathbf{s}(k)) + \mathbf{B}(\mathbf{s}(k)) \cdot \widehat{\mathbf{F}})  \succ 0. \label{lipassm5}
\end{align}

Recalling \eqref{lipassm}, the inequality in \cref{lipassm5} is equivalent to 
\begin{align}
(1 + \omega)\cdot{{{\overline{\mathbf{A}}}^\top}(\mathbf{s}(k)) \cdot \widehat{\mathbf{P}} \cdot \overline{\mathbf{A}}(\mathbf{s}(k)) - (\beta - \kappa \cdot \eta \cdot (1 + \frac{1}{\omega})) \cdot \widehat{\mathbf{P}}}  \prec 0, \nonumber
\end{align}
which, in conjunction with \cref{lipassm3}, leads to $V\left( {k + 1} \right) - \beta  \cdot V\left( k \right) \le 0$, i.e., $V\left( {k + 1} \right) \le \beta  \cdot V\left( k \right)$, we thus complete the proof of \cref{stat1}.

\subsection*{Proof of \cref{stat2}}
For patch envelope \eqref{hacset}, we introduce its boundary:  
\begin{align}
\partial{\Omega_{\text{patch}}} \triangleq \left\{ {\left. \mathbf{x} ~\right|{{( {\mathbf{x} - \overline{\mathbf{s}}^*})}^\top} \cdot \widehat{\mathbf{P}} \cdot ( {\mathbf{x} - \overline{\mathbf{s}}^*}) = {{( {1 - \chi })}^2} \cdot {{\mathbf{s}^\top(k)}} \cdot \widehat{\mathbf{P}} \cdot \mathbf{s}(k)}, ~\widehat{\mathbf{P}} \succ 0 \right\}, \nonumber
\end{align}
which, in light of $\mathbf{e} \triangleq \mathbf{x} - \overline{\mathbf{s}}^{*}$, is equivalent to 
\begin{align}
\partial{\Omega_{\text{patch}}} \triangleq \left\{ {\left. \mathbf{x} ~\right|{{\mathbf{e}}^\top} \cdot \widehat{\mathbf{P}} \cdot \mathbf{e}  = {{\left( {1 - \chi } \right)}^2} \cdot {{ {{\mathbf{s}^\top(k)}}}} \cdot \widehat{\mathbf{P}} \cdot {\mathbf{s}(k)}} \right\}, \label{cc0p0}
\end{align}
by which, the proof of ${\Omega_{\text{patch}}} \subseteq {\Omega}$ equivalently transforms to the proof of ${\partial{\Omega_{\text{patch}}}} \subseteq {\Omega}$. 

To move forward, we first recall $\widehat{\mathbf{P}} = \widehat{\mathbf{Q}}^{-1} \succ 0$, in light of which, the inequality in \cref{cc0} is equivalent to 
\begin{align}
{\mathbf{P}} \prec \widehat{\mathbf{P}} \prec \eta  \cdot {\mathbf{P}}, ~~\text{with}~\eta > 1.   \label{cc0p1}
\end{align}
Meanwhile, recalling $\mathbf{e} = \mathbf{x} - \chi \cdot \mathbf{s}^*$, and we define
\begin{align}
{\mathbf{e}^{*}} = \mathbf{x}^*- \chi \cdot \mathbf{s}(k) = \mathbf{s}(k) - \chi \cdot \mathbf{s}(k)  = (1 - \chi) \cdot \mathbf{s}(k),  \label{cc0p2}
\end{align}
considering which, for any $\mathbf{x} \in \partial{\Omega_{\text{patch}}}$, it is straightforward to verify that  
\begin{align}
{\mathbf{x}^\top} \cdot {\mathbf{P}} \cdot \mathbf{x}
& = {\left( {\mathbf{e} + \chi  \cdot {\mathbf{s}(k)}} \right)^\top} \cdot {\mathbf{P}} \cdot \left( {\mathbf{e} + \chi  \cdot {\mathbf{s}(k)}} \right) \nonumber\\
& = {\mathbf{e}^\top} \cdot {\mathbf{P}} \cdot \mathbf{e} + {\chi ^2} \cdot {{{\mathbf{s}^\top(k)}}} \cdot{\mathbf{P}} \cdot {\mathbf{s}(k)} + 2\chi  \cdot {\mathbf{e}^\top} \cdot {\mathbf{P}} \cdot {\mathbf{s}(k)}\nonumber\\
& \le 2 \cdot {\mathbf{e}^\top} \cdot {\mathbf{P}} \cdot \mathbf{e} + 2{\chi ^2} \cdot {\mathbf{s}^\top(k)} \cdot {\mathbf{P}} \cdot \mathbf{s}(k) \label{cc0p21}\\
& \le 2 \cdot {\mathbf{e}^\top} \cdot {{\mathbf{P}}} \cdot \mathbf{e} + 2{\chi ^2}\cdot \varepsilon  \label{cc0p22}\\
& \le 2 \cdot {\mathbf{e}^\top} \cdot {\widehat{\mathbf{P}}} \cdot \mathbf{e} + 2{\chi ^2}\cdot \varepsilon  \label{cc0p220}\\
& = 2 \cdot {{\left( {1 - \chi } \right)}^2} \cdot {{\mathbf{s}^\top(k)}} \cdot \widehat{\mathbf{P}} \cdot {\mathbf{s}(k)} + 2{\chi ^2} \cdot \varepsilon\label{cc0p240}\\
& = 2 \cdot \eta  \cdot {\left( {1 - \chi } \right)^2} \cdot { {{\mathbf{s}^\top(k)}}} \cdot {\mathbf{P}} \cdot \mathbf{s}(k) + 2{\chi ^2}\cdot \varepsilon  \label{cc0p25}\\
& = 2 \cdot \eta  \cdot {\left( {1 - \chi } \right)^2}\cdot \varepsilon   + 2{\chi ^2}\cdot \varepsilon  \label{cc0p27}
\end{align}
We note \cref{cc0p21} is obtained from its previous step via considering the well-known inequality: $2\chi  \cdot {\mathbf{e}^\top} \cdot {\mathbf{P}} \cdot \mathbf{s}(k) \le {\mathbf{s}^\top(k)}  \cdot {\mathbf{P}} \cdot {\mathbf{s}(k)} + \chi^2  \cdot {\mathbf{e}^\top} \cdot {\mathbf{P}} \cdot {\mathbf{e}}$. We let $\mathbf{x}^{*} = \mathbf{s}(k)$ be a state sample of triggering condition for guaranteeing ${\mathbf{s}^\top(k)} \cdot {\mathbf{P}} \cdot \mathbf{s}(k) \le \varepsilon$, which is the root that derives \cref{cc0p22} from \cref{cc0p21}. \cref{cc0p220} from \cref{cc0p22} is obtained via considering the left-hand side inequality in \cref{cc0p1}. \cref{cc0p240} from \cref{cc0p220} is obtained in light of \cref{cc0p0} and the fact $\mathbf{s} \in \partial{\Omega_{\text{patch}}}$.  \cref{cc0p25} from \cref{cc0p240} is obtained via considering right-hand side inequality in \cref{cc0p1}. Lastly, \cref{cc0p27} is obtained from \cref{cc0p25} by considering the fact that the $\mathbf{x}^{*} = \mathbf{s}(k)$ is a state sample inside safety envelope, so satisfying ${\mathbf{s}^\top(k)} \cdot {\mathbf{P}} \cdot \mathbf{s}(k) \le 1$. 

Recall that the inequality in \cref{cc300} is equivalent to $2\eta  \cdot {\left( {1 - \chi } \right)^2}\cdot \varepsilon   + 2{\chi ^2}\cdot \varepsilon  \le 1$, which, in conjunction with \cref{cc0p27}, leads to the conclusion that for any $\mathbf{x} \in \partial{\Omega_{\text{patch}}}$, ${\mathbf{x}^\top} \cdot \widehat{\mathbf{P}} \cdot \mathbf{x} \le 1$ holds. In other words, in light of \cref{set3}, for any $\mathbf{x} \in \partial{\Omega_{\text{patch}}}$, the $\mathbf{x} \in {\Omega}$ holds. We thus conclude $\partial{\Omega_{\text{patch}}} \subseteq {\Omega}$, which completes the proof.

\newpage
\section{Experiment: Cart-Pole System}
\subsection{Pre-training and Contiunal Learning} \label{mppgcong}
We leverage the DDPG algorithm \cite{lillicrap2015continuous} to pre-train HP-Student, i.e., Phy-DRL, and support its continual learning. The actor and critic networks are implemented as a Multi-Layer Perceptron (MLP) with four fully connected layers. The output dimensions of critic and actor networks are 256, 128, 64, and 1, respectively. The activation functions of the first three neural layers are ReLU, while the output of the last layer is the Tanh function for the actor-network and Linear for the critic network. The input of the critic network is $[\mathbf{s}; \mathbf{a}]$, while the input of the actor-network is $\mathbf{s}$. In more detail, we let discount factor $\gamma = 0.9$, and the learning rates of critic and actor networks are the same as 0.0003. We set the batch size to 200. The maximum step number of one episode is 500.

\subsection{System Dynamics} \label{snsdyanfhgt}
The physics knowledge about the dynamics of cart-pole systems used by HP-Student and HA-Teacher for their designs is from the following dynamics model in \cite{correctIP}:
\begin{subequations}
\begin{align}
    \ddot{\theta} &= \frac{g\sin\theta+\cos \theta\left(\frac{-F-m_{p} l \dot{\theta}^{2} \sin \theta}{m_{c}+m_{p}}\right)}{l\left(\frac{4}{3}-\frac{m_{p} \cos ^{2} \theta}{m_{c}+m_{p}}\right)}, \\
    \ddot{x} &=\frac{F+m_{p} l\left(\dot{\theta}^{2} \sin \theta-\ddot{\theta} \cos \theta\right)}{m_{c}+m_{p}}, 
\end{align} \label{dymcarpole}
\end{subequations}
whose parameters' physical representations and values are given in \cref{notationcart}. 

\begin{table}[ht]
    \centering
    \begin{tabular}{||c|c|c|c||}
        \hline
        \multicolumn{2}{||c|}{Notation} & Value & Unit \\ 
        \hline\hline
        $m_c$  & {mass of cart} & {0.94} & {$kg$}  \\    
        $m_p$  & {mass of pole} & {0.23} & {$kg$} \\
        $g$ & {gravitational acceleration}  & {9.8} & {$m \cdot s^{-2}$} \\
        $l$ & {half length of pole}  & {0.32} & {$m$} \\ 
        $T$ & {sample period} & {1/30} & {$s$} \\
        $F$ & {actuator input} & {[-30, 30]} & {$N$} \\
        $\mu_c$ & {cart friction coefficient} & {18} & {} \\
        $\mu_p$ & {pole friction coefficient} & {0.0031} & {} \\
        $x$  & {position of cart} & {[-0.8, 0.8]} & {$m$}  \\
        $\dot{x}$ & {velocity of cart} & {[-3, 3]} & {$m \cdot s^{-1}$} \\
        $\theta$  & {angle of pole} & {[-0.9, 0.9]} & {$rad$}  \\
        $\dot{\theta}$ & {angular velocity of pole} & {[-4.5, 4.5]} & {$rad \cdot s^{-1}$} \\ 
        \hline
    \end{tabular}
    \vspace{0.2cm}
    \caption{Notation Table for Cart-Pole System}
    \label{notationcart}
\end{table}
\subsection{HP-Student: Physics Knowledge and Design} \label{HP}
As Phy-DRL allows us to simply have nonlinear dynamics \eqref{dymcarpole} to an analyzable line model:
\begin{align}
\dot{\mathbf{s}} = \widehat{\mathbf{A}} \cdot \mathbf{s} + \widehat{\mathbf{B}} \cdot {\mathbf{a}}, \label{linreferem}
\end{align}
where $\mathbf{s} = [x,v, \theta,\omega]^{\top}$. To have 
$\widehat{\mathbf{A}}$ and $\widehat{\mathbf{B}}$ from \cref{dymcarpole}, we let $\cos \theta  \approx 1$, $\sin \theta  \approx \theta$ and ${\omega ^2}\sin \theta  \approx 0$. Meanwhile, the sampling technique transforms the continuous-time model \eqref{linreferem} to the discrete-time model: 
\begin{align}
{\mathbf{s}}(k+1) = {\mathbf{A}} \cdot {\mathbf{s}}(k) + {\mathbf{B}} \cdot {\mathbf{a}}(k), ~\text{with}~{\mathbf{A}} = \mathbf{I}_{4} + T \cdot \widehat{\mathbf{A}}, ~{\mathbf{B}} = T \cdot \widehat{\mathbf{A}}, \nonumber
\end{align}
where we have 
\begin{align} 
\mathbf{A} = \left[{\begin{array}{*{20}{c}}
1&{0.0333}&0&0\\
0&1&{ - 0.0565}&0\\
0&0&1&{0.0333}\\
0&0&{0.8980}&1
\end{array}}\right], ~~~~~\mathbf{B} = \left[
0~~0.0334~~0~~- 0.0783
\right]^\top.  \label{PM2} 
\end{align}

Considering the safety conditions in \cref{safetysetexp} and action space condition in \cref{orgact} and the formulas of safety set in \cref{aset2} and action set in \cref{org}, we have 
\begin{align}
\mathbf{D} \!=\! \left[\! {\begin{array}{*{20}{c}}
1\!&\!0 \!&\! 0 \!&\! 0\\
0\!&\!0 \!&\! 1 \!&\! 0\\
\end{array}} \!\right]\!,~\mathbf{v} \!=\! \left[\! \begin{array}{l}
0\\
0 
\end{array} \!\right]\!,~\overline{\mathbf{v}} \!=\! \left[\! \begin{array}{l}
0.9\\
0.8
\end{array} \!\right]\!,~\underline{\mathbf{v}} \!=\! \left[\! \begin{array}{l}
-0.9\\
-0.8
\end{array} \!\right]\!,~\mathbf{C} \!=\! 1,~\mathbf{z} \!=\! 0,~\overline{\mathbf{z}} \!=\!25,~\underline{\mathbf{z}} \!=\! -25\label{exp1102} 
\end{align}
based on which, then according to the $\overline{\Lambda}$, $\underline{\Lambda}$ and ${\mathbf{d}}$ defined in\cref{safelemmaorg}, $\overline{\Delta}$, $\underline{\Delta}$ and ${\mathbf{c}}$ defined in \cref{safelemma}, we have   
\begin{align}
\mathbf{d} = \left[\! \begin{array}{l}
-1\\
-1 
\end{array} \!\right]\!,~~~\overline{\Lambda} = \underline{\Lambda} = \left[\! {\begin{array}{*{20}{c}}
0.9 & 0 \\
0 & 0.8 
\end{array}} \!\right],~~\overline{\Delta} = \underline{\Delta} = 25,~~\mathbf{c} = -1,~~  \label{exp1101} 
\end{align}
from which and \cref{exp1102}, we then have 
\begin{align}
\overline{\mathbf{D}} = \frac{{\mathbf{D}}}{\overline{\Lambda}} = \underline{\mathbf{D}} = \frac{{\mathbf{D}}}{\underline{\Lambda}} = \left[\! {\begin{array}{*{20}{c}}
\frac{10}{9}\!&\!0 \!&\! 0 \!&\! 0\\
0\!&\!0 \!&\! \frac{5}{4} \!&\! 0\\
\end{array}} \!\right], ~~\overline{\mathbf{C}} = \frac{{\mathbf{C}}}{\overline{\Delta}} = \underline{\mathbf{C}} = \frac{{\mathbf{C}}}{\underline{\Delta}} =  \frac{1}{25}. \label{exp1103} 
\end{align}
Then, for givn $\alpha = 0.87$, $\beta = 0.002$, and the model knowledge $(\mathbf{A}, \textbf{B})$ in \cite{Phydrl1}, using the CVX toolbox \cite{grant2009cvx,vandenberghe1998determinant}, we obtain: 
\begin{align} 
\mathbf{P} &= \left[{\begin{array}{*{20}{c}}
   13.3812  &  6.9085 &  17.0004  &  3.6284 \\
    6.9085  &  4.1226 &  10.3597  &  2.2293\\
   17.0004  & 10.3597 &  28.2701  &  5.8142\\
    3.6284  &  2.2293 &   5.8142  &  1.2723\\
\end{array}} \right], \nonumber\\
\mathbf{F} &= \left[{\begin{array}{*{20}{c}}
 22.4008 &  16.9978 &  69.0659  & 12.6449
\end{array}} \right], \nonumber
\end{align}
with which, physics model knowledge in \cref{PM2}, and letting $w( \mathbf{s}(k),\mathbf{a}(k)) = -\mathbf{a}_{\text{drl}}^2(k)$,  the residual action policy \ref{residual} and the safety-embedded reward \eqref{reward} are then ready for the Phy-DRL.

\subsection{HA-Teacher: Physics Knowledge and Design} \label{HA}
Compared with HP-Student, HA-Teacher has relatively rich physics knowledge about system dynamics, which is directly and equivalently transformed from \cref{dymcarpole} as 
\begin{align}
\frac{d}{{dt}} \underbrace{\left[ {\begin{array}{*{20}{c}}
x\\
{\dot x}\\
\theta \\
{\dot \theta }
\end{array}} \right]}_{\mathbf{s}} &= \underbrace{\left[ {\begin{array}{*{20}{c}}
0&1&0&0\\
0&0&{\frac{{ - {m_p}g\sin \theta \cos \theta }}{{\theta [\frac{4}{3}({m_c} + {m_p}) - {m_p}{{\cos }^2}\theta ]}}}&{\frac{{\frac{4}{3}{m_p}l\sin \theta \dot \theta }}{{\frac{4}{3}({m_c} + {m_p}) - {m_p}{{\cos }^2}\theta }}}\\
0&0&0&1\\
0&0&{\frac{{g\sin \theta ({m_c} + {m_p})}}{{l\theta [\frac{4}{3}({m_c} + {m_p}) - {m_p}{{\cos }^2}\theta ]}}}&{\frac{{ - {m_p}\sin \theta \cos \theta \dot \theta }}{{\frac{4}{3}(m{}_c + {m_p}) - {m_p}{{\cos }^2}\theta }}}
\end{array}} \right]}_{\widehat{\mathbf{A}}(\mathbf{s})} \cdot \left[ {\begin{array}{*{20}{c}}
x\\
{\dot x}\\
\theta \\
{\dot \theta }
\end{array}} \right] \nonumber\\
&\hspace{6.3cm}+ \underbrace{\left[ {\begin{array}{*{20}{c}}
0\\
{\frac{{\frac{4}{3}}}{{\frac{4}{3}({m_c} + {m_p}) - {m_p}{{\cos }^2}\theta }}}\\
0\\
{\frac{{ - \cos \theta }}{{l[\frac{4}{3}({m_c} + {m_p}) - {m_p}{{\cos }^2}\theta ]}}}
\end{array}} \right]}_{\widehat{\mathbf{B}}(\mathbf{s})} \cdot \underbrace{F}_{\mathbf{a}}, \label{exp1101cc} 
\end{align}
where $\widehat{\mathbf{A}}(\mathbf{s})$ and $\widehat{\mathbf{B}}(\mathbf{s})$ are known to HA-Teacher. The sampling technique transforms the continuous-time dynamics model \eqref{exp1101cc}  to the discrete-time one: 
\begin{align}
\mathbf{s}(k+1) = (\mathbf{I}_4 + T \cdot \widehat{\mathbf{A}}(\mathbf{s})) \cdot \mathbf{s}(k) + T \cdot \widehat{\mathbf{B}}(\mathbf{s}) \cdot \mathbf{a}(k), \nonumber
\end{align}
from which we obtain the knowledge of $\mathbf{A}(\overline{\mathbf{s}}^{*})$ and $\mathbf{B}(\overline{\mathbf{s}}^{*})$ in \cref{realsyshadesign} as 
\begin{align}
{\mathbf{A}}(\overline{\mathbf{s}}^{*}) = \mathbf{I}_4 + T \cdot \widehat{\mathbf{A}}(\overline{\mathbf{s}}^{*}) ~~\text{and}~~{\mathbf{B}}(\overline{\mathbf{s}}^{*}) = T \cdot \widehat{\mathbf{B}}(\overline{\mathbf{s}}^{*}). \label{ppko}
\end{align}

Meanwhile, for the center of the envelope patch \eqref{goal}, the model mismatch in \cref{assm}, and the switch-triggering condition and dwell time in \eqref{coordinator}, we let $\chi = 0.25$, $\kappa = 0.02$, $\varepsilon = 0.6$, and $\tau = 10$. To always have feasible LMIs \eqref{cc0} and \eqref{cc3}, we let $\alpha = 0.95$ and $\eta = 1.1$. These parameters also guarantee the condition in \cref{cc300} holds.

\subsection{Additional Experimental Results} \label{add}
To further demonstrate the distinguished feature -- lifetime safety, additional experimental results are presented in \cref{ep3}, \cref{ep4}, and \cref{ep5}, where the `Unsafe Continual Learning' and `SeC-Learning Machine' are picked models trained after only \textbf{three episodes}, \textbf{four episodes}, \textbf{five episodes}, respectively, in continual learning. 
\begin{figure}[http]
    \centering
    \subfloat[Initial Condition 1]{\includegraphics[width=0.33\textwidth]{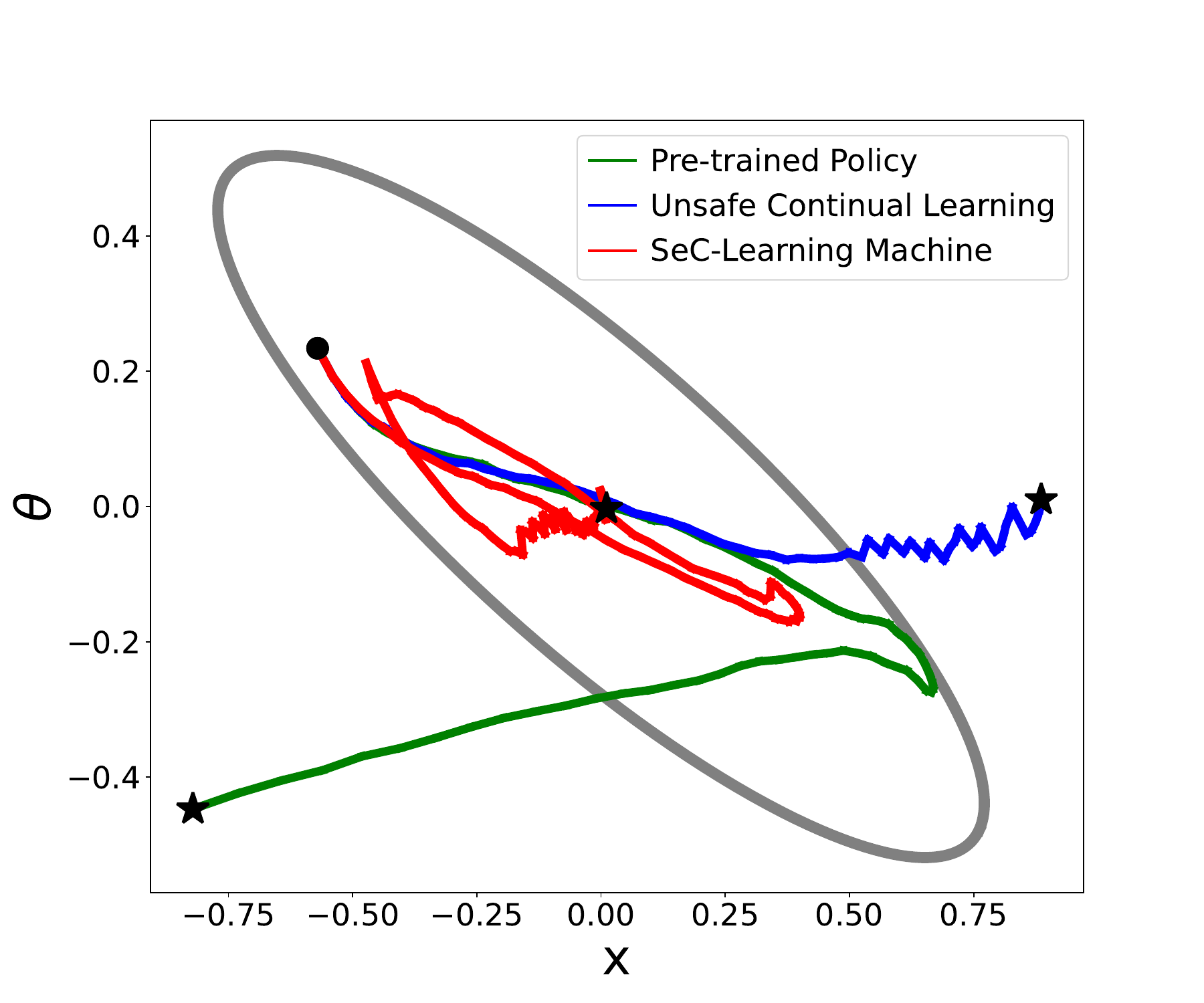}} 
    \centering
    \subfloat[Initial Condition 2]{\includegraphics[width=0.33\textwidth]{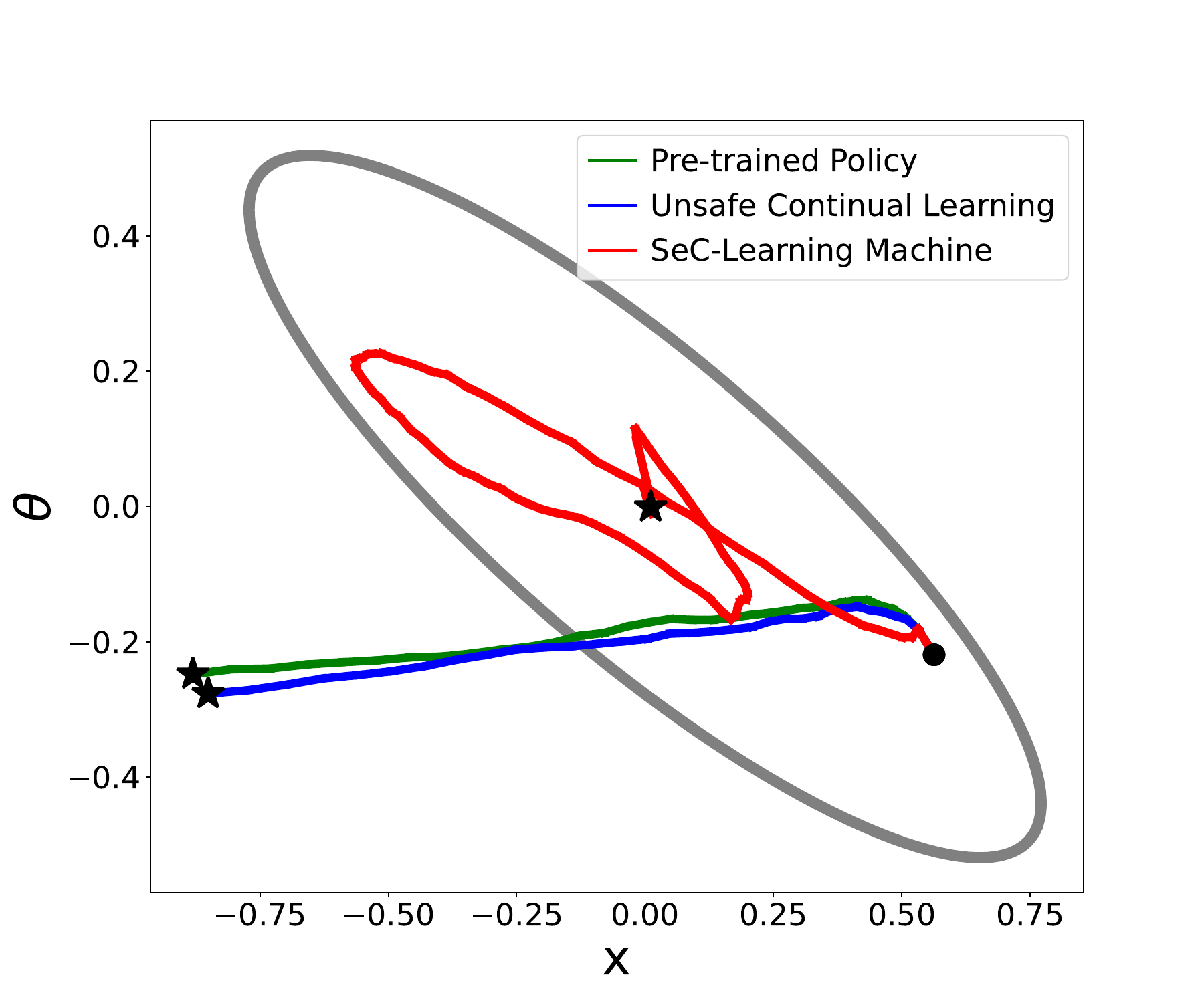}} 
    \centering
    \subfloat[Initial Condition 3]{\includegraphics[width=0.33\textwidth]{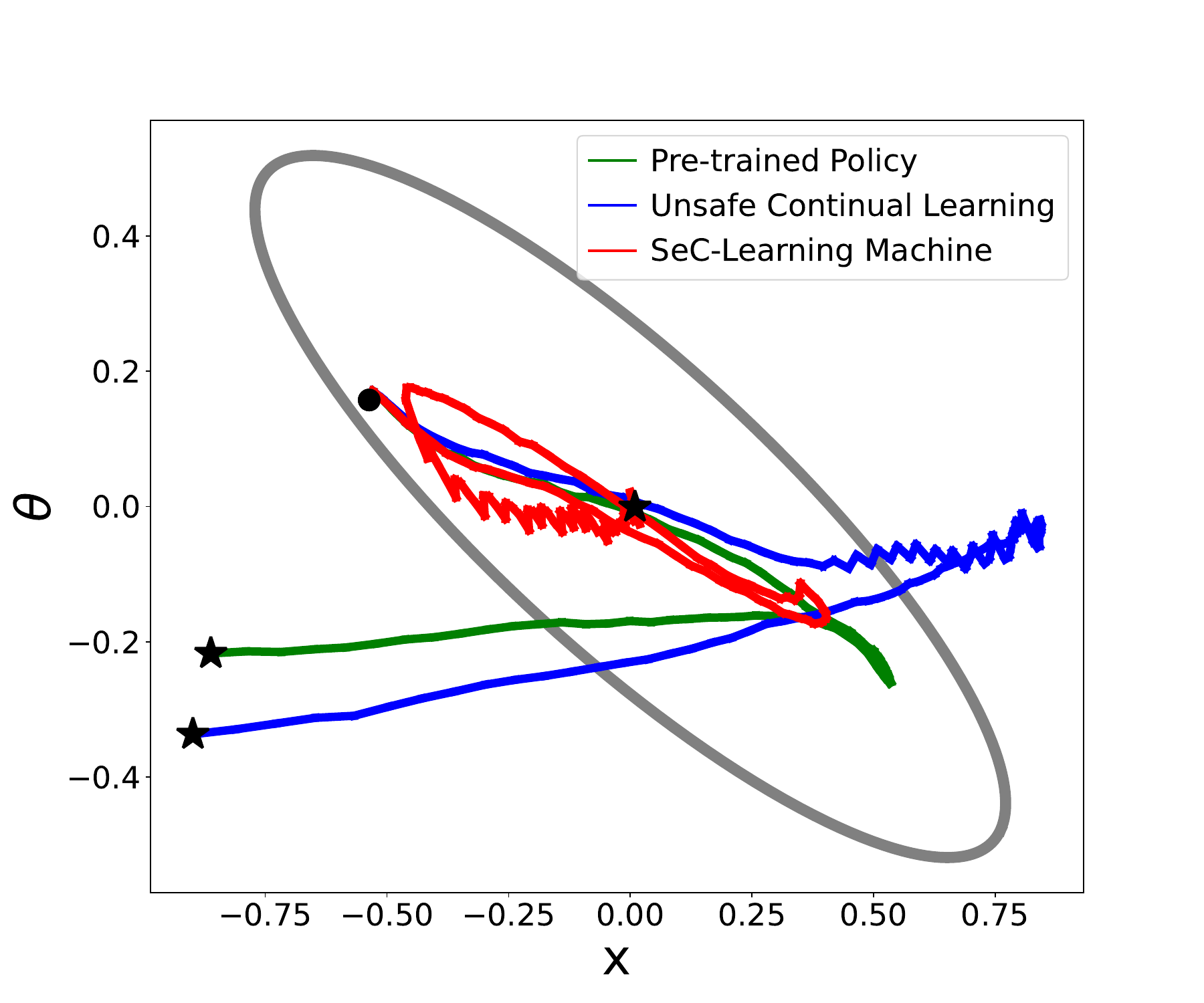}} 
    \centering
    \vspace{-0.0cm}
\caption{\textbf{Three Episodes}. Phase plots, given the same initial condition. The black dot and star denote the initial condition and final location, respectively.}
\label{ep3}
\end{figure}
\begin{figure}[http]
    \centering
    \subfloat[Initial Condition 1]{\includegraphics[width=0.33\textwidth]{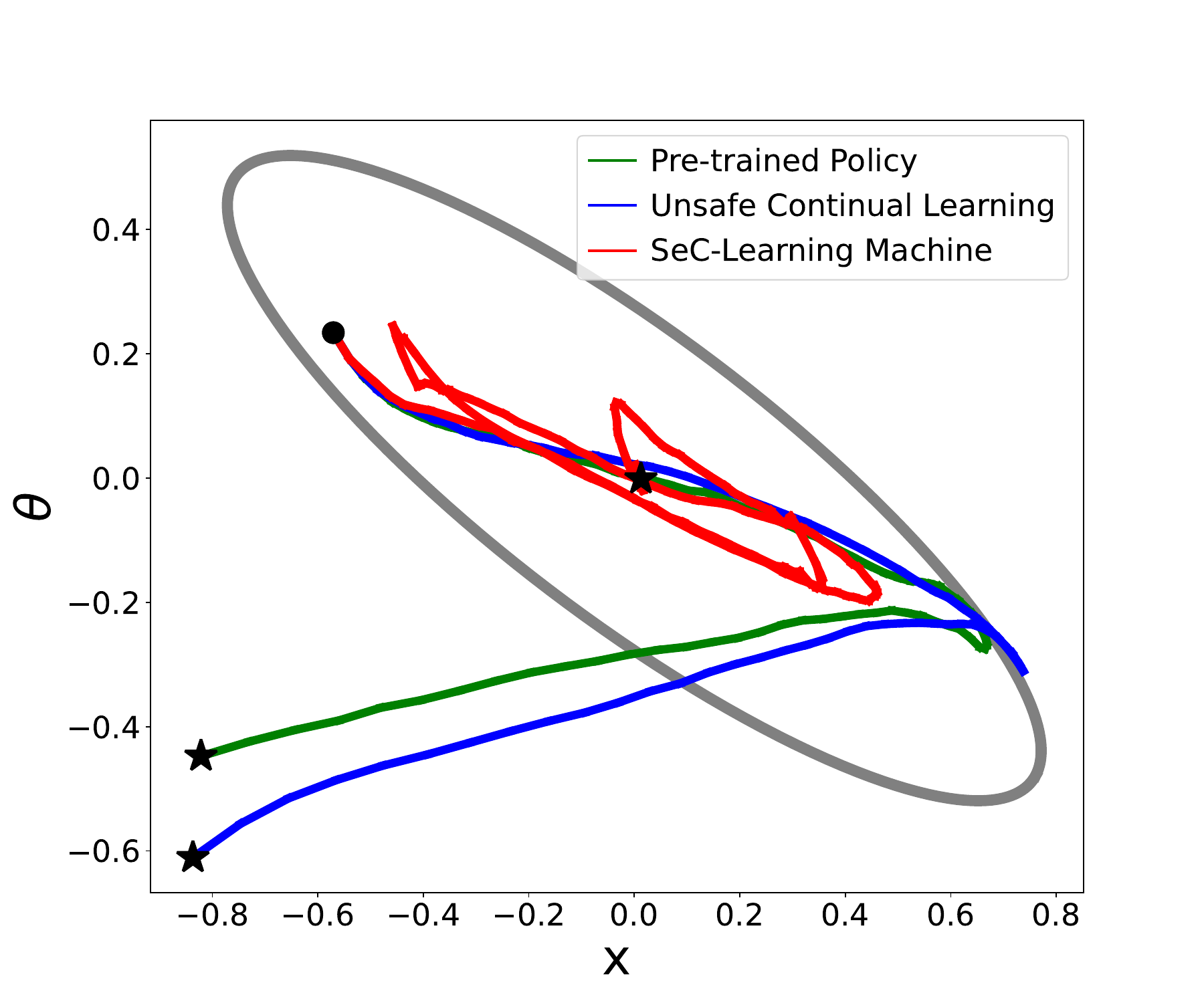}} 
    \centering
    \subfloat[Initial Condition 2]{\includegraphics[width=0.33\textwidth]{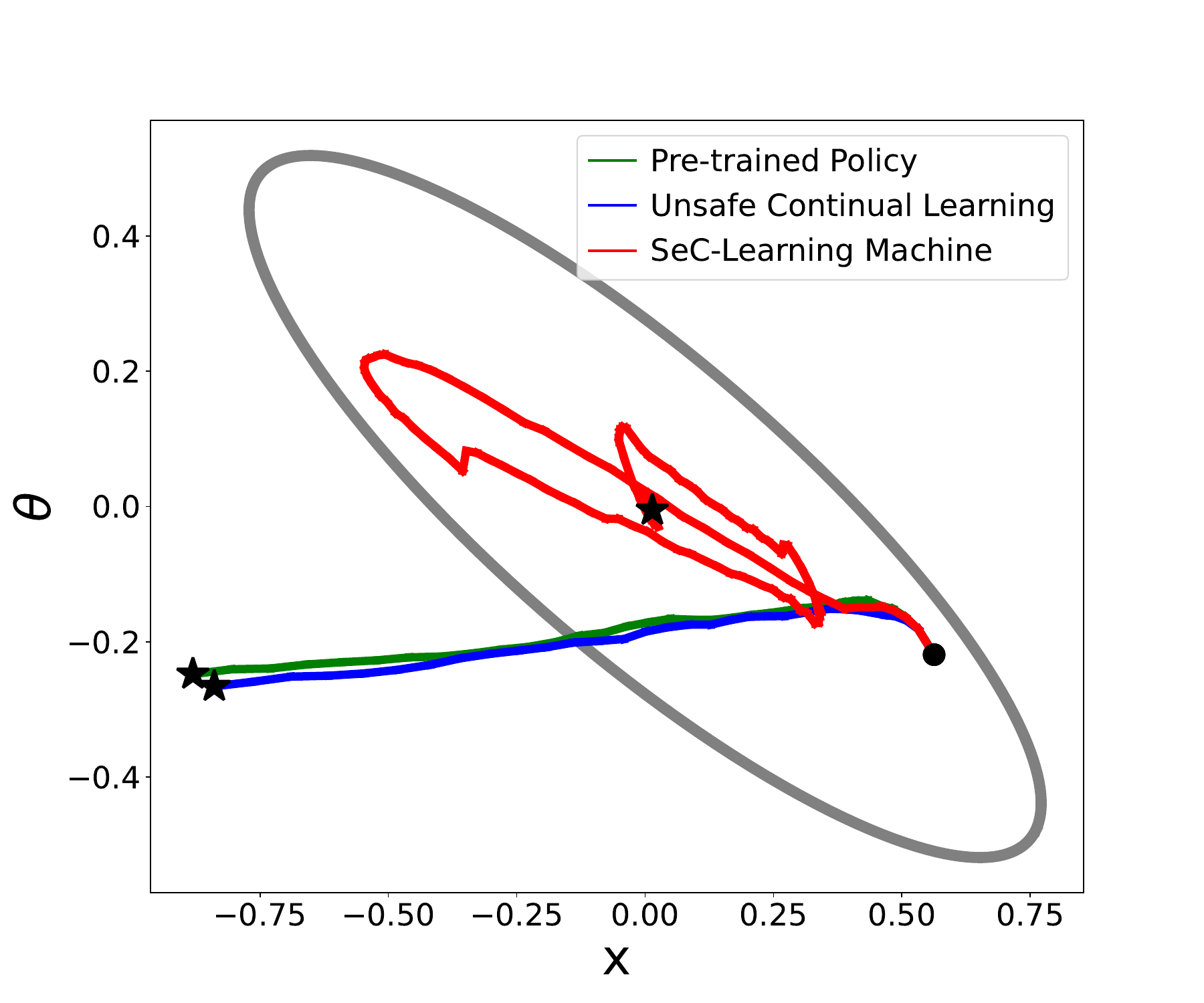}} 
    \centering
    \subfloat[Initial Condition 3]{\includegraphics[width=0.33\textwidth]{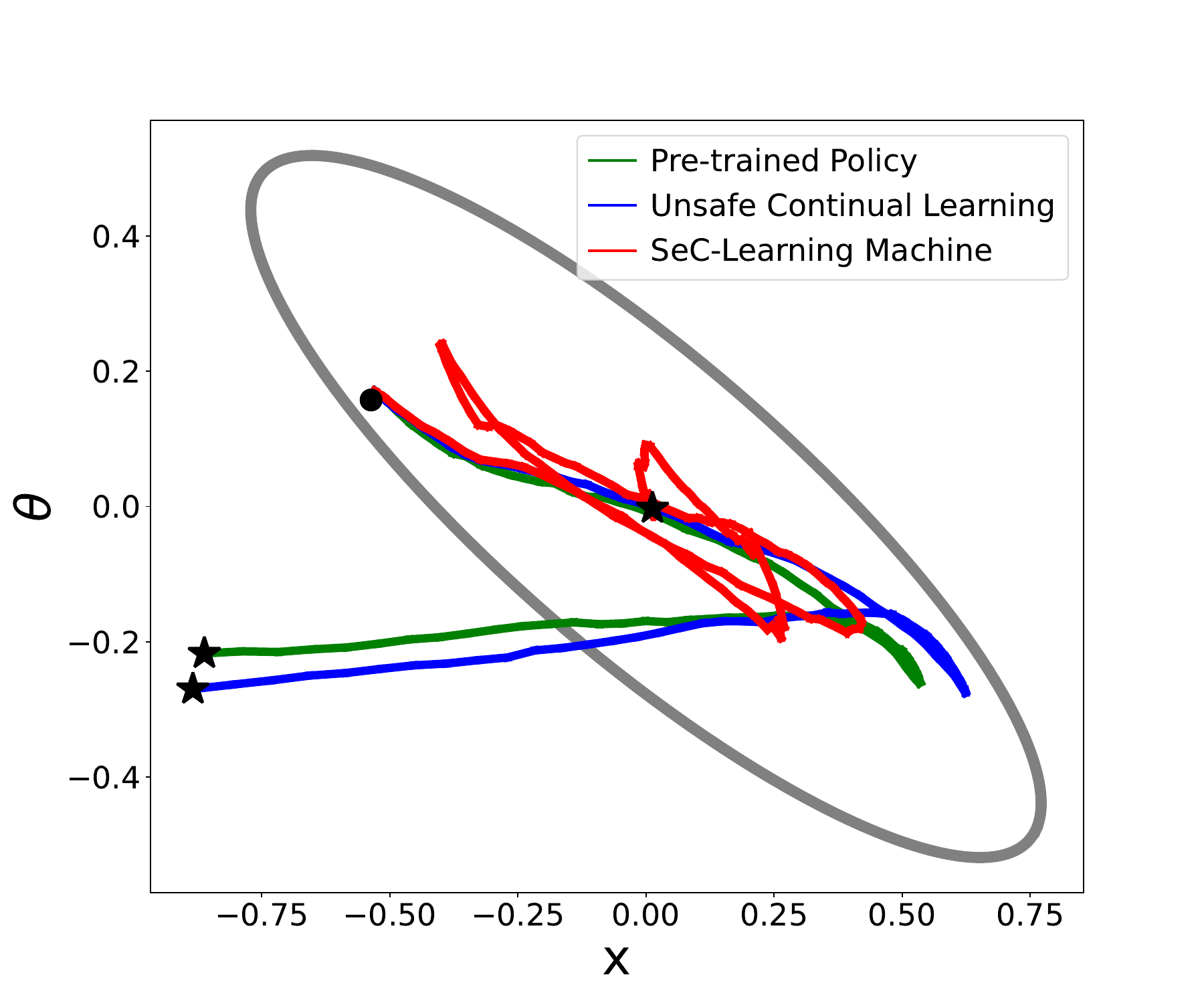}} 
    \centering
    \vspace{-0.0cm}
\caption{\textbf{Four Episodes} Phase plots, given the same initial condition. The black dot and star denote the initial condition and final location, respectively.}
\label{ep4}
\end{figure}
\begin{figure}[http]
    \centering
    \subfloat[Initial Condition 1]{\includegraphics[width=0.33\textwidth]{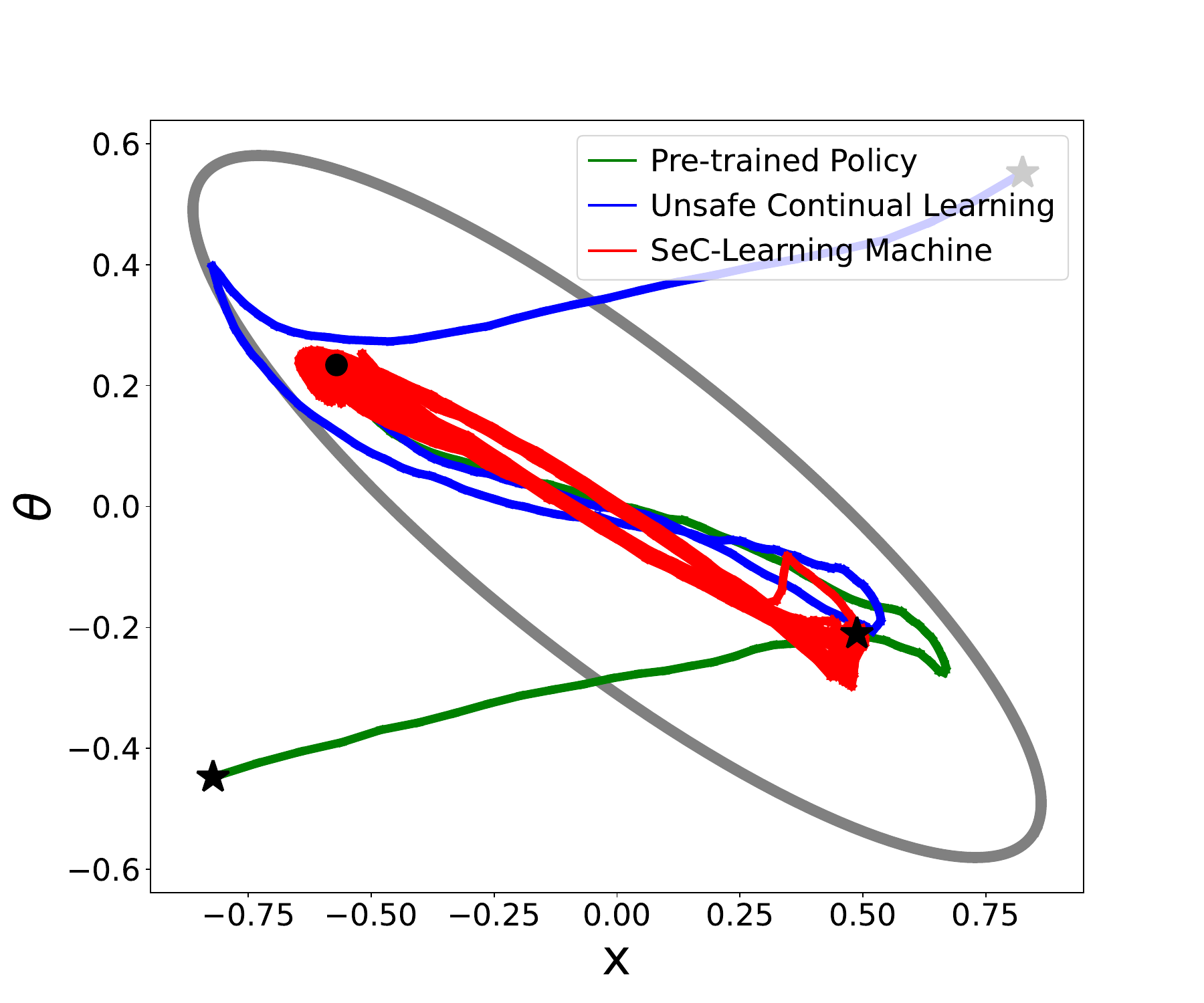}} 
    \centering
    \subfloat[Initial Condition 2]{\includegraphics[width=0.33\textwidth]{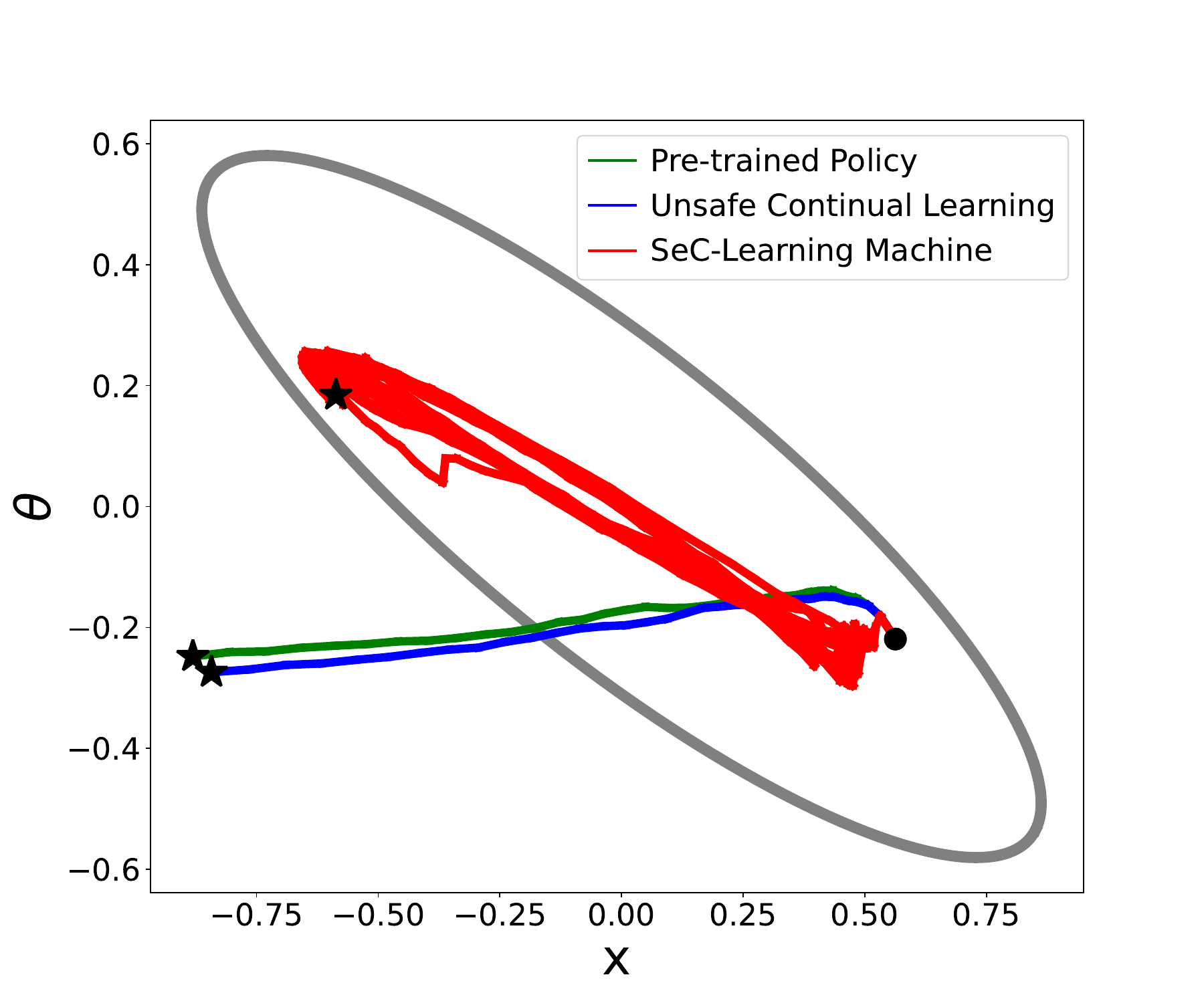}} 
    \centering
    \subfloat[Initial Condition 3]{\includegraphics[width=0.33\textwidth]{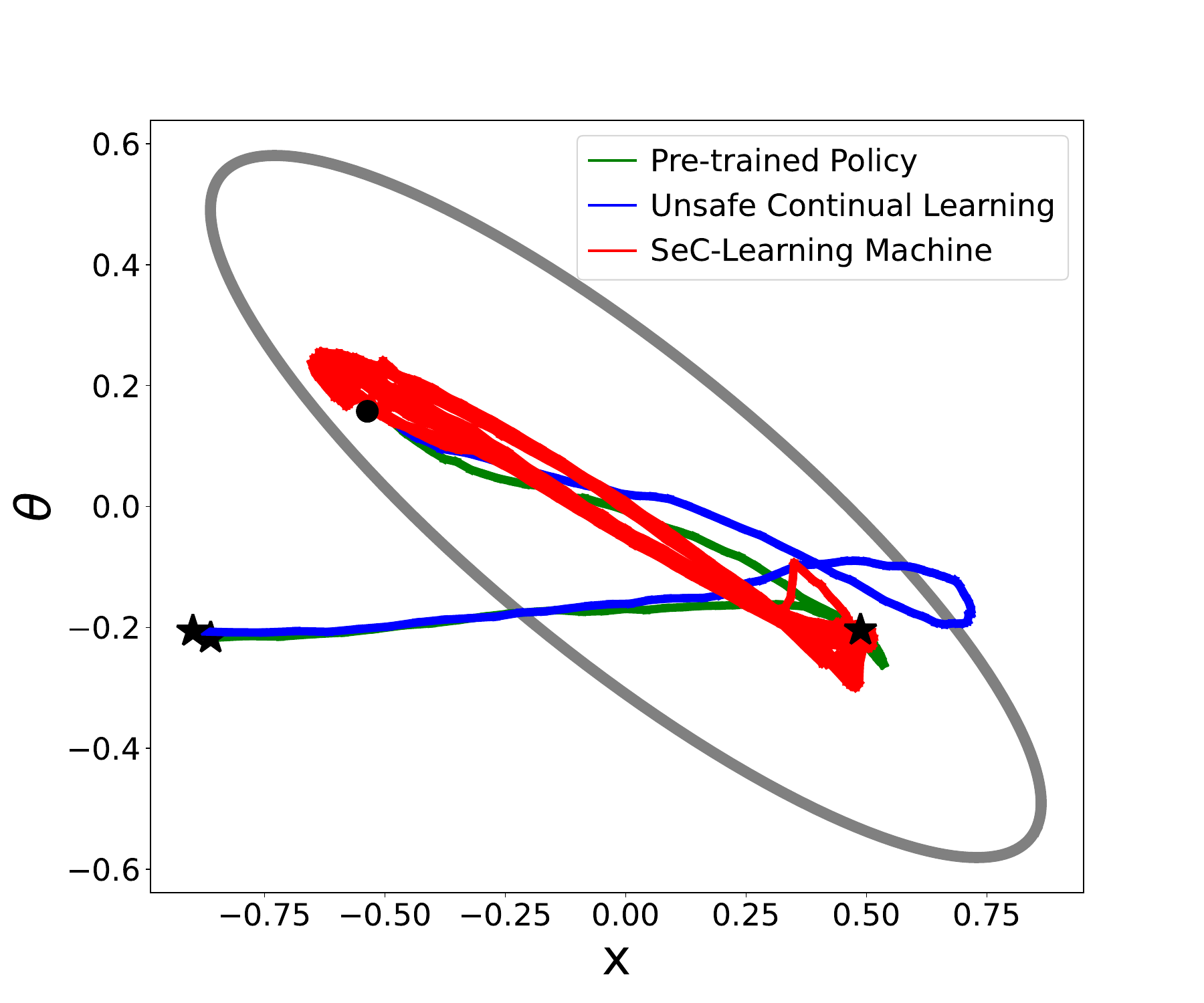}} 
    \centering
    \vspace{-0.0cm}
\caption{\textbf{Five Episodes}. Phase plots, given the same initial condition. The black dot and star denote the initial condition and final location, respectively.}
\label{ep5}
\end{figure}

\newpage
\section{Experiment: Real Quadruped Robot}

In the quadruped experiment, we adopt a Python-based framework for a Unitree A1 robot, released in GitHub by \cite{gityang}. The framework includes a simulation based on Pybullet, an interface for direct sim-to-real transfer, and an implementation of the Convex MPC Controller for basic motion control.

\subsection{Policy Learning} \label{robotjusnetf}
We train SeC-learning machine and  Phy-DRL to achieve the safe mission introduced in \cref{ccrobnt}. The observation of the policy is a 12-dimensional tracking error vector between the robot's state vector and the mission vector. The agent's actions offset the desired positional and lateral accelerations generated from the model-based policy. The computed accelerations are then converted to the low-level motors' torque control. 

The policy is trained using DDPG algorithm \cite{lillicrap2015continuous}. The actor and critic networks are implemented as a Multi-Layer Perceptron (MLP) with four fully connected layers. The output dimensions of the critic network are 256, 128, 64, and 1. The output dimensions of actor networks are 256, 128, 64, and 6. The input of the critic network is the tracking error vector and the action vector. The input of the actor network is the tracking error vector. The activation functions of the first three neural layers are ReLU, while the output of the last layer is the Tanh function for the actor network and Linear for the critic network. In more detail, we let discount factor $\gamma = 0.9$, and the learning rates of critic and actor networks are the same as 0.0003. We set the batch size to 512. The maximum step number for one episode is 10,000.

\subsection{System Dynamics} \label{snsdyanfhgtrobot}

The physics knowledge about the robot used by HP-Student and HA-Teacher for their designs is from the dynamics model of the robot, which is based on a single rigid body subject to forces at the 
contact patches \cite{di2018dynamic}. The considered robot dynamics is characterized by the position of the body's center of mass (CoM) height $h$, the CoM velocity $\mathbf{v}  \triangleq \dot{\mathbf{p} } = [\text{CoM x-velocity}; \text{CoM y-velocity}; \text{CoM z-velocity}] \in \mathbb{R}^{3}$, the Euler angles $\widetilde{\mathbf{e}} = [\phi; \theta; \psi] \in \mathbb{R}^{3}$ with $\phi$, $\theta$ and $\psi$ being roll, pitch and yaw angles, respectively, and the angular velocity in world coordinates $\mathbf{w} \in \mathbb{R}^{3}$.

According to the literature \cite{di2018dynamic}, the body dynamics of quadruped robots can be described by 
\begin{align}
&\frac{\mathbf{d}}{{\mathbf{d}t}}\underbrace{\left[ {\begin{array}{*{20}{c}}
h\\
\widetilde{\mathbf{e}}\\
\mathbf{v}\\
\mathbf{w}
\end{array}} \right]}_{\triangleq ~\widehat{\mathbf{s}}}  = \underbrace{\left[ {\begin{array}{*{20}{c}}
{{\mathbf{O}_{1 \times 1}}}&{{\mathbf{O}_{1 \times 5}}}&{1}&{{\mathbf{O}_{1 \times 3}}}\\
{{\mathbf{O}_{3 \times 3}}}&{{\mathbf{O}_{3 \times 3}}}&{{\mathbf{O}_{3 \times 3}}}&{\mathbf{R}(\phi ,\theta ,\psi)}\\
{{\mathbf{O}_{3 \times 3}}}&{{\mathbf{O}_{3 \times 3}}}&{{\mathbf{O}_{3 \times 3}}}&{{\mathbf{O}_{3 \times 3}}}\\
{{\mathbf{O}_{3 \times 3}}}&{{\mathbf{O}_{3 \times 3}}}&{{\mathbf{O}_{3 \times 3}}}&{{\mathbf{O}_{3 \times 3}}}
\end{array}} \right]}_{\triangleq ~\widehat{\mathbf{A}}(\phi,\theta,\psi)} \cdot \left[ {\begin{array}{*{20}{c}}
h\\
\widetilde{\mathbf{e}}\\
\mathbf{v}\\
\mathbf{w}
\end{array}} \right] + \widehat{\mathbf{B}} \cdot \widehat{a} + \left[ {\begin{array}{*{20}{c}}
{{0}}\\
\mathbf{O}_{3 \times 1}\\
\mathbf{O}_{3 \times 1}\\
{\widetilde{\mathbf{g}}}
\end{array}} \right] \nonumber\\
&\hspace{11cm} +~ \mathbf{f}(\widehat{\mathbf{s}}), \label{dogdynamics}
\end{align}
where ${\widetilde{\mathbf{g}}} = [0;0; -g] \in \mathbb{R}^{3}$ with $g$ being the gravitational acceleration, $\mathbf{f}(\widehat{\mathbf{s}})$ denotes model mismatch, $\widetilde{\mathbf{B}}  = [\mathbf{O}_{4 \times 6}; ~\mathbf{I}_6]^{\top}$, and the $\mathbf{R}(\phi,\theta,\psi) = \mathbf{R}_{z}(\psi) \cdot \mathbf{R}_{y}(\theta) \cdot \mathbf{R}_{x}(\phi) \in \mathbb{R}^{3 \times 3}$ is the rotation matrix with 
\begin{align}
{\mathbf{R}_x}(\phi) \!=\! \left[\!\! {\begin{array}{*{20}{c}}
1&0&0\\
0&{\cos \phi }&{ - \sin \phi }\\
0&{\sin \phi }&{\cos \phi }
\end{array}} \!\!\right]\!,\!~{\mathbf{R}_y}(\theta) \!=\! \left[\!\! {\begin{array}{*{20}{c}}
{\cos \theta }&0&{\sin \theta }\\
0&1&0\\
{ - \sin \theta }&0&{\cos \theta }
\end{array}} \!\!\right]\!,\!~{\mathbf{R}_z}(\psi) \!=\! \left[\!\! {\begin{array}{*{20}{c}}
{\cos \psi }&{ - \sin \psi }&0\\
{\sin \psi }&{\cos \psi }&0\\
0&0&1
\end{array}} \!\!\right]\!. \nonumber
\end{align}

\subsection{HP-Student: Physics Knowledge and Design} \label{HPdog}
To have the model knowledge represented by $\left( {\mathbf{A},~ \mathbf{B}} \right)$ pertaining to robot dynamics \eqref{dogdynamics}, we make the simplification: $\mathbf{R}(\phi,\theta,\psi) =  \mathbf{I}_{3}$, which is obtained through setting the zero angels of roll, pitch and yaw, i.e., $\phi = \theta = \psi = 0$. Referring to \eqref{dogdynamics} and ignoring unknown model mismatch, we can obtain a simplified linear model pertaining to robot dynamics (\ref{dogdynamics}): 
\begin{align}
\frac{\mathbf{d}}{{\mathbf{d}t}}\underbrace{\left[ {\begin{array}{*{20}{c}}
\widetilde{h} \\
\widetilde{\widetilde{\mathbf{e}}}\\
\widetilde{\mathbf{v}} \\
\widetilde{\mathbf{w}}
\end{array}} \right]}_{\triangleq ~ \widetilde{\mathbf{s}}}  = \underbrace{\left[ {\begin{array}{*{20}{c}}
{{\mathbf{O}_{1 \times 1}}}&{{\mathbf{O}_{1 \times 3}}}&{1}&{{\mathbf{O}_{1 \times 5}}}\\
{{\mathbf{O}_{3 \times 3}}}&{{\mathbf{O}_{3 \times 3}}}&{{\mathbf{O}_{3 \times 3}}}&{\mathbf{R}(\phi ,\theta ,\psi)}\\
{{\mathbf{O}_{3 \times 3}}}&{{\mathbf{O}_{3 \times 3}}}&{{\mathbf{O}_{3 \times 3}}}&{{\mathbf{O}_{3 \times 3}}}\\
{{\mathbf{O}_{3 \times 3}}}&{{\mathbf{O}_{3 \times 3}}}&{{\mathbf{O}_{3 \times 3}}}&{{\mathbf{O}_{3 \times 3}}}
\end{array}} \right]}_{\triangleq ~ \widetilde{\mathbf{A}}} \cdot \left[ {\begin{array}{*{20}{c}}
\widetilde{h} \\
\widetilde{\widetilde{\mathbf{e}}}\\
\widetilde{\mathbf{v}} \\
\widetilde{\mathbf{w}}
\end{array}} \right] + {\widehat{\mathbf{B}}} \cdot \widetilde{a}. \label{simdogdynamics}
\end{align}

Given the equilibrium point (or control goal) $\mathbf{s}^*$ and $\mathbf{\widetilde{\mathbf{s}}}$ given in \cref{simdogdynamics}, we define  ${\mathbf{s}} \triangleq \mathbf{\widetilde{\mathbf{s}}} - \mathbf{s}^*$. It is then straightforward to obtain a dynamics from \cref{simdogdynamics} as $\dot{{\mathbf{s}}} = \widetilde{\mathbf{A}} \cdot {\mathbf{s}} + {\widehat{\mathbf{B}}} \cdot \widetilde{\mathbf{a}}$, which transforms to a discrete-time model via sampling technique: 
\begin{align}
\mathbf{s}(k+1) =  \mathbf{A} \cdot {\mathbf{s}}(k) + {{\mathbf{B}}} \cdot \widetilde{\mathbf{a}}(k), ~\text{with}~\mathbf{A} = \mathbf{I}_{10} + T \cdot \widetilde{\mathbf{A}}~\text{and}~\mathbf{B} = T \cdot \widehat{\mathbf{B}}, \label{discremodeless}
\end{align}
where $T$ is the sampling period.

with which and matrices $\mathbf{A}$ and $\mathbf{B}$ in \cref{discremodeless}, we are able to deliver the residual action policy \eqref{residual} and safety-embedded reward \eqref{reward}.

\subsection{HA-Teacher: Physics Knowledge and Design} \label{HAdog}
Compared with HP-Student, HA-Teacher has relatively rich physics knowledge about system dynamics, which is directly and equivalently transformed from \cref{dogdynamics} as 
\begin{align}
\frac{d}{{dt}} \underbrace{\left[ {\begin{array}{*{20}{c}}
h\\
\widetilde{\mathbf{e}}\\
\mathbf{v}\\
\mathbf{w}
\end{array}} \right]}_{\mathbf{s}} &= \underbrace{\left[ {\begin{array}{*{20}{c}}
{{\mathbf{O}_{1 \times 1}}}&{{\mathbf{O}_{1 \times 3}}}&{1}&{{\mathbf{O}_{1 \times 5}}}\\
{{\mathbf{O}_{3 \times 3}}}&{{\mathbf{O}_{3 \times 3}}}&{{\mathbf{O}_{3 \times 3}}}&{\mathbf{R}(\phi ,\theta ,\psi)}\\
{{\mathbf{O}_{3 \times 3}}}&{{\mathbf{O}_{3 \times 3}}}&{{\mathbf{O}_{3 \times 3}}}&{{\mathbf{O}_{3 \times 3}}}\\
{{\mathbf{O}_{3 \times 3}}}&{{\mathbf{O}_{3 \times 3}}}&{{\mathbf{O}_{3 \times 3}}}&{{\mathbf{O}_{3 \times 3}}}
\end{array}} \right]}_{\widehat{\mathbf{A}}(\mathbf{s})} \cdot \left[ {\begin{array}{*{20}{c}}
h\\
\widetilde{\mathbf{e}}\\
\mathbf{v}\\
\mathbf{w} 
\end{array}} \right] + \underbrace{\left[ {\begin{array}{*{20}{c}}
{{\mathbf{O}_3}}&{{\mathbf{O}_3}}&{{\mathbf{O}_3}}&{{\mathbf{O}_3}}\\
{{\mathbf{O}_3}}&{{\mathbf{O}_3}}&{{\mathbf{O}_3}}&{{\mathbf{O}_3}}\\
{{\mathbf{O}_3}}&{{\mathbf{O}_3}}&{{\mathbf{I}_3}}&{{\mathbf{O}_3}}\\
{{\mathbf{O}_3}}&{{\mathbf{O}_3}}&{{\mathbf{O}_3}}&{{\mathbf{I}_3}}
\end{array}} \right]}_{\widehat{\mathbf{B}}(\mathbf{s})} \!~\cdot~\! \mathbf{a} \nonumber\\
&\hspace{10cm} + \mathbf{g}(\mathbf{s}), \label{exp1101cc} 
\end{align}
where $\widehat{\mathbf{A}}(\mathbf{s})$ and $\widehat{\mathbf{B}}(\mathbf{s})$ are known to HA-Teacher. The sampling technique transforms the continuous-time dynamics model \eqref{exp1101cc}  to the discrete-time one: 
\begin{align}
\mathbf{s}(k+1) = (\mathbf{I}_4 + T \cdot \widehat{\mathbf{A}}(\mathbf{s})) \cdot \mathbf{s}(k) + T \cdot \widehat{\mathbf{B}}(\mathbf{s}) \cdot \mathbf{a}(k) + T \cdot \mathbf{g}(\mathbf{s}), \nonumber
\end{align}
from which we obtain the knowledge of $\mathbf{A}(\overline{\mathbf{s}}^{*})$ and $\mathbf{B}(\overline{\mathbf{s}}^{*})$ in \cref{realsyshadesign} as 
\begin{align}
{\mathbf{A}}(\overline{\mathbf{s}}^{*}) = \mathbf{I}_4 + T \cdot \widehat{\mathbf{A}}(\overline{\mathbf{s}}^{*}) ~~\text{and}~~{\mathbf{B}}(\overline{\mathbf{s}}^{*}) = T \cdot \widehat{\mathbf{B}}(\overline{\mathbf{s}}^{*}). \label{ppko}
\end{align}

Meanwhile, for the center of the envelope patch \eqref{goal}, the model mismatch in \cref{assm}, and the switch-triggering condition and dwell time in \eqref{coordinator}, we let $\chi = 0.25$, $\kappa = 0.02$, $\varepsilon = 0.6$, and $\tau = 10$. To always have feasible LMIs \eqref{cc0} and \eqref{cc3}, we let $\alpha = 0.99$ and $\eta = 1.1$. These parameters also guarantee that the condition in \cref{cc300} is fulfilled. 

\subsection{Additional Experimental Results} 
\subsubsection{Trajectories} \label{addtextsim}
The real robot's trajectories of COM height and COM x-velocities under the control of the SeC-learning machine in the 5th episode, 10th episode, 15th episode, and 20th episode are shown in \cref{ep5sec,ep10sec,ep15sec,ep20sec}, respectively. The figures straightforwardly depict that the SeC-learning machine guarantees the safety of real robots in all picked episodes of continual learning.   

\begin{figure}[http]
    \centering
    \subfloat[Height trajectory]{\includegraphics[width=0.45\textwidth]{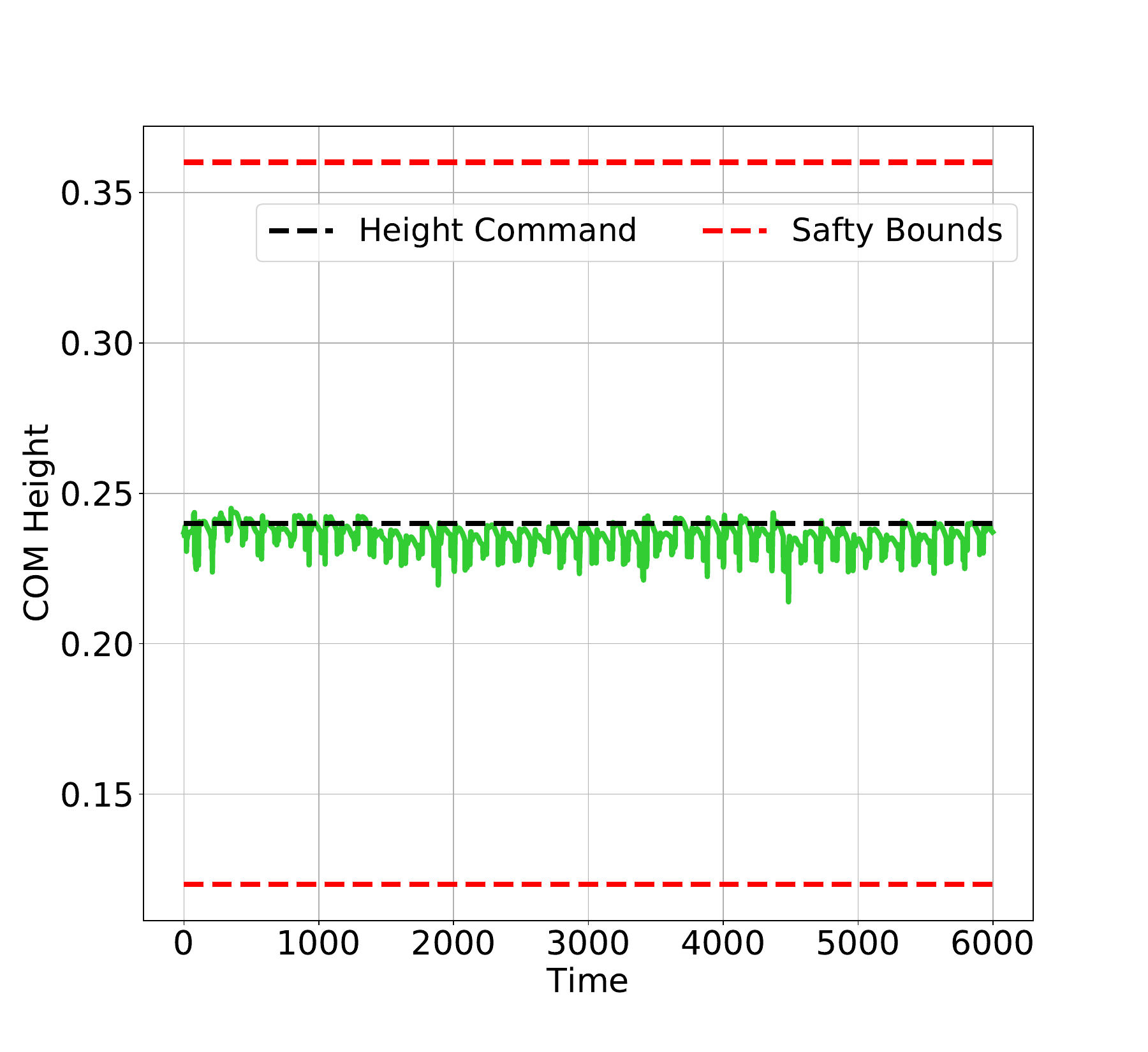}} 
    \centering
    \subfloat[Velocity trajectory]{\includegraphics[width=0.45\textwidth]{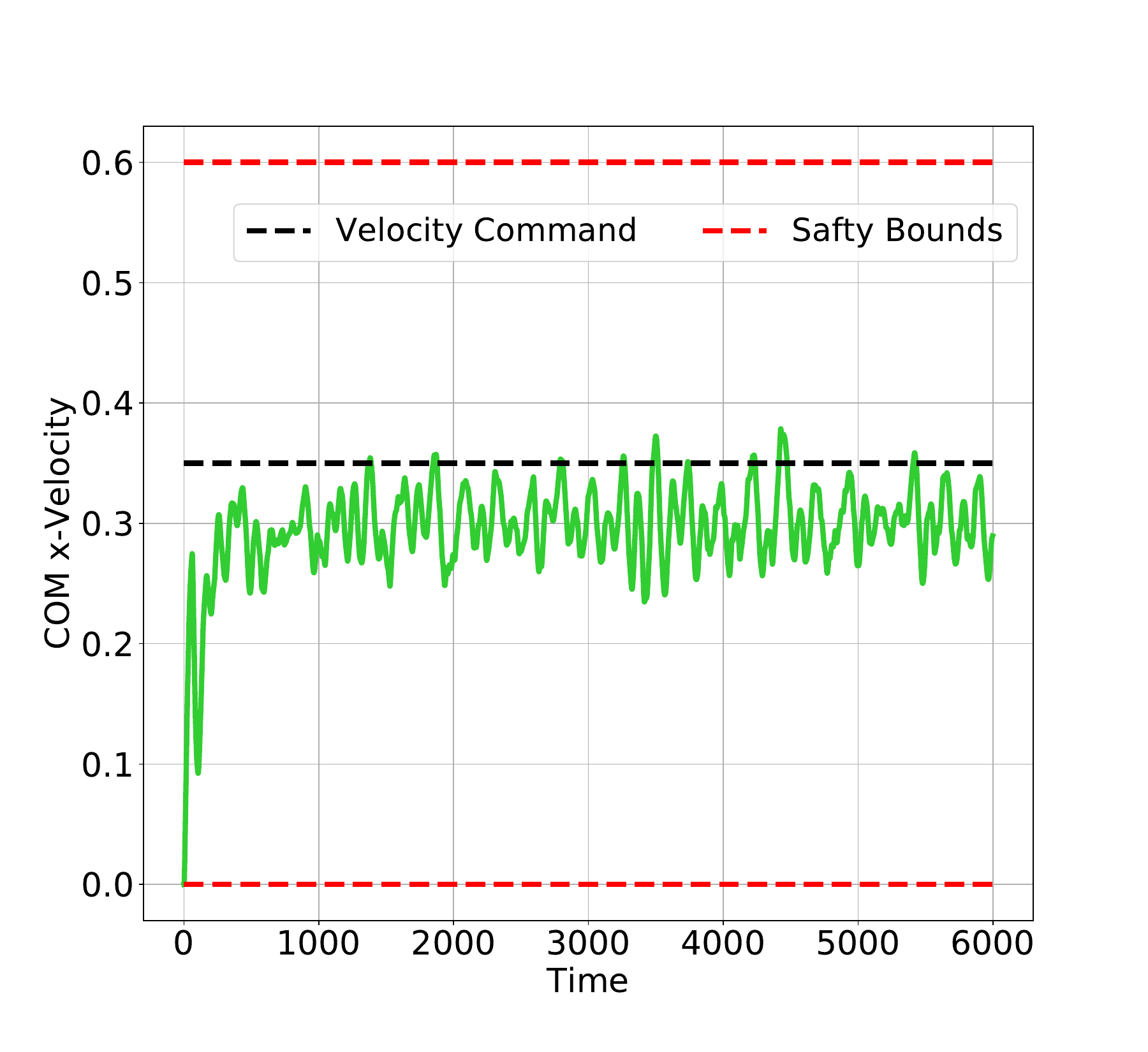}} 
    \centering
    \vspace{0.1cm}
\caption{Robot's trajectories under control of SeC learning machine in the \textbf{5th Episode}.}
\label{ep5sec}
\end{figure}
\begin{figure}[http]
    \centering
    \subfloat[Height trajectory]{\includegraphics[width=0.45\textwidth]{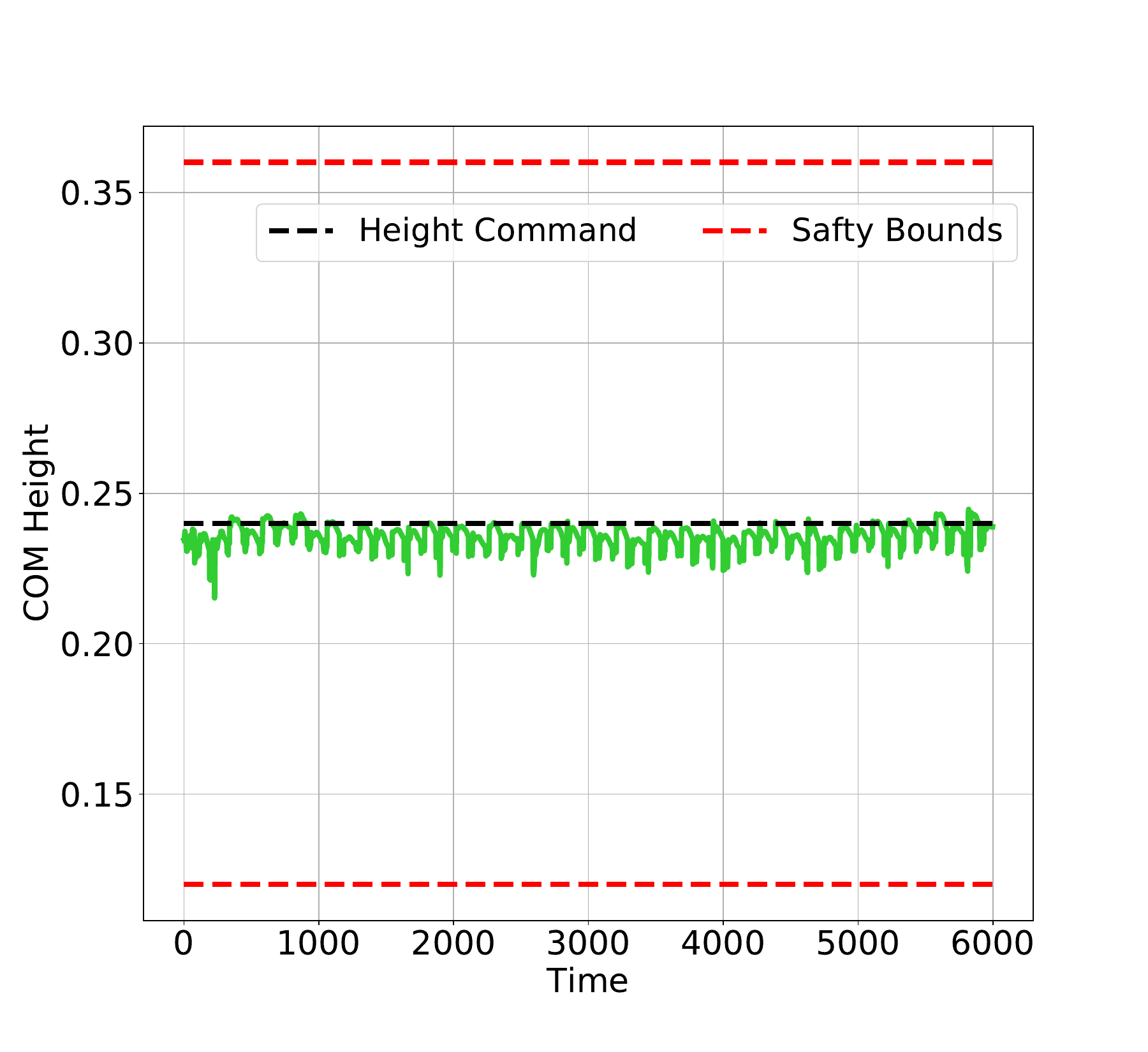}} 
    \centering
    \subfloat[Velocity trajectory]{\includegraphics[width=0.45\textwidth]{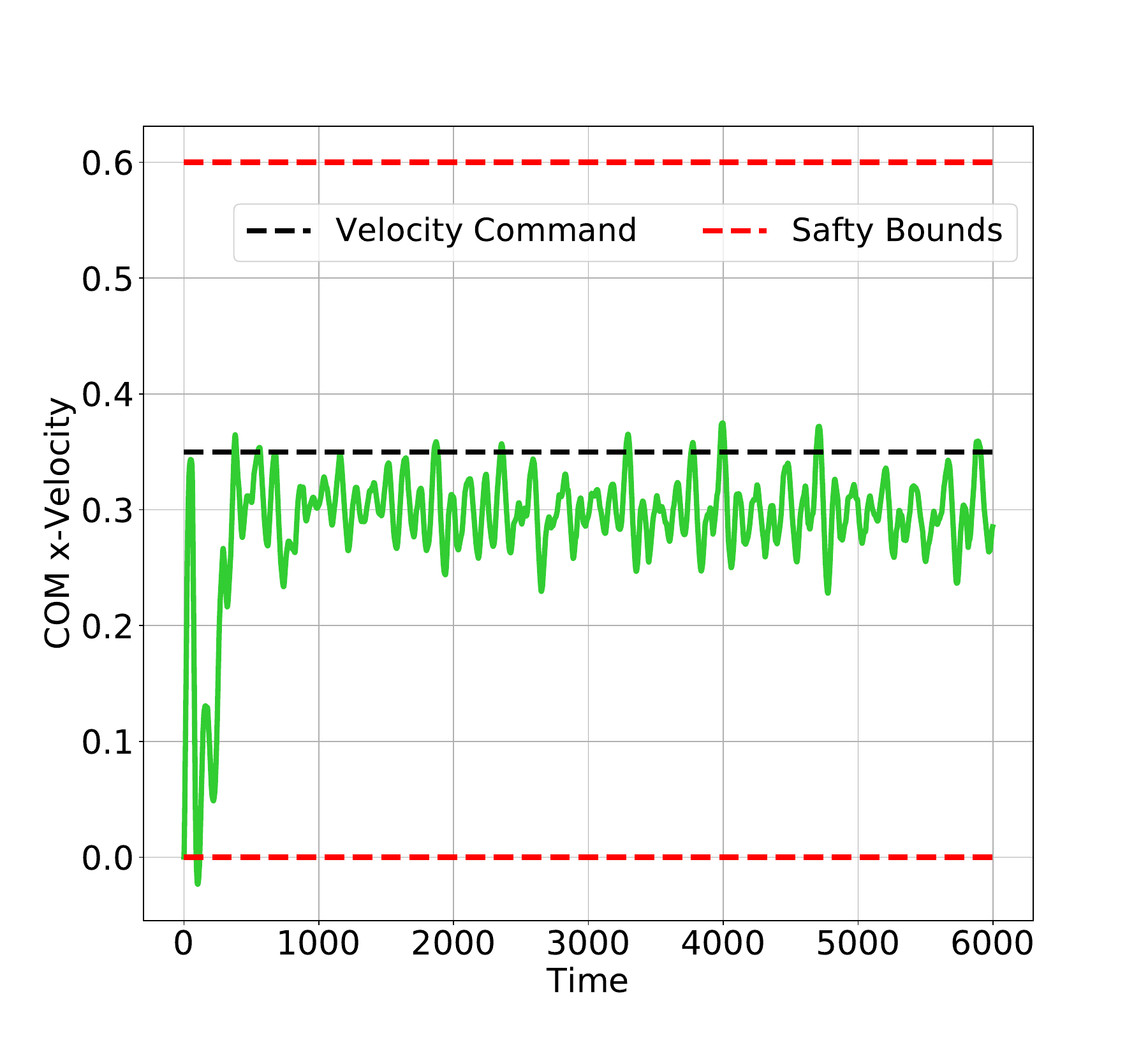}} 
    \centering
    \vspace{-0.0cm}
\caption{Robot's trajectories under control of SeC learning machine in the \textbf{10th Episode}.}
\label{ep10sec}
\end{figure}
\begin{figure}[http]
    \centering
    \subfloat[Height trajectory]{\includegraphics[width=0.45\textwidth]{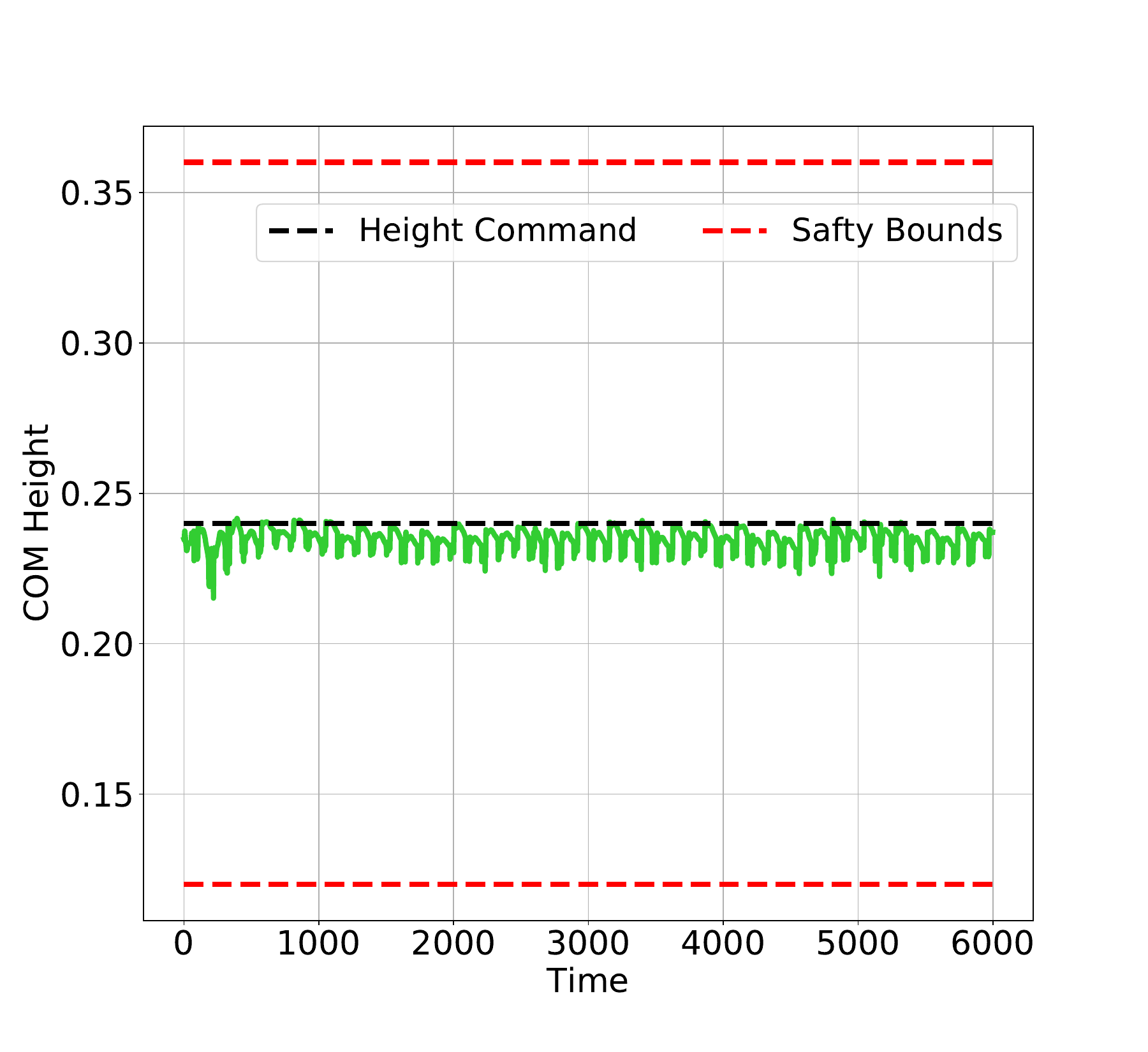}} 
    \centering
    \subfloat[Velocity trajectory]{\includegraphics[width=0.45\textwidth]{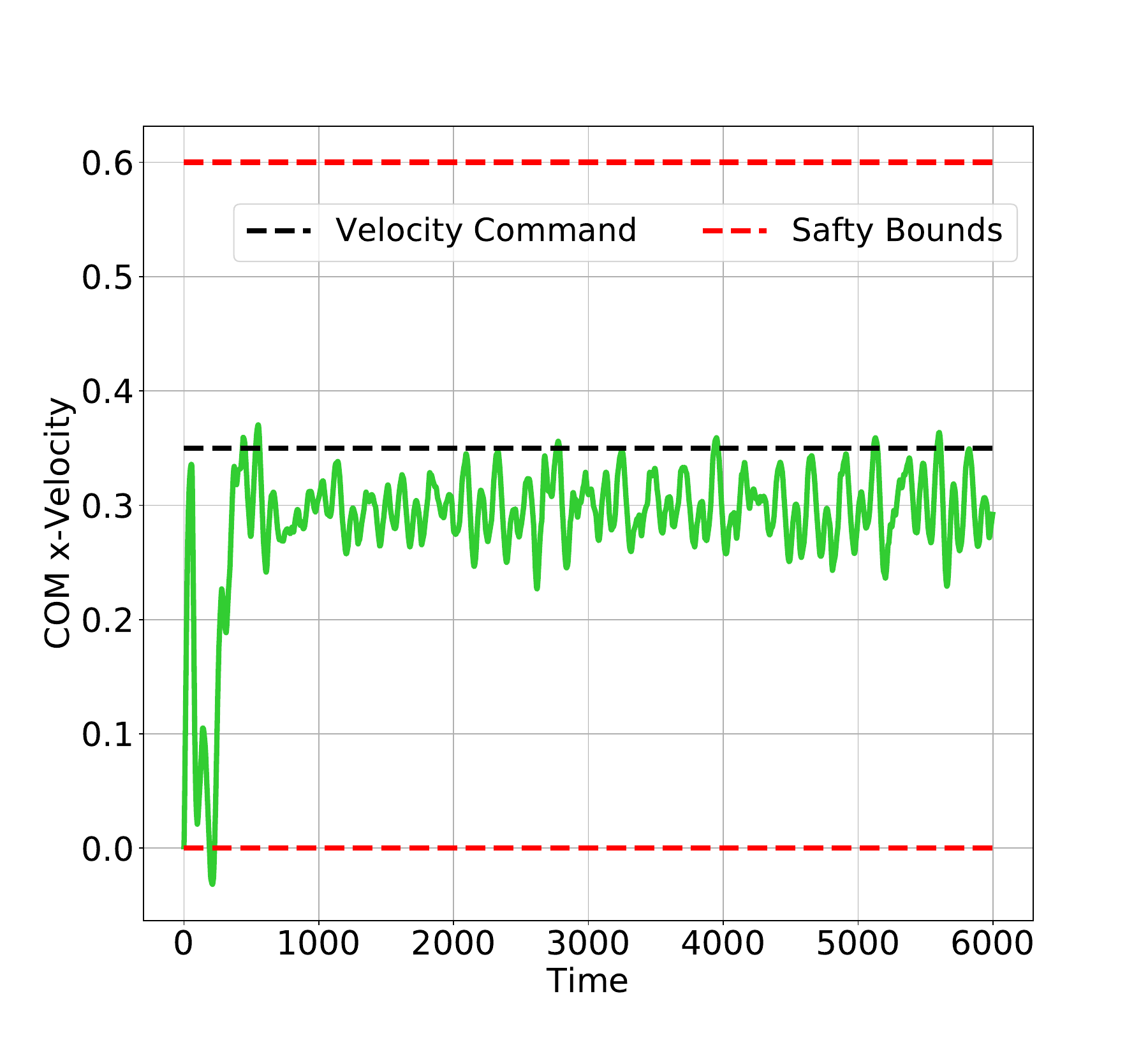}} 
    \centering
    \vspace{-0.0cm}
\caption{Robot's trajectories under control of SeC learning machine in the \textbf{15th Episode}.}
\label{ep15sec}
\end{figure}
\begin{figure}[http]
    \centering
    \subfloat[Height trajectory]{\includegraphics[width=0.45\textwidth]{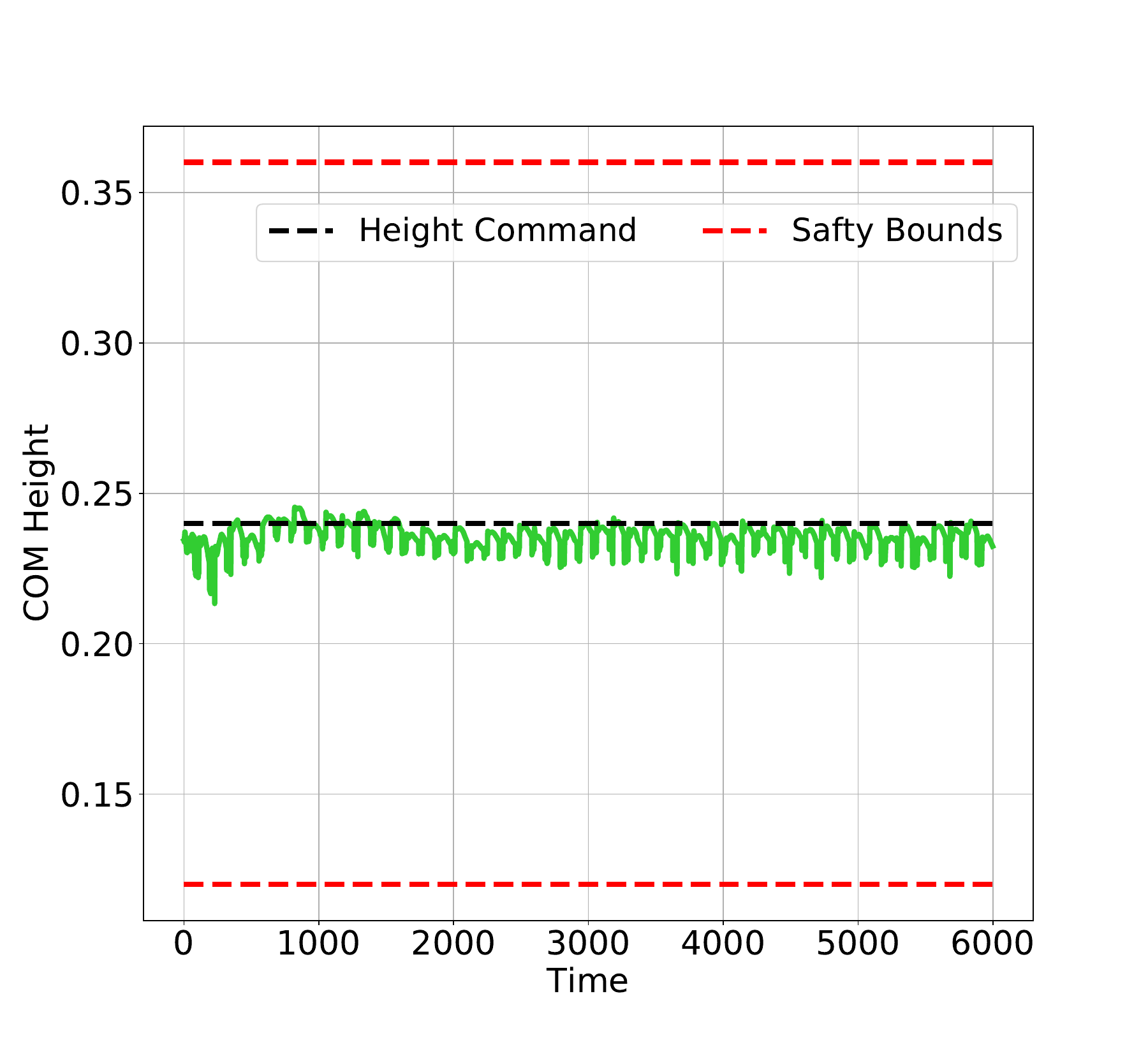}} 
    \centering
    \subfloat[Velocity trajectory]{\includegraphics[width=0.45\textwidth]{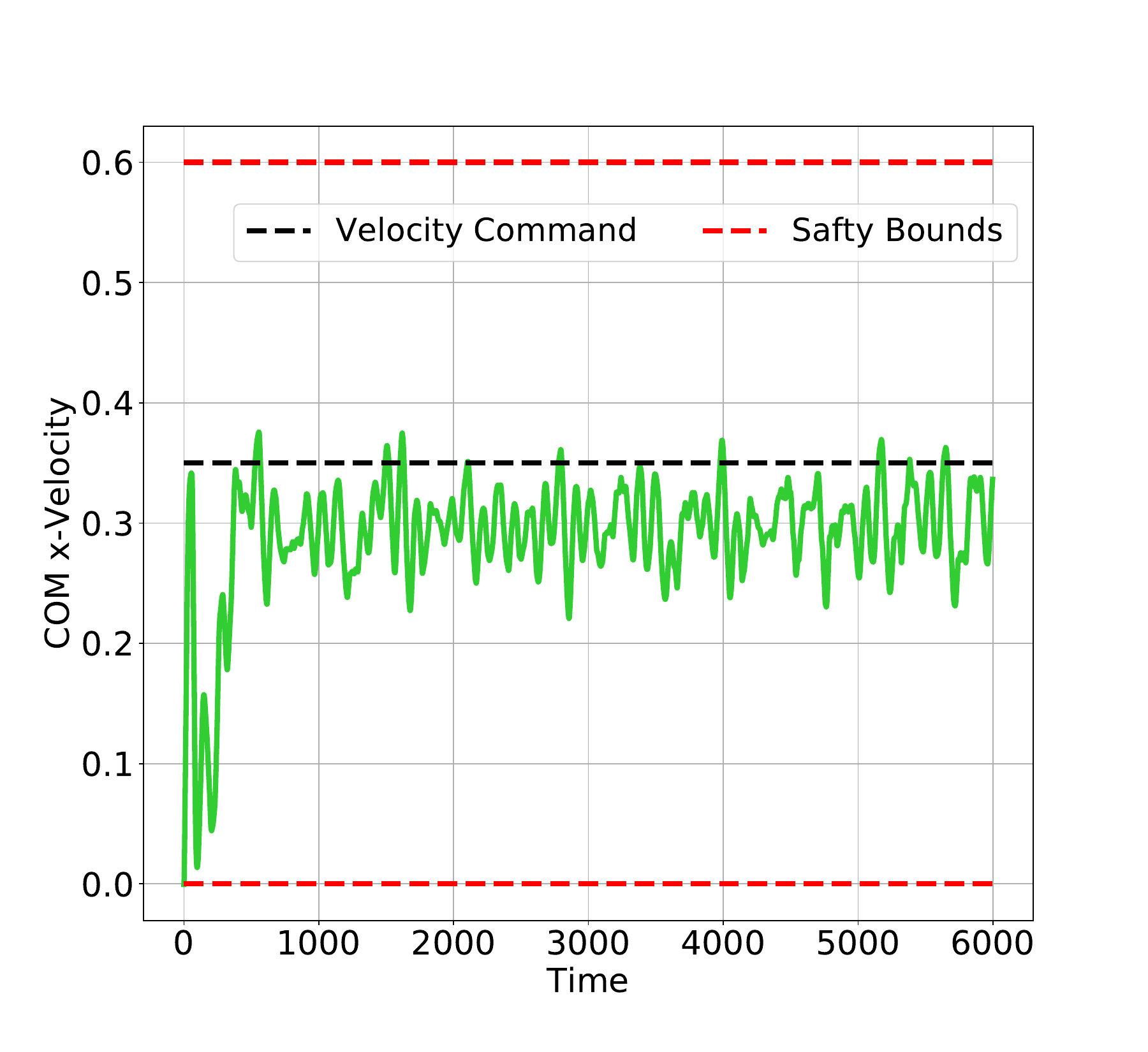}} 
    \centering
    \vspace{-0.0cm}
\caption{Robot's trajectories under control of SeC learning machine in the \textbf{20th Episode}.}
\label{ep20sec}
\end{figure}

\subsubsection{Reward} \label{addtextsimppo}
The reward curves in the iteration step for 20 episodes are shown in \cref{addewwards}. Observing \cref{addewwards}, we conclude that given the same reward for learning, the SeC-Learning Machine exhibits remarkably fast and stable learning, compared with continual learning without Simplex logic. 
\begin{figure}[http]
    \centering
    \subfloat{\includegraphics[width=0.70\textwidth]{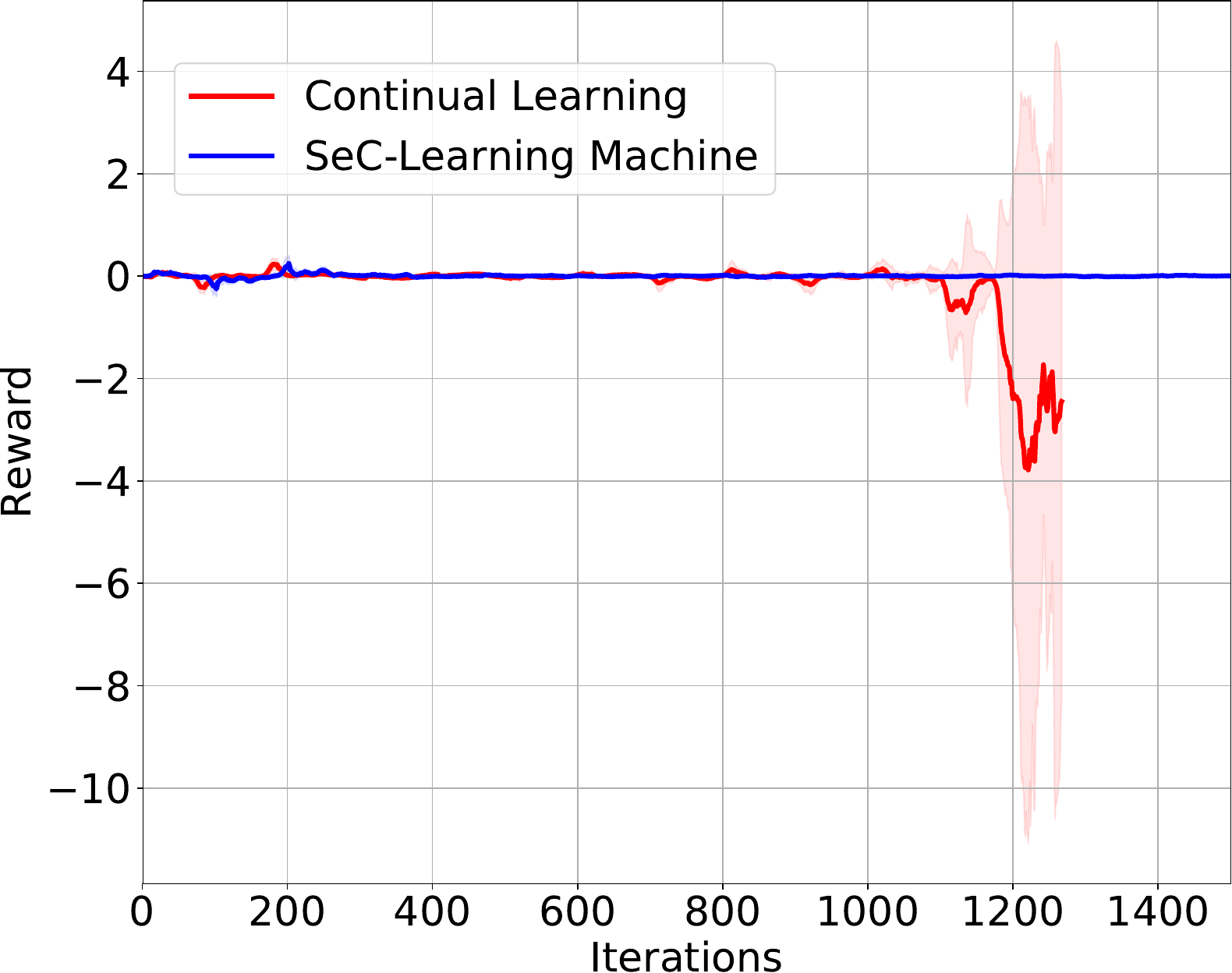}} 
    \centering
    \vspace{-0.0cm}
\caption{Reward curves in the term of iteration steps.}
\label{addewwards}
\end{figure}

\end{document}